\documentclass{article}

\usepackage{microtype}
\usepackage{graphicx}
\usepackage{booktabs} 

\usepackage{hyperref}

\usepackage[accepted]{icml2025} 

\usepackage{amsmath}
\usepackage{amssymb}
\usepackage{mathtools}
\usepackage{amsthm}

\usepackage{amsmath}
\usepackage{amssymb}
\usepackage{mathtools}
\usepackage{amsthm}
\usepackage{comment}
\usepackage{subcaption}
\usepackage{microtype}
\usepackage{graphicx}
\usepackage{caption}
\usepackage{booktabs}
\RequirePackage{algorithm}
\RequirePackage{algorithmic}
\usepackage{threeparttable}
\usepackage{colortbl}
\usepackage{booktabs} 
\usepackage[capitalize,noabbrev]{cleveref}

\theoremstyle{plain}
\newtheorem{theorem}{Theorem}[section]

\newtheorem{lemma}[theorem]{Lemma}

\theoremstyle{definition}

\newtheorem{assumption}[theorem]{Assumption}
\theoremstyle{remark}

\usepackage[textsize=tiny]{todonotes}

\icmltitlerunning{Improving Generalization in Federated Learning via Momentum-Based
	Stochastic Controlled Weight Averaging}
\begin{document}

\twocolumn[
\icmltitle{Improving Generalization in Federated Learning with Highly Heterogeneous Data via Momentum-Based Stochastic Controlled Weight Averaging}




\begin{icmlauthorlist}
	\icmlauthor{Junkang Liu}{tju}
	\icmlauthor{Yuanyuan Liu$^*$}{xdu}
	\icmlauthor{Fanhua Shang$^*$}{tju}
	\icmlauthor{Hongying Liu$^*$}{tju2,pcl}
	\icmlauthor{Jin Liu}{xdu2}
	\icmlauthor{Wei Feng}{tju}
\end{icmlauthorlist}

\icmlaffiliation{tju}{College of Intelligence and Computing, Tianjin University, Tianjin, China}
\icmlaffiliation{tju2}{Medical School, Tianjin University, Tianjin, China}
\icmlaffiliation{pcl}{Peng Cheng Lab, Shenzhen, China}
\icmlaffiliation{xdu}{School of Artificial Intelligence, Xidian University, Xi'an, China}
\icmlaffiliation{xdu2}{School of Cyber Engineering, Xidian University, Xi'an, China}

\icmlcorrespondingauthor{Yuanyuan Liu}{yyliu@xidian.edu.cn}
\icmlcorrespondingauthor{Fanhua Shang}{fhshang@tju.edu.cn}
\icmlcorrespondingauthor{Hongying Liu}{hyliu2009@tju.edu.cn}




\vskip 0.2in
]



\printAffiliationsAndNotice{}  
\vspace{-3mm}
\begin{abstract}
For federated learning (FL) algorithms such as FedSAM, their generalization capability is crucial for real-word applications. In this paper, we revisit the generalization problem in FL and investigate the impact of data heterogeneity on FL generalization. We find that FedSAM usually performs worse than FedAvg in the case of highly heterogeneous data, and thus propose a novel and effective federated learning algorithm with Stochastic Weight Averaging (called \texttt{FedSWA}), which aims to find flatter minima in the setting of highly heterogeneous data. Moreover, we introduce a new momentum-based stochastic controlled weight averaging FL algorithm (\texttt{FedMoSWA}), which is designed to better align local and global models. 
 Theoretically, we provide both convergence analysis and generalization bounds for \texttt{FedSWA} and \texttt{FedMoSWA}. We also prove that the optimization and generalization errors of \texttt{FedMoSWA} are smaller than those of their counterparts, including FedSAM and its variants. Empirically, experimental results on CIFAR10/100 and Tiny ImageNet demonstrate the superiority of the proposed algorithms compared to their counterparts. Open source code at: \url{https://github.com/junkangLiu0/FedSWA}.
\end{abstract}

\section{Introduction}
\label{sec:intro}
\begin{figure}[tb]
	\begin{minipage}[b]{0.155\textwidth}
		\centering
		\subcaptionbox{ FedAvg (0.1)}{\includegraphics[width=\textwidth]{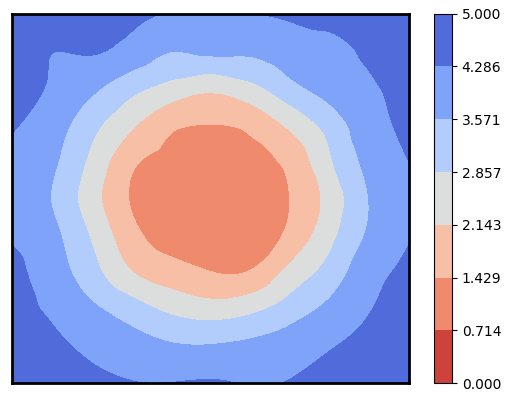}}
	\end{minipage}
	\begin{minipage}[b]{0.155\textwidth}
		\centering
		\subcaptionbox{FedSAM (0.1)}{\includegraphics[width=\textwidth]{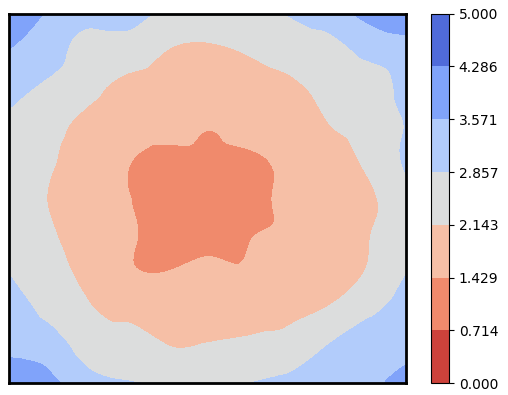}}
	\end{minipage}
	\begin{minipage}[b]{0.155\textwidth}
		\centering
		\subcaptionbox{\texttt{FedSWA} (0.1)}{\includegraphics[width=\textwidth]{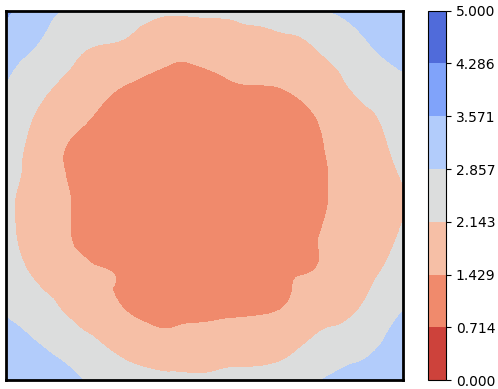}}
	\end{minipage}
	\begin{minipage}[b]{0.15\textwidth}
		\centering
		\subcaptionbox{FedSAM}{\includegraphics[width=\textwidth]{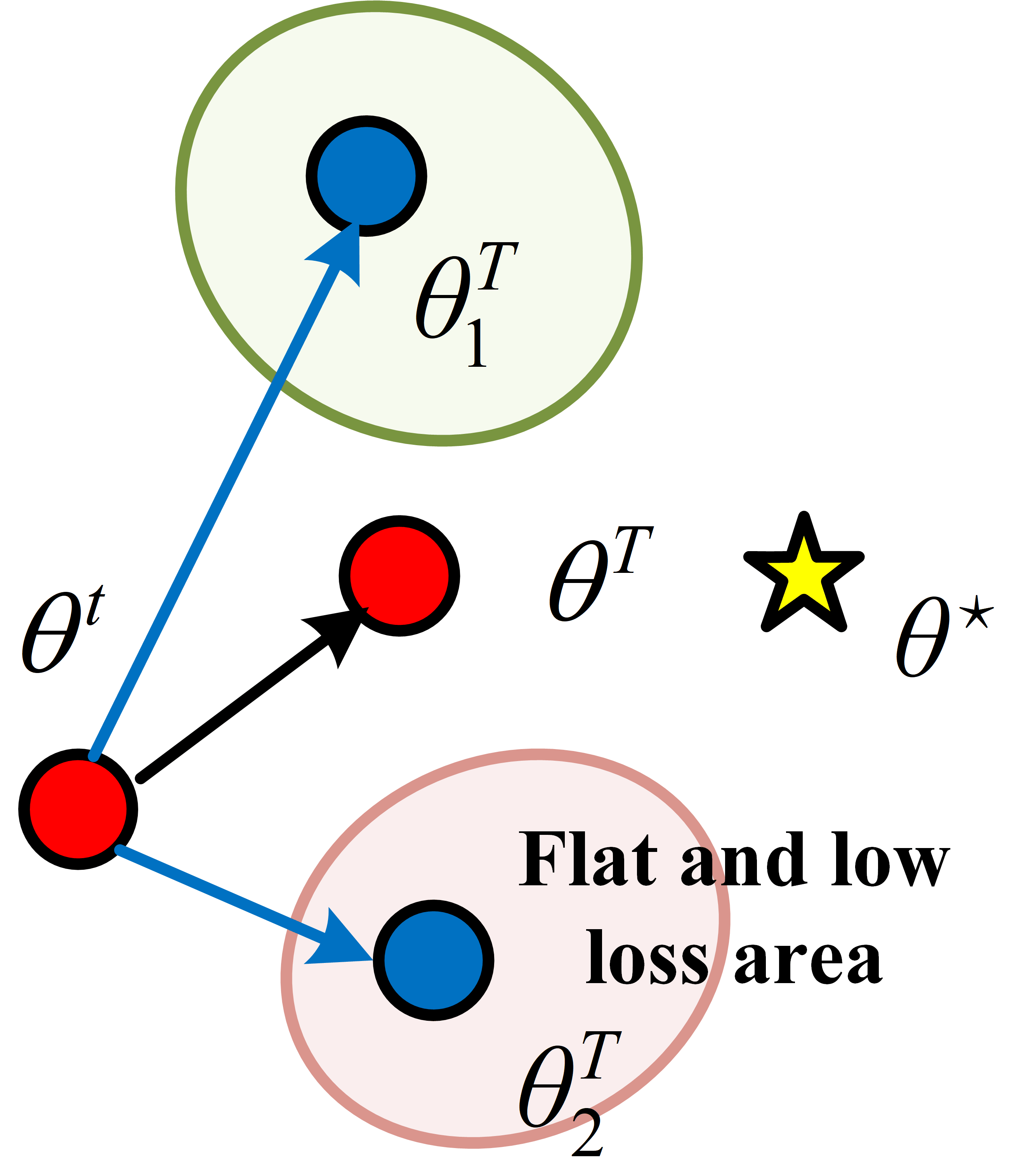}}
	\end{minipage}
	\begin{minipage}[b]{0.15\textwidth}
		\centering
		\subcaptionbox{\texttt{FedSWA}}{\includegraphics[width=\textwidth]{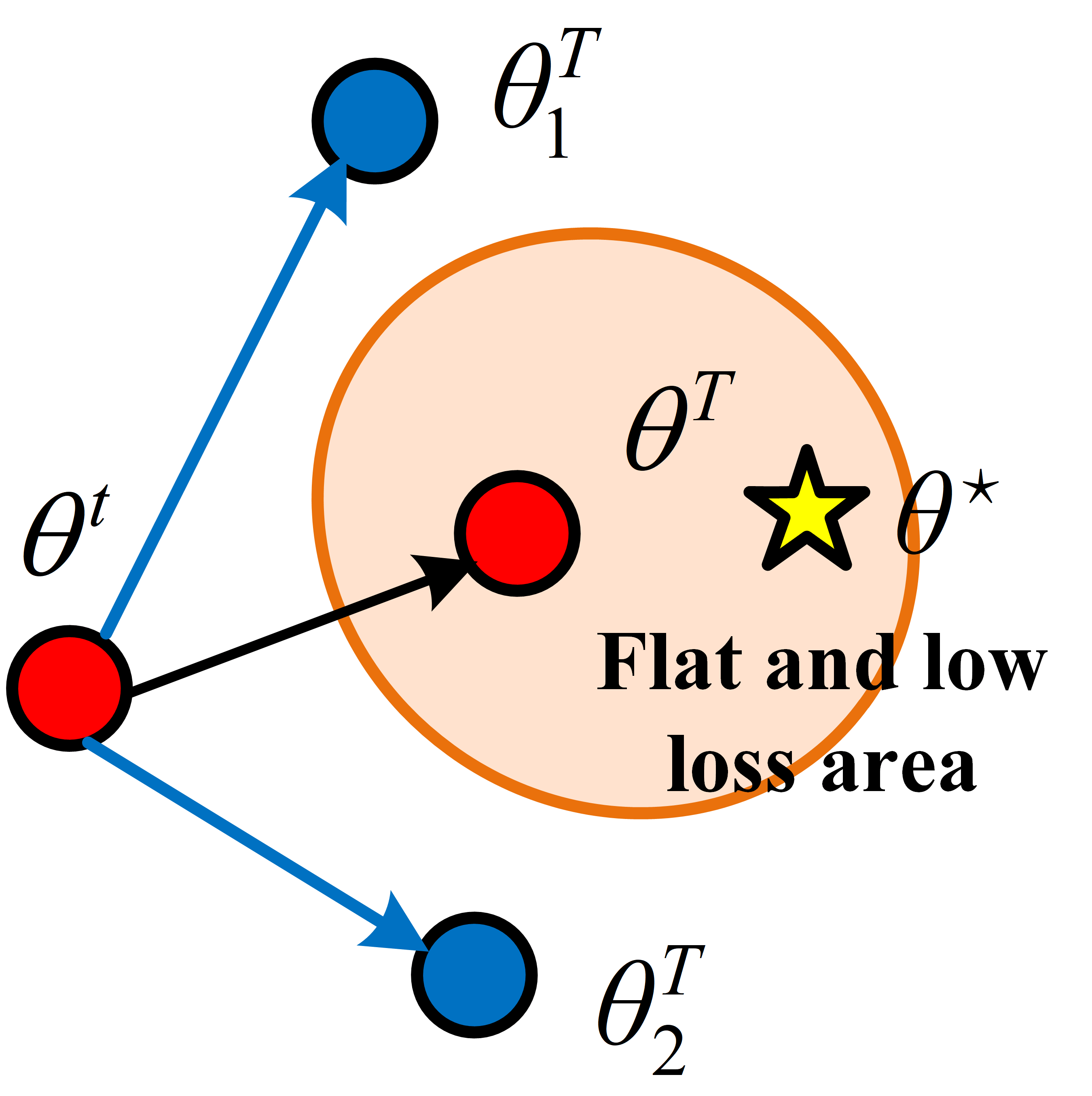}}
	\end{minipage}
	\begin{minipage}[b]{0.17\textwidth}
		\centering
		\subcaptionbox{\texttt{FedSWA} }{\includegraphics[width=\textwidth]{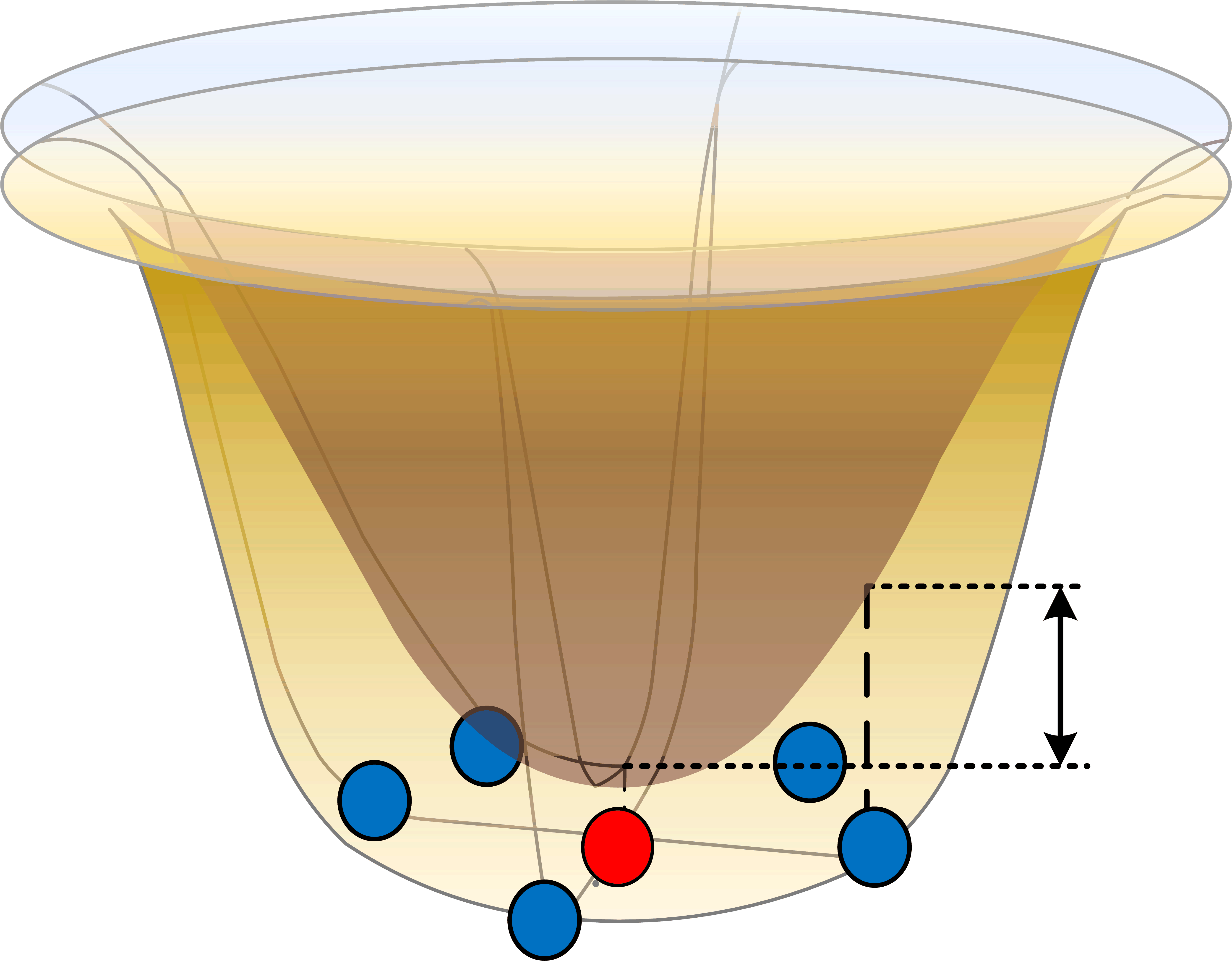}}
	\end{minipage}
	\caption{The training loss surfaces of ResNet-18 trained by FedAvg (a), FedSAM (b), and our FedSWA (c) on CIFAR-100 with Dirichlet-0.1 (high data heterogeneity).  The test accuracies of FedAvg, FedSAM, and \texttt{FedSWA} are 45.8\%, 40.1\%, and 50.3\%, respectively, and  FedSAM performs the worst and fails in this case. (d) highlights the sharpness minimizing conflicts due to discrepancies between local and global loss landscapes caused by data heterogeneity. (e) and (f) show that \texttt{FedSWA} can obtain global flat minima in the setting of  high data heterogeneity. The blue dots denote the client model, and the red dot is the server model in (f). The outer layer is the training loss surface, the inner layer is the test loss surface, and the global model generalizes well at a flat minimum in (f). The double arrow denotes the generalization error gap between clients and the server. 
	} 
	\label{figure 1}
\end{figure}
With increasing concerns about data privacy and security, federated learning (FL) has emerged as a potential distributed machine learning paradigm \cite{mcmahan2017communication}. FL allows for collaborative training of a global model across several clients without exchanging raw data, ensuring privacy while exploiting decentralized data. FL algorithms have considerable promise in various areas, including healthcare, finance, and personalized mobile services \cite{rieke2020future,antunes2022federated,byrd2020differentially,wang2024taming,bian2024lora,bian2025fedalt}.

FL confronts the issue of data heterogeneity because of the non-independent and identically distributed (Non-IID)  data across clients \cite{zhang2021fedpd,liu2024fedbcgd}, along with different computational capabilities and network connectivity. These characteristics bring complications, particularly regarding the model's capacity to generalize to  unseen data, which is critical to the practical applicability of FL algorithms. The generalization study assesses the model's capacity to generalize from empirical to theoretical optimal risks, providing vital insights into the efficacy of FL algorithms.

The generalization of many deep learning methods has been extensively studied to address the overfitting problem. Finding flat minima in model parameters has been demonstrated to improve generalization performance. The two methods, Stochastic Weight Averaging (SWA) \cite{izmailov2018averaging} and Sharpness-Aware Minimization (SAM) \cite{foret2020sharpness}, have received significant attention. In distributed learning, particularly FL, data heterogeneity always enables the global model to converge to a sharp local minimum \cite{fan2024locally}, and generalization of sharp local minimum is poor, as shown in Figure  \ref{figure 1}(a). 

To tackle this challenge, several studies have introduced SAM to local client optimization, which aims to enhance generalization by finding a flat minimum through minimizing the local loss after perturbation. \citet{qu2022generalized} adopted  SAM to FL and introduced FedSAM, and they also proposed  MoFedSAM, which incorporates local momentum acceleration. \citet{caldarola2022improving} applied  adaptive SAM optimizer in FL local training and proposed FedASAM. FedGAMMA \cite{dai2023fedgamma} improves FedSAM by using the SCAFFOLD's variance reduction technique \cite{karimireddy2020scaffold}. FedSpeed \cite{sun2023fedspeed}  utilized   the gradient computed at the SAM-based perturbed weight. \citet{fan2024locally} proposed FedLESAM  locally estimates the direction of global perturbation on the client side as the difference between global models received in the previous active and current rounds.  However, adequate experimental results show that FedSAM is extremely ineffective in the setting of highly heterogeneous data (see in Figure  \ref{figure 1}(b)).  Moreover, \citet{leerethinking} also pointed out that SAM in local clients can find a local flat minimum, not a global flat minimum, and FedSAM fails to generalize in the setting of high data heterogeneity.
In addition, \citet{kaddour2022flat} noted that SAM can be close to sharp directions,  which is unfriendly to the federated averaging algorithm.

To address the above shortcomings of  SAM-type algorithms as local optimizers, this paper introduces SWA to improve  generalization  of FL with highly  heterogeneous data. SWA improves the model's generalization capability by averaging different model weights at the final stage of training, thereby locating a model lying in a flat region of the loss landscape. Therefore, we propose a new Stochastic  Weight Averaging Federated Learning (\texttt{FedSWA}) algorithm, which aims to find flatter minima in the setting of highly  heterogeneous data, as shown in Figure \ref{figure 1}(e,f). In other words, the goal of  \texttt{FedSWA} is to find global flat minima, avoiding the drawbacks of  FedSAM, which finds local flat minima instead of global flat minima.
Moreover, SWA is computationally more efficient than  SAM since SWA does not require additional forward and backward propagations to compute the perturbations in SAM.  In this paper, we will revisit SAM and SWA on FL to improve the generalization ability on highly heterogeneous data.

In Figure  \ref{figure 1}(d), we can observe that FedSAM finds a global flat minima  but a high loss surface, and the generalization to the test loss surface does not work very well. This is because  FedSAM focuses too much on local flat minima, not global flat minima, causing it to overlook  superior global minima. The  flatness and low loss surface of FedSAM in local client training does not imply that the global model is flat \cite{leerethinking}. Our \texttt{FedSWA} algorithm aims to find a  flat and low global loss surface, which generalizes to the test set much better than FedSAM, as shown in Figure  \ref{figure 1}(c,e). 

While \texttt{FedSWA} enhances the global model's generalization compared to FedSAM, it does not guarantee a  consistent flat and low minimum for the local model to agree with the global model in the setting of highly heterogeneous data. To seamlessly integrate the smoothness of both local and global models, we present a more advanced algorithm within our framework, called Federated Learning via Momentum-Based Stochastic Controlled Weight Averaging (\texttt{FedMoSWA}). 
The key idea of  \texttt{FedMoSWA} is inspired by the momentum-based variance reduction \cite{2020Momentum}.
Intuitively, \texttt{FedMoSWA} estimates the update direction for the server variable $\boldsymbol{m}$ and the  client variable  $\boldsymbol{c_i}$.
Their difference, ($\boldsymbol{m}-\boldsymbol{c_i}$), is used to correct the estimation of the client drift of local updates, maintaining the consistency of local updates with global updates. This strategy successfully overcomes data heterogeneity. 

\definecolor{LightRed}{RGB}{255,182,193} 
\definecolor{LightBlue}{RGB}{173, 216, 230}
\begin{table*}[tb]
	\caption{Comparison of generalization  and optimization errors in the non-convex setting. 
		$\sigma$ is the variance of stochastic gradients, $\sigma_g$ is data heterogeneity, $s$ is the number of participating clients, and $T$ is the number of communication rounds. $K$ is the number of local iterations. $\ell(\boldsymbol{x},\boldsymbol{y}; \boldsymbol{\theta})$ is  $L$-Lipschitz continous and  $\beta$-smooth. 
		$F^{\star}$ is a minimum value of Problem (\ref{eq 1}) below, $F:=F_0-F^{\star}$, $\overline{c}\! >\tilde{c}\!=\!1+\left(2+1/KT\right)^{K-1}/T\gg 1$, which means the generalization error of \texttt{FedSWA} is smaller than that of FedSAM. The generalization error of \texttt{MoFedSWA} is smaller than that of \texttt{FedSWA}.  
		The optimization bounds of FedSAM and MoFedSAM are proved in \cite{patel2022towards}. Generalization bounds without client sampling.
	}
	\setlength{\tabcolsep}{10pt}
	\centering
	\begin{threeparttable}          
		\begin{tabular}{*3{c}}\midrule[1.5pt]	         
			\text { Algorithm } & \text { Generalization error} & \text { Optimization error}  \\ \midrule
			
			\text {FedSAM} \cite{qu2022generalized} &$
			\mathcal{O}\big(\frac{L}{m n \beta}e^{1+\frac{1}{T}}(\overline{c}L+ \overline{c}\textcolor{red}{\sigma_g}+\overline{c} \sigma) \big)$
			
			&$\mathcal{O}\left(\frac{\beta F}{\sqrt{T K s}}+\frac{\sqrt{K} \textcolor{red}{\sigma_g}^2}{\sqrt{T s}}+\frac{L^2 \sigma^2}{T^{3 / 2} K}+\frac{L^2}{T^2} \right)$  \\
			\text { MoFedSAM } \cite{qu2022generalized} &$
			\mathcal{O}\big(\frac{L}{m n \beta}e^{1+\frac{1}{T}}(\overline{c} L+ \overline{c}\textcolor{red}{\sigma_g}+\overline{c}\sigma) \big)$
			
			&$\mathcal{O}\left(\frac{\beta L F}{\sqrt{T K s}}+\frac{\beta \sqrt{K} \textcolor{red}{\sigma_g}^2}{\sqrt{T s}}+\frac{L^2 \sigma^2}{T^{3 / 2} K}+\frac{\sqrt{K} L^2}{T^{3 / 2} \sqrt{s}} \right)$ \\	
			\rowcolor{LightBlue}
			\text{\texttt{FedSWA} (ours)}  &$
			\mathcal{O}\big(\frac{L}{m n \beta}e^{1+\frac{1}{T}}(\tilde{c} L+ \tilde{c}\textcolor{red}{\sigma_g}+\tilde{c} \sigma) \big)$   
			& $\mathcal{O}\left(\frac{\beta \left(\sigma+\sqrt{K }\textcolor{red}{\sigma_g}\right)\sqrt{F}}{\sqrt{T K s}}+\frac{F^{2 / 3}\left(\beta \textcolor{red}{\sigma_g}^2\right)^{1 / 3}}{T^{2 / 3}}+\frac{ \beta F}{T} \right)$ \\	
			\rowcolor{LightRed}
			\text {\texttt{FedMoSWA} (ours) } &$  \mathcal{O}\big(\frac{L}{m n \beta}e^{1+\frac{1}{T}}(\tilde{c} L+ \textcolor{red}{\sigma_g}+\tilde{c} \sigma) \big)$ &$\mathcal{O}\left(\frac{\sigma \sqrt{F}}{\sqrt{T K s}}\left(\sqrt{1+\frac{s}{\alpha^2}}\right)+\frac{\beta F}{T}\left(\frac{m}{s}\right)^{\frac{2}{3}} \right)$\\
			\midrule[1.5pt]			
		\end{tabular}
	\end{threeparttable} 
	\label{table 1}
\end{table*}

Despite that there is  considerable research on optimization error and convergence in FL, the analysis of generalization remains significantly challenging. To theoretically characterize generalization performance, researchers have developed a variety of analysis techniques, including uniform convergence methods \cite{vapnik1998statistical}, operator approximation techniques \cite{smale2007learning}, information-theoretic tools \cite{russo2019much}, and algorithmic stability analysis \cite{rogers1978finite,hardt2016train}. Among these methods, stability analysis has attracted particular attention due to its theoretical guarantees that are independent of the capacity measurement of the hypothesis function space. It is also widely applicable and sensitive to data distribution. This  paper proposes an analytical framework based on uniform stability, which is particularly useful for studying generalization errors by considering the dependence on specific FL algorithms and data heterogeneity.


\textbf{Contributions:}  
We propose a new FL generalization analysis framework and develop two efficient algorithms to  improve the generalization and optimization errors.  Our main contributions can  be summarized as follows:\\
\textbf{$\bullet$ New FL algorithm inspired by SWA with better generalization:} We revisit the role of SWA and SAM for federated learning and discuss the superiority of  SWA with heterogeneous data. 
To generalize the global model, we first propose a novel algorithm, \texttt{FedSWA}, to improve the generalization of FL, which is better than FedSAM.\\
\textbf{$\bullet$  Momentum-based stochastic controlled weight averaging algorithm:} To help local models  find  consistent flat minima aligned with the global model, we develop a new \texttt{FedMoSWA} algorithm. We also propose one  momentum-based stochastic controlled method, which is  designed to increase FL generalization and overcome data heterogeneity. Our theoretical result shows that \texttt{FedMoSWA} provides theoretically better upper bounds for  generalization error than FedSAM. Experimentally, \texttt{FedMoSWA} is better than both FedSAM and his variants.\\
\textbf{$\bullet$ Novel FL generalization analysis framework:} 
We propose a new  FL generalization error analysis framework, which  can account for the effect of data heterogeneity on generalization error. Our theoretical results show  that the generalization bound of our \texttt{FedSWA} is $\mathcal{O}\big(\frac{L}{m n \beta}e^{\frac{1}{T}+1}(\tilde{c} L+ \tilde{c}\sigma_g+\tilde{c} \sigma) \big)$, which is correlated with data heterogeneity and is superior to  that of  FedSAM.
We prove that the generalization  error of  \texttt{FedMoSWA} is   $\mathcal{O}\big(\frac{L}{m n\beta} e^{\frac{1}{T}+1}(\tilde{c} L+ \sigma_g+\tilde{c} \sigma) \big)$, which is better than  those of FedSAM and  \texttt{FedSWA}, where  $\tilde{c}=1+\left(2+1/KT\right)^{K-1}/T\gg 1$.\\ 

\section{Related work}
\textbf{$\bullet$ Heterogeneity Issues in FL:} In past years, various strategies have been proposed to solve
the heterogeneity issues in FL.  FedAvgM \cite{hsu2019measuring} is a proposed method that introduces momentum terms during global model updating. FedACG \cite{kim2024communication} improves the inter-client interoperability by the server broadcasting a global model with a prospective gradient consistency. SCAFFOLD \cite{karimireddy2020scaffold} uses SAGA-like control variables to mitigate client-side drift, which can be regarded as adopting the idea of variance reduction at the client side. Our \texttt{FedMoSWA} method is based on momentum-based variance reduction, which is different from SCAFFOLD and other momentum methods.\\
\textbf{$\bullet$ Generalization Analysis of FL:}
\citet{hu2022generalization} provided a systematic analysis of the generalization error of FL in the two-level framework, which captures the missed participating gap. \citet{sefidgaran2022rate} used tools from rate-distortion theory to establish new upper bounds on the generalization error of statistical distributed learning algorithms. 
\citet{sun2024understanding} analyzed the generalization performance of federated learning through algorithm stability, but did not propose  an improved algorithm to overcome data heterogeneity.\\ 
\textbf{$\bullet$ Uniform Stability Generalization Analysis of Algorithms:}
Uniform stability \cite{bousquet2002stability} is a classical tool to analyze the generalization error of an algorithm. For instance, \citet{hardt2016train} and \citet{zhang2022stability} analyzed the generalization of stochastic gradient descent (SGD) \cite{bottou2010large} via uniform stability.  \citet{yuan2019stagewise} investigated  stagewise SGD and showed the advantages of the statewise strategy.\\
\textbf{$\bullet$ SWA:} The concept of averaging weights can be traced back to early efforts to accelerate the convergence of SGD \cite{polyak1992acceleration,kaddour2022stop}. SWA is motivated by an observation about SGD’s behavior when training neural networks: although SGD frequently explores regions in the weight space associated with high-performing models, it seldom reaches the central points of this optimal set. By averaging parameter values, SWA guides the solution closer to the centroid of this space  \cite{izmailov2018averaging}.

\section{Proposed Algorithm}
\subsection{Problem Setup}
FL aims to optimize global model with the collaboration local clients, i.e., minimizing the following population risk:
\begin{align}
	F(\boldsymbol{\theta})=\frac{1}{m} \sum_{i=1}^m\left(F_i(\boldsymbol{\theta}):=\mathbb{E}_{\zeta_i\sim \mathcal{D}_i}\left[F_i\left(\boldsymbol{\theta} ; \zeta_i\right)\right]\right).
	\label{eq 1}
\end{align}
The function $F_i$ represents the loss function on client $i$. $\mathbb{E}_{\zeta_i \sim \mathcal{D}_i}[\cdot]$ denotes the conditional expectation with respect to  the sample $\zeta_i$. Generally, the unattainability of the local population risk $F_i\left(\boldsymbol{\theta}\right)$ forces us to train the model by minimizing the following empirical approximation of population risk $F_{S_i}\left(\boldsymbol{\theta}\right)$. For the generalization problem, we consider:
\begin{align}
	& F_{\mathcal{S}}\big(\boldsymbol{\theta}\big) \! =\! \frac{1}{m} \sum_{i=1}^m F_{\mathcal{S}_i}\big(\boldsymbol{\theta}\big)\!\! =\!\! \frac{1}{mn}  \sum_{i=1}^m \sum_{j=1}^n \ell\big(\boldsymbol{x}_{i,j}, \boldsymbol{y}_{i,j} ; \boldsymbol{\theta}\big),
	\label{eq 2}
\end{align}
where $F_S(\boldsymbol{\theta})$ and $F_{S_i}\left(\boldsymbol{\theta}\right)$ indicate the empirical risks for the global model and the local model of the $i$-th client, respectively. $F_{\mathcal{S}_i}(\boldsymbol{\theta}) = \frac{1}{n} \sum_{i=1}^n \ell\left(\boldsymbol{x}_{i,j}, \boldsymbol{y}_{i,j} ; \boldsymbol{\theta}\right)$ represents the empirical risk loss function on client $i$. 
Given a dataset $\mathcal{S}=\left\{\left(\boldsymbol{x}_{i,j}, \boldsymbol{y}_{i,j}\right)\right\}_{i=1,j=1}^{i=m,j=n}$, $\mathcal{S}=\mathcal{S}_1 \cup \cdots \cup\mathcal{S}_i\cup \cdots \cup \mathcal{S}_m$, where the dataset $\mathcal{S}_i=\left\{\left(\boldsymbol{x}_{i,j}, \boldsymbol{y}_{i,j}\right)\right\}_{j=1}^n$, and $\left(\boldsymbol{x}_{i,j}, \boldsymbol{y}_{i,j}\right)$ is drawn from an unknown distribution $\mathcal{D}_i$. Assume $\left\{\mathcal{D}_1, \cdots, \mathcal{D}_m\right\}$ are independently sampled from $\mathcal{D}$ according to $P$ \cite{hu2022generalization}.
$m$ is the number of clints, $n$ is the number of training samples per client.
In this paper, we assume that $F_{\mathcal{S}}(\boldsymbol{\theta})$ has a minimum value $F_{\mathcal{S}}^{\star}$, when $\boldsymbol{\theta}=\boldsymbol{\theta}^{\star}$.

\definecolor{LightRed}{RGB}{255,182,193} 
\definecolor{LightBlue}{RGB}{173, 216, 230}
 

\subsection{Our FedSWA  Algorithm}
\vspace{-2mm}
To enhance the generalization of FL, our \texttt{FedSWA} algorithm uses SWA for local training and weight aggregation. In the client-side update, this paper proposes a local learning rate decay strategy instead of a constant learning rate. In the $t$-th round of local training,
\vspace{-3mm}
\begin{equation}
	\eta_0^t=\eta_l, \eta_K^t=\rho \eta_l , \eta_k^t=\eta_l\left(1-\frac{k}{K}\right)+\frac{k}{K} \rho \eta_l,
	\label{eq3}
\end{equation}
\vspace{-2mm}
\begin{equation}
	\boldsymbol{\theta}_{i, k+1}^{(t)} = \boldsymbol{\theta}_{i, k}^{(t)}-\eta_k^tg_i\left(\boldsymbol{\theta}_{i, k}^{(t)}\right),
	\label{eq 3}
\end{equation}
where $
g_i\left(\boldsymbol{\theta}_{i, k}^{t}\right)
\!= \!\frac{1}{|\mathcal{B}|} \sum_{(\boldsymbol{x}, \boldsymbol{y}) \in \mathcal{B}} \nabla \ell\left(\boldsymbol{x}_{i,j}, \boldsymbol{y}_{i,j} ; \boldsymbol{\theta}_{i,k}^{t}\right)$, $0 \leq \rho \leq 1$. Here $\mathcal{B}$ is the sampled minibatch at the $(t, k)$-th iteration and $\eta_l$ is the initial local learning rate. $k$ is the $k$-th local iteration and $K$ is the total number of local iterations. $\rho$ is the coefficient of local learning rate decay, the smaller $\rho$, the faster the local learning rate decays.
Inspired by the LookAhead \cite{zhang2019lookahead}, the server aggregation employs an exponential moving average (EMA) to aggregate the historical weights instead of average aggregation.  The specific steps of the algorithm are described in Algorithm \ref{algorithm 1}. In Eq.(\ref{eq3}), the local learning rate decreases from $\eta_l$ to $\rho\eta_l$. After model aggregation,  the learning rate is adjusted to the initial learning rate $\eta_l$ to restart local training. \citet{smith2017cyclical,gotmare2018closer} observed that restoring the initial large learning rate (learning rate restart)  helps us to jump out of the worse local minimum and find a better local minimum. We  experimentally demonstrate that restarting the learning rate is  effective for federated learning (see Table \ref{tab:4}).
\begin{algorithm}[tb]
	
	\caption{\fcolorbox{LightBlue}{LightBlue}{\texttt{FedSWA}},\fcolorbox{LightRed}{LightRed}{ \texttt{FedMoSWA} }algorithm.}
	\begin{algorithmic}[1] 
		\STATE \textbf{Input:} $\lambda$, $\rho$, initial server model $\boldsymbol{\theta}_{0}$, number of clients $N$, number of communication rounds $T$, number of local iterations $K$, local learning rate $\eta_l$.
		\FOR{$t=0,...,T$}
		\STATE {\fcolorbox{LightBlue}{LightBlue}{Communicate $(\boldsymbol{\theta}_{t-1})$ to selected clients $i \in [s]$.}}		
		\STATE {\fcolorbox{LightRed}{LightRed}{Communicate
				$(\boldsymbol{\theta}_{t-1},\boldsymbol{m})$ to selected clients $i \in [s]$.}}
		\FOR{$i=1, \ldots, s$ clients in parallel}	
		\FOR{$k=0, \ldots, K$ local update}	
		\STATE  Compute mini-batch gradient $g_i\left(\boldsymbol{\theta}_{i,k}^{t}\right)$.
		\STATE$\eta_k^t=\eta_l\left(1-\frac{k}{K}\right)+\frac{k}{K} \rho \eta_l $.
		\STATE{\fcolorbox{LightBlue}{LightBlue}{$\boldsymbol{\theta}_{i,k+1}^{t}\leftarrow \boldsymbol{\theta}_{i,k}^{(t)}-\eta_{k}^{t}\left(g_i\left(\boldsymbol{\theta}_{i,k}^{t}\right)\right)$.}}
		\STATE
		{\fcolorbox{LightRed}{LightRed}{$\boldsymbol{\theta}_{i,k+1}^{t}\leftarrow \boldsymbol{\theta}_{i,k}^{(t)}-\eta_{k}^{t}\left(g_i\left(\boldsymbol{\theta}_{i,k}^{t}\right)-\boldsymbol{c_i}+\boldsymbol{m}\right)$.}}
		\ENDFOR
		\STATE  
		{\fcolorbox{LightBlue}{LightBlue}{$\text {Communicate }$$\left(\boldsymbol{\theta}_{i,K}^{t}\right)$	to server.}}
		\STATE{\fcolorbox{LightRed}{LightRed}{$\boldsymbol{c}_i^{+} \leftarrow$ (i) $ g_i(\boldsymbol{x})$ or (ii) $\boldsymbol{c}_i-\boldsymbol{m}+\frac{1}{\sum_{k} \eta_k^{t}}\left(\boldsymbol{\theta}_{t-1}-\boldsymbol{\theta}_{i,k}^{t}\right)$.}}
		\STATE  
		{\fcolorbox{LightRed}{LightRed}{$\text {Communicate }$$\left(\boldsymbol{\theta}_{i,K}^{t}, \boldsymbol{c}_i^{+}-\boldsymbol{m}\right)$	to server, $\boldsymbol{c_i} \leftarrow \boldsymbol{c_i^{+}}$.}}		
		\ENDFOR
		\STATE
		{\fcolorbox{LightRed}{LightRed}{$\boldsymbol{m}\leftarrow \boldsymbol{m}+\gamma \frac{1}{s} \sum_{i \in [s]} \Delta\boldsymbol{c}_i$.}}
		\STATE $\boldsymbol{v}_t=\frac{1}{s} \sum_{i=1}^s \boldsymbol{\theta}_{i, K}^{t}$, $\boldsymbol{\theta}_t=\boldsymbol{\theta}_{t-1}+\alpha\left(\boldsymbol{v}_t-\boldsymbol{\theta}_{t-1}\right)$.
		\ENDFOR
	\end{algorithmic}
	\label{algorithm 1}
\end{algorithm}

\textbf{Compared with  FedAvg \cite{mcmahan2017communication} and FedSAM \cite{qu2022generalized}.}  FedAvg and FedSAM adopt constant local learning rate and model averaging to aggregate, while our \texttt{FedSWA} uses cyclical learning rate and exponential moving average aggregation which help \texttt{FedSWA} jump out of worse flat minima and look for better flat minima as discussed in \cite{smith2017cyclical,gotmare2018closer}. The test accuracy of \texttt{FedSWA} is also better than those of FedAvg and FedSAM (see Tables \ref{table 2} and \ref{table 3}).\\
\begin{figure}[bt]
	\begin{minipage}[b]{0.23\textwidth}
		\centering
		\subcaptionbox{\texttt{FedSWA} }{\includegraphics[width=\textwidth]{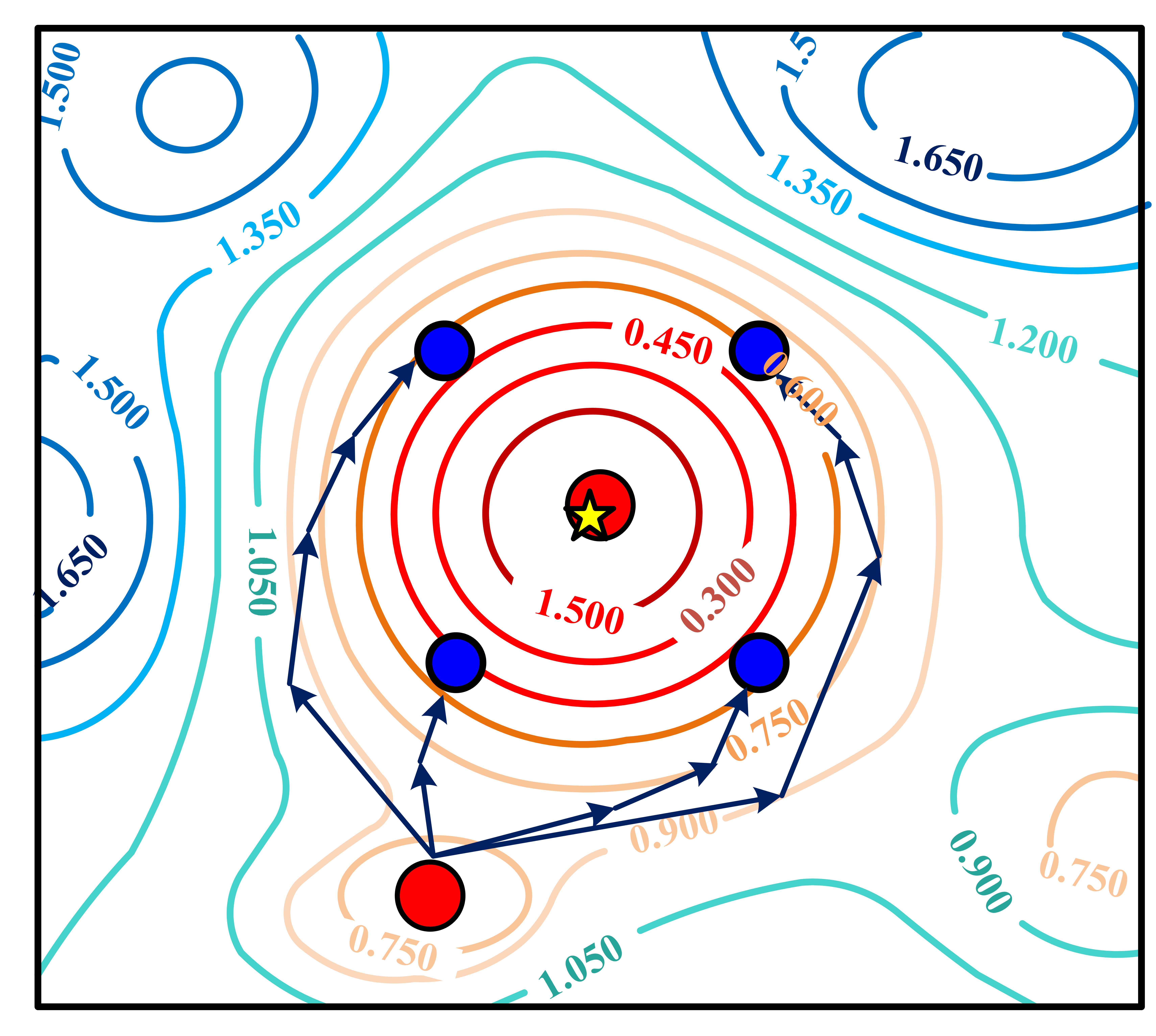}}
	\end{minipage}
	\begin{minipage}[b]{0.23\textwidth}
		\centering
		\subcaptionbox{\texttt{FedMoSWA} }{\includegraphics[width=\textwidth]{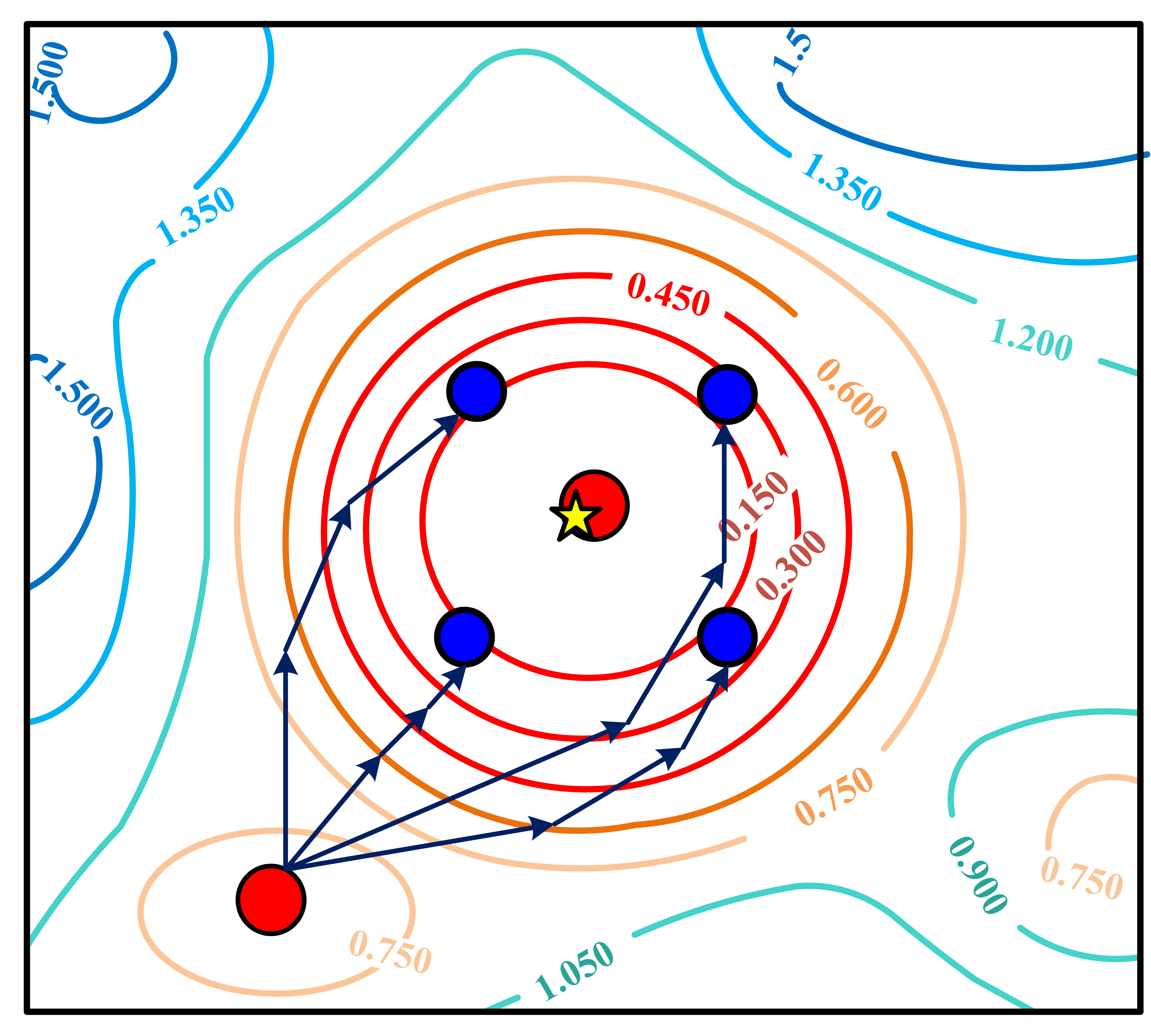}}
	\end{minipage}
	\caption{(a) and (b) indicate that due to data heterogeneity, the local and global models of \texttt{FedSWA} differ too much to find a less flatter minima, while our \texttt{FedMoSWA}  aligns the local and global models and finds a flatter minima.}
	\label{figure 2}
\end{figure}
\subsection{Our FedMoSWA  Algorithm}

Inspired by momentum-based variance reduction  \cite{2020Momentum}, which achieves variance reduction through the use of a variant of the momentum term, we propose a new \texttt{FedMoSWA} algorithm. Different from our \texttt{FedSWA}, the local update of \texttt{FedMoSWA} becomes 
\begin{align}
	\boldsymbol{\theta}_{i, k+1}^{(t)} = \boldsymbol{\theta}_{i, k}^{(t)}-\eta_k^t\left(g_i\left(\boldsymbol{\theta}_{i, k}^{(t)}\right)-\boldsymbol{c}_{\boldsymbol{i}}+\boldsymbol{m}\right),
	\label{eq 3}
\end{align}
where  $\boldsymbol{c}_i$ is the client control variable, and $\boldsymbol{m}$ is the server control variable. 
We provide two options to update $	\boldsymbol{c}_i^{+}$:
\vspace{-2mm}
\begin{align}
	\boldsymbol{c}_i^{+} \leftarrow \begin{cases}\text {I.}&g_i(\boldsymbol{\theta}_t), \text { or } \\ \text {II.} & \boldsymbol{c}_i-\boldsymbol{m}+\frac{1}{\sum_{k} \eta_k^{t}}\big(\boldsymbol{\theta}_{t-1}-\boldsymbol{\theta}_{i, K}^{(t)}\big)\end{cases}.
\end{align}
Depending on the situation, Option I may be more stable than  II, but  II is less computationally expensive and usually sufficient (we use Option II for our experiments). In the local update in Eq.\ (\ref{eq 3}), our proposed  \texttt{FedMoSWA} algorithm has global control variable $\boldsymbol{m}$, which is update by   
\vspace{-2mm}
\begin{align}
	& \boldsymbol{m} \leftarrow \boldsymbol{m}+\gamma \frac{1}{s} \sum_{i \in \mathcal{S}} \left(\boldsymbol{c}_i^{+}-\boldsymbol{m}\right).
\end{align}

\textbf{Compared with SCAFFOLD \cite{karimireddy2020scaffold}.} The SCAFFOLD algorithm has a global variable, which is updated as: $\boldsymbol{c}=\boldsymbol{c}\!+\!\frac{1}{m} \sum_{i \in \mathcal{S}}\left(\boldsymbol{c}_i^{+}-\boldsymbol{c}_i\right)$, but the global control variable $\boldsymbol{c}$ has an update delay problem. This is due to the fact  that the update of $\boldsymbol{c}$ aggregates the old $\boldsymbol{c}_i$ and the new $\boldsymbol{c}_i$ by using the same weights. When the client participation rate is low, the old $\boldsymbol{c}_i$ is too different from $\boldsymbol{c}$, which instead has a bad effect on the update of $\boldsymbol{c}$.
Our proposed \texttt{FedMoSWA} algorithm has global variables updated with momentum
as $\boldsymbol{m}\!=\!\boldsymbol{m}+\gamma\frac{1}{s} \sum_{i \in \mathcal{S}}\left(\boldsymbol{c}_i^{+}\!-\!\boldsymbol{m}\right)$, which gives higher weight to the latest uploaded $\boldsymbol{c}_i$ and lower weight to the old $\boldsymbol{c}_i$, instead of giving the same weight to all $\boldsymbol{c}_i$ as that SCAFFOLD does. 

\textbf{Compared with FedSWA.} In Figure \ref{figure 2}, we can observe that the local model of our \texttt{FedMoSWA} tends to be more close to the global optimal  than \texttt{FedSWA}.  Influenced by the control variables, \texttt{FedMoSWA}  has fewer  differences between clients, and minima are flatter compared to \texttt{FedSWA}.

\textbf{Compared with MoFedSAM \cite{qu2022generalized}.} MoFedSAM is based on FedSAM and presents a momentum-based mechanism to accelerate convergence and improve performance in data heterogeneity. \texttt{FedMoSWA} introduces a momentum-based variance reduction mechanism in  \texttt{FedSWA} to accelerate convergence in the setting of high data heterogeneity. Theoretically, our  \texttt{FedMoSWA} also outperforms MoFedSAM.
\section{FL Generalization via uniform stability}

\textbf{Excess Risk Error Decomposition.} One often minimizes the empirical risk $F_{\mathcal{S}}(\boldsymbol{\theta})=\frac{1}{m} \frac{1}{n} \sum_{i=1}^m \sum_{i=1}^n \ell\left(\boldsymbol{x}_{i,j}, \boldsymbol{y}_{i,j} ; \boldsymbol{\theta}\right)$ in Eq.(\ref{eq 2}). In client $i$ via a randomized algorithm $\mathcal{A}_i$, e.g., SGD, to find an estimated optimum $\boldsymbol{\theta}_{\mathcal{A}_i, \mathcal{S}_i} \approx \operatorname{argmin}_{\boldsymbol{\theta}} F_{\mathcal{S}_i}(\boldsymbol{\theta})$. In server, we use FL optimizer $\mathcal{A}$ , e.g., FedAvg, to find an estimated optimum $\boldsymbol{\theta}_{\mathcal{A}, \mathcal{S}} \approx \operatorname{argmin}_{\boldsymbol{\theta}} F_{\mathcal{S}}(\boldsymbol{\theta})$.  However, this empirical solution $\boldsymbol{\theta}_{\mathcal{A}, \mathcal{S}}$ differs from the desired optimum $\boldsymbol{\theta}_{\mathcal{D}}^{\star}$ of the population risk $F_{\mathcal{D}}(\boldsymbol{\theta})=\mathbb{E}_{(\boldsymbol{x}, \boldsymbol{y}) \sim \mathcal{D}}[\ell\left(\boldsymbol{x}, \boldsymbol{y} ; \boldsymbol{\theta}\right)]$,
\begin{align}
	\boldsymbol{\theta}_{\mathcal{D}}^{\star} \in \operatorname{argmin}_{\boldsymbol{\theta}} F_{\mathcal{D}}(\boldsymbol{\theta})=\mathbb{E}_{(\boldsymbol{x}, \boldsymbol{y}) \sim \mathcal{D}}[\ell\left(\boldsymbol{x}, \boldsymbol{y} ; \boldsymbol{\theta}\right)] .
\end{align}

To understand the generalization ability of FL algorithms, we want to know: what performance of the estimated optimum $\boldsymbol{\theta}_{\mathcal{A}, \mathcal{S}}$ can achieve on the test data $(\boldsymbol{x}, \boldsymbol{y}) \sim \mathcal{D}$ in FL. To answer this question, we analyze the test error $\mathbb{E}_{\mathcal{A}, \mathcal{S}}\left[F_{\mathcal{D}}\left(\boldsymbol{\theta}_{\mathcal{A}, \mathcal{S}}\right)\right]$ of $\boldsymbol{\theta}_{\mathcal{A}, \mathcal{S}}$ via investigating the well-known excess risk error $\varepsilon_{\text {exc }}$ defined as
$$
\begin{aligned}
	\begin{split}
		& \varepsilon_{\mathrm{exc}}=\mathbb{E}_{\mathcal{A}, \mathcal{S}}\left[F_{\mathcal{D}}\left(\boldsymbol{\theta}_{\mathcal{A}, \mathcal{S}}\right)\right]-\mathbb{E}_{\mathcal{A}, \mathcal{S}}\left[F_{\mathcal{S}}\left(\boldsymbol{\theta}_{\mathcal{S}}^{\star}\right)\right] \\
		& =\underbrace{\mathbb{E}\left[F_{\mathcal{D}}\left(\boldsymbol{\theta}_{\mathcal{A}, \mathcal{S}}\right)-F_{\mathcal{S}}\left(\boldsymbol{\theta}_{\mathcal{A}, \mathcal{S}}\right)\right]}_{\varepsilon_{\mathrm{gen}}}+\underbrace{\mathbb{E}\left[F_{\mathcal{S}}\left(\boldsymbol{\theta}_{\mathcal{A}, \mathcal{S}}\right)-F_{\mathcal{S}}\left(\boldsymbol{\theta}_{\mathcal{S}}^{\star}\right)\right]}_{\varepsilon_{\mathrm{opt}}},
	\end{split}
\end{aligned}
$$
where $\boldsymbol{\theta}_{\mathcal{S}}^{\star}\!\!\in\!\!\operatorname{argmin}_{\boldsymbol{\theta}} F_{\mathcal{S}}(\boldsymbol{\theta})$ is the optimum of  $F_{\mathcal{S}}$. 
The optimization error $\varepsilon_{\text {opt }}\!\!=\!\! \mathbb{E}_{\mathcal{A}, \mathcal{S}}\left[F_{\mathcal{S}}\left(\boldsymbol{\theta}_{\mathcal{A}, \mathcal{S}}\right)-F_{\mathcal{S}}\left(\boldsymbol{\theta}_{\mathcal{S}}^{\star}\right)\right]$ denotes the difference between the exact optimum $\boldsymbol{\theta}_{\mathcal{S}}^{\star}$ and the estimated solution $\boldsymbol{\theta}_{\mathcal{A}, \mathcal{S}}$, and generalization error $\varepsilon_{\text {gen }} = $ $\mathbb{E}_{\mathcal{A}, \mathcal{S}}\left[F_{\mathcal{D}}\left(\boldsymbol{\theta}_{\mathcal{A}, \mathcal{S}}\right)-F_{\mathcal{S}}\left(\boldsymbol{\theta}_{\mathcal{A}, \mathcal{S}}\right)\right]$ measures the effects of minimizing empirical risk instead of population risk. 

\textbf{Uniform Stability and Generalization.} One popular approach to analyze generalization error $\varepsilon_{\text {gen }}$ of  algorithm $\mathcal{A}$ is uniform stability \cite{hardt2016train,zhang2022stability}. 
\begin{lemma}(Uniform Stability and Generalization Error) 
	\cite{hardt2016train} We say a randomized algorithm $\mathcal{A}$ is $\epsilon$-uniformly stable if for all datasets $\mathcal{S} \sim \mathcal{D}$ and $\mathcal{S}^{\prime} \sim \mathcal{D}$ where $\mathcal{S}$ and $\mathcal{S}^{\prime}$ differ in at most one sample,
	\vspace{-2mm}
	\begin{align}
		\sup _{(\boldsymbol{x}, \boldsymbol{y}) \sim \mathcal{D}} \mathbb{E}_{\mathcal{A}}\left[\ell\left(\boldsymbol{x},\boldsymbol{y} ; \boldsymbol{\theta}_{\mathcal{A}, \mathcal{S}} \right)-\ell\left(\boldsymbol{x}, \boldsymbol{y}; \boldsymbol{\theta}_{\mathcal{A}, \mathcal{S}^{\prime}} \right)\right] \leq \epsilon .
	\end{align}
	Moreover, if $\mathcal{A}$ is $\epsilon$-uniformly stable, then its generalization error $\varepsilon_{\text {gen }}$, which is defined as $\varepsilon_{\text {gen }}=$ $\left|\mathbb{E}_{\mathcal{A}, \mathcal{S}}\left[F_{\mathcal{D}}\left(\boldsymbol{\theta}_{\mathcal{A}, \mathcal{S}}\right)-F_{\mathcal{S}}\left(\boldsymbol{\theta}_{\mathcal{A}, \mathcal{S}}\right)\right]\right|$ satisfying  $\varepsilon_{\text {gen }} \leq \epsilon$.
\end{lemma}

\textbf{FL generalization error via uniform stability.} Here we aim to use uniform stability to upper bound the FL generalization error. 
Suppose given $n$ samples $\mathcal{S}_i\!=\!\left\{\boldsymbol{z}_{i,1}, \boldsymbol{z}_{i,2}, \cdots, \boldsymbol{z}_{i,n}\right\}$, where $\boldsymbol{z}_{i,j}\!=\!\left(\boldsymbol{x}_{i,j}, \boldsymbol{y}_{i,j}\right)$ is sampled from an unknown distribution $\mathcal{D}_i$, and suppose the generated sample set $\mathcal{S}_i^{(j)}\!=\!\left\{\boldsymbol{z}_{i,1}^{\prime}, \boldsymbol{z}_{i,2}^{\prime}, \cdots, \boldsymbol{z}_{i,n}^{\prime}\right\}\!=\!$ $\left\{\boldsymbol{z}_{i,1}, \boldsymbol{z}_{i,2}, \cdots, \boldsymbol{z}_{i,j-1}, \boldsymbol{z}_{i,j}^{\prime}, \boldsymbol{z}_{i,j+1} \cdots, \boldsymbol{z}_{i,n}\right\}$ which only differs from the set $\mathcal{S}_i$ with the $j$-th sample. $\mathcal{S}\!=\!\left\{\mathcal{S}_i\right\}_{i=1}^m, \mathcal{S}^{\left(j_i\right)}\!=\!\left\{\mathcal{S}_1, \ldots, \mathcal{S}_{i-1}, \mathcal{S}_i^{(j)}, \mathcal{S}_{i+1}, \ldots, \mathcal{S}_m\right\}$. One usually analyzes the stability of an algorithm by replacing one sample in $\mathcal{S}$ by another sample from $\mathcal{D}$, and we get $\mathcal{S}^{\left(j_i\right)}$. Then based on these two sets, one can train the algorithm on $\mathcal{S}$ to obtain different solutions $\boldsymbol{\theta}$ of the function $F_{\mathcal{S}}(\boldsymbol{\theta})$. When using $\mathcal{S}^{\left(j_i\right)}$, we adopt $\boldsymbol{\theta}^{\prime}_t$ and $\boldsymbol{\theta}_{i,k}^{t,\prime}$ to denote their corresponding versions $\boldsymbol{\theta}_t$ and $\boldsymbol{\theta}_{i,k}^{t}$ ($\boldsymbol{\theta}_t$ and $\boldsymbol{\theta}_{i,k}^{t}$ are trained on $\mathcal{S}$ in Algorithm \ref{algorithm 1}).

\section{Theoretical Analysis}


\subsection{Generalization Error Analysis}
We analyze generalization based on  following assumptions: \\
\textbf{Assumption 1.} \textit{The loss function $l(\cdot, z)$ is $\mu$-strongly convex for any sample $z$, $l(\boldsymbol{\theta} ; z) \geq l\left(\boldsymbol{\theta}^{\prime} ; z\right)+\left\langle\nabla l\left(\boldsymbol{\theta}^{\prime} ; z\right), \boldsymbol{\theta}-\boldsymbol{\theta}^{\prime}\right\rangle+\frac{\mu}{2}\left\|\boldsymbol{\theta}-\boldsymbol{\theta}^{\prime}\right\|^2$, for any $z, \boldsymbol{\theta}, \boldsymbol{\theta}^{\prime}$.}\\
\textbf{Assumption 2.} \textit{The loss function $l(\cdot, z)$ is $L$ Lipschitz continous, $\mid l(\boldsymbol{\theta} ; z)-$ $l\left(\boldsymbol{\theta}^{\prime} ; z\right) \mid \leq L\left\|\boldsymbol{\theta}-\boldsymbol{\theta}^{\prime}\right\|$, for any $z,\boldsymbol{\theta}, \boldsymbol{\theta}^{\prime}$.}\\
\textbf{Assumption 3.}
\textit{There exists a constant $\sigma>0$ such that for any $\boldsymbol{\theta}, i \in[m]$, and $z_i \sim D_i$, $\mathbb{E}\left[\left\|\nabla l\left(\boldsymbol{\theta} ; z_i\right)-\nabla F_i(\boldsymbol{\theta})\right\|^2\right] \leq \sigma^2$.}\\
\textbf{Assumption 4.} \textit{The loss function $l(\cdot, z)$ is $\beta$-smooth, $\left\|\nabla l(\boldsymbol{\theta} ; z)-\nabla l\left(\boldsymbol{\theta}^{\prime} ; z\right)\right\| \leq \beta\left\|\boldsymbol{\theta}-\boldsymbol{\theta}^{\prime}\right\|$, for any $z, \boldsymbol{\theta}, \boldsymbol{\theta}^{\prime}$.}\\
\textbf{Assumption 5 (data heterogeneity).} 
\textit{Given $i \in[m]$, there exists a constant $\sigma_{g}>0$, for any $\boldsymbol{\theta}$ we have
	$
	\frac{1}{m} \sum_{i=1}^m\left\|\nabla F_i(\boldsymbol{\theta})-\nabla F(\boldsymbol{\theta})\right\| \leq\sigma_{g}.
	$}\\
Assumption 2, as in  \cite{zhang2022stability}, links model perturbation to stability. Assumptions 3 and 5 in our study address data heterogeneity and algorithm convergence impacts.
	


\begin{theorem}[Generalization Error] Assuming all clients participate in each round with Option I:\\
	\label{T 5.1}\textbf{Strongly convex}: Under Assumptions 1-5, suppose  loss $\ell(\boldsymbol{x} ,\boldsymbol{y}; \boldsymbol{\theta})$ is $\mu$-strongly convex. By setting  $\eta_k^t \leq \frac{1}{\beta KT}$, $\tilde{b}=1+\left(\frac{\mu}{(\beta+\mu) K}\right)^{K-1}\frac{1}{T}$, 
	the generalization error satisfies:
	\vspace{-2mm}
	$$
	\begin{aligned}
		& \texttt{ FedSWA: } \varepsilon_{g e n}\leq \frac{2 L}{m n\beta}  e^{1-\frac{\mu}{(\beta+\mu)T}}\left(\tilde{b} L+\tilde{b} \sigma_g+\tilde{b} \sigma\right). \\
		& \texttt{ FedMoSWA: } \varepsilon_{\text {gen }}\leq \frac{2 L}{m n\beta}  e^{1-\frac{\mu}{(\beta+\mu)T}}\left(\tilde{b} L+\sigma_g+\tilde{b} \sigma\right).
	\end{aligned}
	$$\\
	\textbf{Non-convex}:Under Assumptions 2-5,
	assume $\ell(\boldsymbol{x} ,\boldsymbol{y}; \boldsymbol{\theta})$ is  $\beta$-smooth.  Together with  $\eta_k^t \leq \frac{1}{\beta KT}$,  $\tilde{c}=1+\left(2+\frac{1}{KT}\right)^{K-1}\frac{1}{T}$,  the generalization error  satisfies:
	\vspace{-2mm}
	$$
	\begin{aligned}
		\texttt{ FedSWA: } \varepsilon_{\text {gen }}& \leq \frac{2L}{m n\beta}  e^{\frac{1}{T}+1}\left(\tilde{c} L+\tilde{c}\sigma_g+\tilde{c} \sigma\right).\\
	   \texttt{ FedMoSWA: } \varepsilon_{\text {gen }}&\leq \frac{2L}{m n\beta} e^{\frac{1}{T}+1}\left(\tilde{c} L+\sigma_g+\tilde{c} \sigma\right).
	\end{aligned}
	$$
\end{theorem}

Theorem \ref{T 5.1} shows that for strongly convex and non-convex problems,   $\varepsilon_{\text {opt }}$ has three terms: the first bias term has the  Lipschitz continous constant $L$, the second bias term characterizes the effect of data heterogeneity $\sigma_g$, and the third term reveals the impact of the stochastic gradient noise $\sigma$.

For strongly convex problems, Theorem \ref{T 5.1} shows that  $\varepsilon_{\text {gen }}$ of \texttt{FedMoSWA} can be upper bounded by $\mathcal{O}(\frac{2L}{m n\beta}  e^{1-\frac{\mu}{(\beta+\mu)T}}(\tilde{b} L+\sigma_g+\tilde{b} \sigma))$,  which is better than that of \texttt{FedSWA},  $\mathcal{O}(\frac{2L}{m n\beta}  e^{1-\frac{\mu}{(\beta+\mu)T}}(\tilde{b} L+\tilde{b}\sigma_g+\tilde{b} \sigma))$ ($\tilde{b} \sigma_g \gg \sigma_g$). Compared to \texttt{FedSWA},  \texttt{FedMoSWA} reduces the effect of data heterogeneity on the generalization error. 
In addition, we can observe that the generalization ability of the algorithm is affected by the number of clients $m$ and the amount of data $mn$ on the clients. 
More client participation ($m$) can effectively enhance FL generalization. Similarly, to achieve smaller generalization error, clients should use large amount of data ($mn$). Reducing the number of local iterations $K$ improves generalization capabilities. And greater data heterogeneity impairs federated learning generalization.

For the non-convex generalization error,  Theorem \ref{T 5.1} shows that  $\varepsilon_{\text {gen }}$ of \texttt{FedMoSWA} is better than \texttt{FedSWA}'s $\mathcal{O}(\frac{2L}{m n\beta}  e^{\frac{1}{T}+1}(\tilde{c} L+\tilde{c}\sigma_g+\tilde{c} \sigma))$, and $\tilde{c} \sigma_g \gg \sigma_g$ in Table \ref{table 1}. \texttt{FedMoSWA} and  \texttt{FedSWA}  generalization error is better than FedSAM in Table \ref{table 1}.
\subsection{Optimization Error Analysis }
\begin{theorem}[Optimization Error of  \texttt{FedMoSWA}]
	\label{T 5.2}
	For $\beta$-smooth functions $\left\{F_i\right\}$, which satisfy Assumptions 6-9, and are the same as in the SCAFFOLD \cite{karimireddy2020scaffold} algorithm (see the Appendix for details), the output of \texttt{FedMoSWA} has expected error smaller than $\epsilon$.\\	
	\textbf{Strongly convex}: $ \eta_k^{t} \!\leq\! \min \big(\frac{1}{\beta K \alpha}, \frac{s}{ \mu m K \alpha}\big), T \!\geq \!\max \big(\frac{\beta}{\mu}, \frac{m}{s}\big)$ then
	$$
	\begin{aligned}
		&\mathcal{O}\left(\frac{\sigma^2}{\mu T K s}\left(1+\frac{s}{\alpha^2}\right)+\frac{m \mu}{s} D^2 \exp \left(-\left\{\frac{s}{m}+ \frac{\mu}{\beta}\right\}T\right)\right)
	\end{aligned}.
	$$
	\textbf{Non-convex}: $\eta_k^{t} \leq \frac{1}{ K \alpha \beta}\left(\frac{s}{m}\right)^{\frac{2}{3}}$, and $T \geq 1$, then
	$$
	\mathcal{O}\big(\frac{\sigma \sqrt{F}}{\sqrt{T K s}}\big(\sqrt{1+\frac{m}{\alpha^2}}\big)+\frac{\beta F}{T}\left(\frac{m}{s}\right)^{\frac{2}{3}}\big).
	$$
	Here $D^2:=\left\|\boldsymbol{\theta}^0-\boldsymbol{\theta}^{\star}\right\|^2+\frac{1}{2 m \beta^2} \sum_{i=1}^m\left\|\boldsymbol{c}_i^0-\nabla F_i\left(\boldsymbol{\theta}^{\star}\right)\right\|^2$, $F:=F\left(\boldsymbol{\theta}^0\right)-F\left(\boldsymbol{\theta}^{\star}\right)$.
\end{theorem}

Theorem \ref{T 5.2} shows that for strongly convex and non-convex problems, optimization error $\varepsilon_{\text {opt }}$ has two terms: the first term reveals the impact of the stochastic gradient noise, and the second bias term characterizes the effect of initialization $\boldsymbol{\theta}_0$. The optimization error $\varepsilon_{\text {opt }}$ is not affected by the data heterogeneous parameter $\sigma_g$ and \texttt{FedMoSWA} converges faster as $\alpha$ increases. Convergence is also accelerated when the number of client $s$ increases. Its convergence  is accelerated when the number of local iteration $K$ increases.

\begin{table*}[h]
	\centering
	\setlength{\tabcolsep}{1pt}
	\caption{Comparison of the accuracies at the target rounds ($1000R$) and the communication round to reach target test
		accuracy (Acc.\%) of each algorithm  on Dirichlet-0.6 of the CIFAR10 and CIFAR100 datasets.}
	\begin{tabular}{p{2cm}ccccccccccccccc}
		\midrule[1.5pt]
		
		& \multicolumn{6}{c}{CIFAR10}  & \multicolumn{6}{c}{CIFAR100} \\
		\cmidrule(lr){2-7} \cmidrule(lr){8-13} 
		Method  & \multicolumn{2}{c}{LeNet-5} & \multicolumn{2}{c}{VGG-11} & \multicolumn{2}{c}{ResNet-18}& \multicolumn{2}{c}{LeNet-5} & \multicolumn{2}{c}{VGG-11} & \multicolumn{2}{c}{ResNet-18} \\
		\cmidrule(lr){2-3} \cmidrule(lr){4-5} \cmidrule(lr){6-7}	\cmidrule(lr){8-9} \cmidrule(lr){10-11} \cmidrule(lr){12-13}
		\setlength{\tabcolsep}{0pt}
		& Acc.(\%) & Rounds & Acc.(\%)& Rounds& Acc.(\%) & Rounds  & Acc.(\%)  & Rounds& Acc.(\%)  & Rounds& Acc.(\%)  & Rounds \\
		& 1000R & 78\%  &1000R & 80\% &1000R & 85\%& 1000R & 52\%  &1000R & 53\% &1000R & 55\% \\
		\midrule
		FedAvg  & $79.6_{\pm0.2}$& 574  &$84.0_{\pm0.3}$&313  & $  86.0_{\pm0.2}$&542 &$41.2_{\pm0.2 }$&1000+   &$48.9_{\pm0.4} $&1000+ & $54.2_{\pm0.2} $ &1000+  \\
		
		FedDyn & $80.8_{\pm0.3} $ &390 & $83.1_{\pm0.2}  $&557 & $ 84.6_{\pm0.2}$&1000+  & $40.8_{\pm0.5} $&1000+  & $ 45.2_{\pm0.6}$&1000+ & $ 46.5_{\pm0.6}  $&1000+ \\
		SCAFFOLD   & $80.7_{\pm0.3} $&799  & $86.9_{\pm0.2}$&272  &  $85.9_{\pm0.6}$& 543 & $50.3_{\pm0.2}$&1000+   &$47.9_{\pm0.2}$ &1000+  &  $54.1_{\pm0.2}$ &1000+\\
		
		FedSAM       &$79.9_{\pm0.2}$  &821  &$85.5_{\pm0.2}$ &227   & $83.6_{\pm0.2}$ &1000+   & $47.9_{\pm0.1}$&1000+  &$51.7_{\pm0.3}$  &1000+ &   $47.8_{\pm0.1}$&1000+ \\
		
		MoFedSAM & $82.5_{\pm0.1}  $&483  & $87.4_{\pm0.4 } $&345&$87.0_{\pm0.3} $&569 & $ 49.3_{\pm0.2}$&1000+  & $ 1.0_{\pm0.0}$  &1000+&$60.1_{\pm0.2}$&603\\

		FedLESAM  & $79.7_{\pm0.2} $&496  & $83.5_{\pm0.4}$&169   &  $89.0_{\pm0.1}$&480 & $41.3_{\pm0.3} $  &1000+& $49.2_{\pm0.3}  $&1000+& $52.1_{\pm0.1} $&1000+\\
		
		FedASAM  & $ 81.0_{\pm0.2} $&371& $86.3_{\pm0.3}$&169   &  $88.3_{\pm0.2}$& 542 & $42.1_{\pm0.3}$  &1000+& $53.4_{\pm0.2}   $&908& $49.8_{\pm0.3}$&1000+\\
		
		FedACG & $ 82.9_{ \pm0.1} $& 321  & $84.5_{\pm0.4}$ &249  &$90.3_{\pm0.4}$ &\textbf{237 }& $52.5_{\pm0.2}   $&729  & $ 50.9_{\pm0.2} $&1000+&$61.7_{\pm0.4} $ &518\\
		\rowcolor{LightBlue}
		\texttt{FedSWA}  & $ 81.5_{\pm0.2} $&591  & $86.5_{\pm0.1}$&169   &  $89.5_{\pm0.1}$&538 & $48.5_{\pm0.1 } $  &1000+& $52.5_{\pm0.4}  $&1000+& $59.8_{\pm0.3}$&574\\
		\rowcolor{LightRed}
		\texttt{FedMoSWA}   & $\mathbf{83.8_{\pm0.1} }$& \textbf{301}  &$\mathbf{88.1_{\pm0.2} }$&\textbf{137}  & $\mathbf{91.2_{\pm0.1}}$ &265 & $\mathbf{53.8_{\pm0.1} }$&\textbf{622 } &$\mathbf{61.7_{\pm0.4} }$ &\textbf{389} &$\mathbf{67.9_{\pm0.4} }$&\textbf{330 }\\
		\midrule[1.5pt]
	\end{tabular}
	\label{table 2}
\end{table*}

For strongly convex problems in Theorem \ref{T 5.2} shows that $\varepsilon_{\text {opt }}$ of FedMoSWA  is $\mathcal{O}\left(\frac{\sigma^2}{\mu T K s}\left(1+\frac{s}{\alpha^2}\right)+\frac{m \mu}{s} D^2 \exp \left(-\left\{\frac{s}{m}+ \frac{\mu}{\beta}\right\}T\right)\right)$. 
For $\alpha$, it can be observed that with the increase of $\alpha$, $\varepsilon_{\text {opt }}$ becomes smaller.
The   $\varepsilon_{\text {opt }}$ of \texttt{FedMoSWA} for nonconvex problems is $\mathcal{O}\left(\frac{\sigma \sqrt{F}}{\sqrt{TK s}}\left(\sqrt{1+\frac{m}{\alpha^2}}\right)+\frac{\beta F}{T}\left(\frac{m}{s}\right)^{\frac{2}{3}}\right)$, which  is not affected by the data heterogeneous parameter $\sigma_g$ and is better than \texttt{FedSWA}, FedSAM and MoFedSAM, as shown in Table \ref{table 1}. We can  prove that the $\varepsilon_{\text {opt }}$ of \texttt{FedMoSWA} is faster than SCAFFOLD's $\mathcal{O}\left(\frac{\sigma \sqrt{F}}{\sqrt{T K s}}\left(\sqrt{1+m}\right)+\frac{\beta F}{T}\left(\frac{m}{s}\right)^{\frac{2}{3}} \right)$ \cite{karimireddy2020scaffold}.

\begin{table*}[h]
	\centering
	\setlength{\tabcolsep}{1pt}
	\caption{Comparison of each algorithm on the CIFAR100 and Tiny ImageNet datasets with different data heterogeneity.}
	\begin{tabular}{p{2cm}ccccccccccccccc}
		\midrule[1.5pt]
		& \multicolumn{6}{c}{CIFAR100 (ResNet-18)} & \multicolumn{6}{c}{Tiny ImageNet (ViT-Base)} \\
		\cmidrule(lr){2-7} \cmidrule(lr){8-13} 
		Method & \multicolumn{2}{c}{Dirichlet-0.1} & \multicolumn{2}{c}{Dirichlet-0.3} & \multicolumn{2}{c}{Dirichlet-0.6}& \multicolumn{2}{c}{Dirichlet-0.1} & \multicolumn{2}{c}{Dirichlet-0.3} & \multicolumn{2}{c}{Dirichlet-0.6} \\
		\cmidrule(lr){2-3} \cmidrule(lr){4-5} \cmidrule(lr){6-7}	\cmidrule(lr){8-9} \cmidrule(lr){10-11} \cmidrule(lr){12-13}
		& Acc.(\%) & Rounds & Acc.(\%)& Rounds& Acc.(\%) & Rounds  & Acc.(\%)  & Rounds& Acc.(\%)  & Rounds& Acc.(\%)  & Rounds \\
		& 1000R & 55\%  &1000R & 55\% &1000R & 55\%& 400R & 70\%  &400R & 70\% &400R & 70\% \\
		
		\midrule
		FedAvg   & $45.8_{\pm0.3}$&1000+  &$52.5_{\pm0.3}  $ &1000+  &$54.2_{\pm0.2}$&1000+ & $70.9_{\pm0.1} $&258 &$71.8_{\pm0.1}$&223 &$72.8_{\pm0.1}$&208\\

		FedDyn  &  $45.8_{\pm0.2} $&1000+  & $ 45.9_{\pm0.3} $&1000+ & $46.5_{\pm0.2}$&1000+ & $67.5_{\pm0.3} $&400+   &$ 68.2_{\pm0.3}$&400+   &$ 69.3_{\pm0.3}$&400+ \\
		
		SCAFFOLD  &  $44.3_{\pm0.3}$&1000+  & $50.3_{\pm0.3}  $&1000+  &$52.3_{\pm0.2}$&1000+  & $71.6_{\pm0.1} $&202  &$72.5_{\pm0.1} $&192  &$73.1_{\pm0.2}$& 169\\
		
		FedSAM  & $40.1_{\pm0.4} $&1000+  & $49.0_{\pm0.3}  $&1000+  & $ 51.9_{\pm0.5}$&1000+  &$ 71.4_{\pm0.2} $& 212 & $ 72.2_{\pm0.2} $&194  & $72.9_{\pm0.4}$& 180\\
		
		MoFedSAM  & $51.5_{\pm0.2} $&1000+  & $57.5_{\pm0.2} $&770  & $60.1_{\pm0.1}$& 603 & $71.6_{\pm0.4}$&229  &$72.4_{\pm0.3} $&214  & $ 72.5_{\pm0.4}$&209  \\
		
		FedLESAM  &  $ 48.7_{\pm0.2}  $&1000+ & $53.3_{\pm0.4}   $&1000+ & $52.1_{\pm0.1} $ & 1000+   & $71.9_{\pm0.3} $&210  &$72.1_{\pm0.2}$&188  &$72.5_{\pm0.3}$& 182\\
		
		FedASAM  &  $  47.7_{\pm0.3}  $&1000+& $46.6_{\pm0.2}   $&1000+ &$49.8_{\pm0.1}$  & 1000+  & $69.2_{\pm0.3} $&400+   &$71.3_{\pm0.2}$&234 &$72.1_{\pm0.3}$& 196\\
		
		FedACG& $ 52.2_{\pm0.4}$&1000+  &$57.7_{\pm0.2} $&717  & $61.7_{\pm0.4}$&518 & $66.2_{\pm0.2}$&400+  & $68.5_{\pm0.1}$&400+  &$70.2_{\pm0.3}$&386 \\
		\rowcolor{LightBlue}
		\texttt{FedSWA}& $50.3_{\pm0.3} $ &1000+  &$55.5_{\pm0.4} $&889 &$59.8_{\pm0.3}$ &574 &$71.9_{\pm0.3} $ &199  &$72.6_{\pm0.2} $&179 &$73.2_{\pm0.2}$&168\\
		\rowcolor{LightRed}
		\texttt{FedMoSWA}& $\mathbf{61.9_{\pm0.5}} $ &$\mathbf{577}$ &$\mathbf{66.2_{\pm0.4}} $&$\mathbf{468}$  &$\mathbf{67.9 _{\pm0.4}}$&$\mathbf{330}$ & $\mathbf{73.8_{\pm0.3}} $ &$\mathbf{161}$ &$\mathbf{74.4_{\pm0.3}} $&$\mathbf{152}$ &$\mathbf{74.7_{\pm0.1}}$&$\mathbf{144}$\\
		\midrule[1.5pt]
	\end{tabular}
	
	\label{table 3}
\end{table*}

\section{Experiments}
\subsection{Experimental Settings}
\begin{figure}[bt]
	\centering
	\begin{minipage}[b]{0.235\textwidth}
		\centering
		\subcaptionbox{Dirichlet-0.1}{\includegraphics[width=\textwidth]{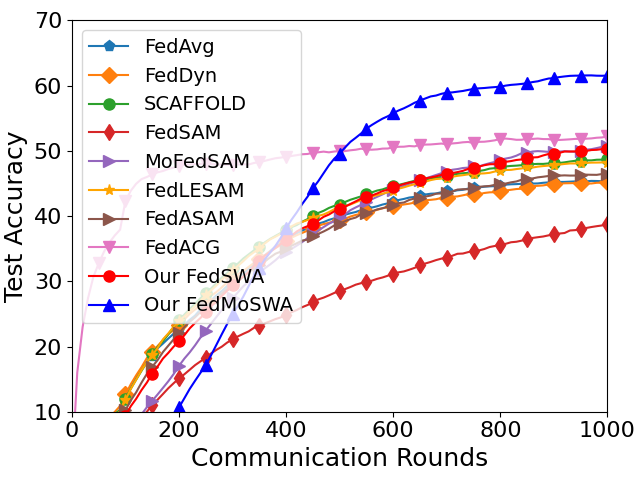}}
	\end{minipage}
	\hfill
	\begin{minipage}[b]{0.235\textwidth}
		\centering
		\subcaptionbox{Dirichlet-0.6}{\includegraphics[width=\textwidth]{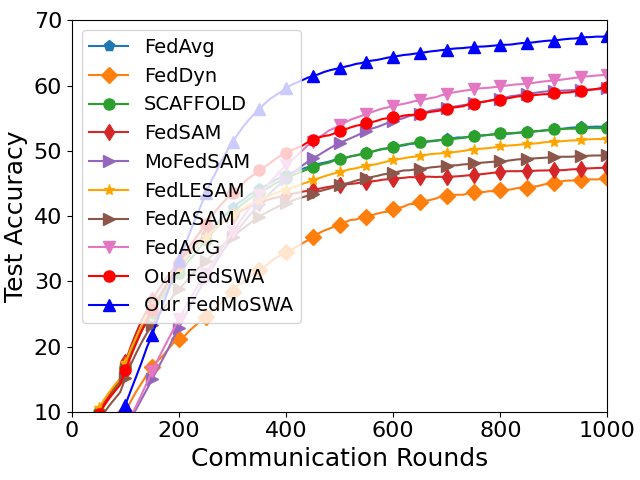}}
	\end{minipage}
	\caption{Convergence plots on CIFAR100 with ResNet-18.}
	\label{figure 3}
\end{figure}
\textbf{{Datasets:}} We evaluate our algorithms on  CIFAR10, CIFAR100 \cite{krizhevsky2009learning}, Tiny ImageNet \cite{le2015tiny}. For Non-IID data setup, we simulate the data heterogeneity by sampling the label ratios from Dirichlet distribution \cite{hsu2019measuring}. Dirichlet-0.1 means very high data heterogeneity and Dirichlet-0.6 implies average data heterogeneity.\\
\textbf{{Models:}} To test the robustness of our algorithms, we use various  classifiers including LeNet-5 \cite{lecun2015lenet}, VGG-11 \cite{simonyan2015very},  ResNet-18 \cite{he2016deep},  vision Transformer (ViT-Base) \cite{dosovitskiy2020image}.
\\
\textbf{{Methods:}} We compare \texttt{FedSWA} and \texttt{FedMoSWA} with many 
SOTA FL methods, including FedAvg \cite{mcmahan2017communication}, FedDyn \cite{durmus2021federated}, SCAFFOLD \cite{karimireddy2020scaffold}, FedAvgM \cite{hsu2019measuring},  FedSAM \cite{qu2022generalized}, MoFedSAM \cite{qu2022generalized}, FedLESAM \cite{fan2024locally}, FedASAM \cite{caldarola2022improving}, and  FedACG \cite{kim2024communication}.\\
\textbf{Hyper-parameter Settings:}
The number of clients is 100, batch size $B\!=\!50$, local epoch $E \!=\! 5$, $K \!=\! 50$  and client selection rate $C \!=\! 0.1$. We set the grid search range of client learning rate by $\eta_l \!\in\!\{10^{-3}, 3\times 10^{-3},...,10^{-1}, 3\times 10^{-1}\}$. The learning rate decay per round of communication is $0.998$, and the total is $T=1,000$. Specifically, we set $\rho\!=\!0.1$,  $\alpha=1.5$, $\gamma=0.2$ for \texttt{FedMoSWA} and \texttt{FedSWA}.

\subsection{Results on Nonconvex Problems}
\begin{figure*}[tb]
	\centering
	
	\begin{minipage}[b]{0.16\textwidth}
		\centering
		\subcaptionbox{SCAFFOLD}{\includegraphics[width=\textwidth]{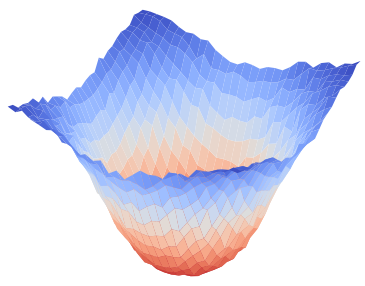}}
	\end{minipage}
	\hfill
	\begin{minipage}[b]{0.16\textwidth}
		\centering
		\subcaptionbox{FedLESAM}{\includegraphics[width=\textwidth]{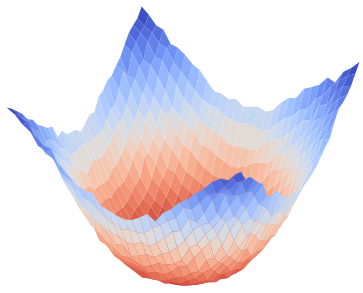}}
	\end{minipage}
	\hfill
	\begin{minipage}[b]{0.16\textwidth}
		\centering
		\subcaptionbox{FedSAM}{\includegraphics[width=\textwidth]{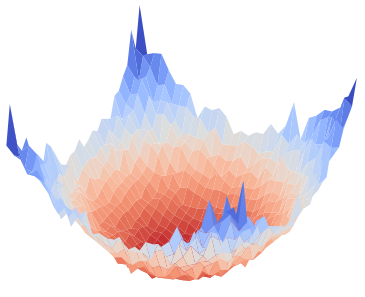}}
	\end{minipage}
	\hfill
	\begin{minipage}[b]{0.16\textwidth}
		\centering
		\subcaptionbox{MoFedSAM}{\includegraphics[width=\textwidth]{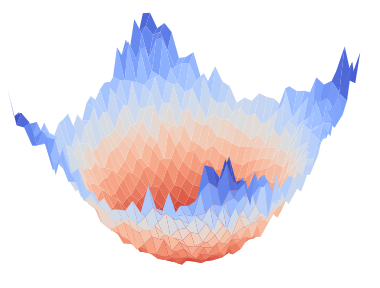}}
	\end{minipage}
	\hfill
	\begin{minipage}[b]{0.16\textwidth}
		\centering
		\subcaptionbox{\texttt{FedSWA}}{\includegraphics[width=\textwidth]{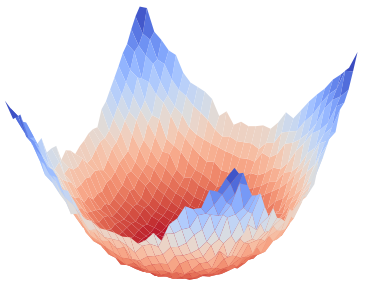}}
	\end{minipage}
	\hfill
	\begin{minipage}[b]{0.16\textwidth}
		\centering
		\subcaptionbox{\texttt{FedMoSWA}}{\includegraphics[width=\textwidth]{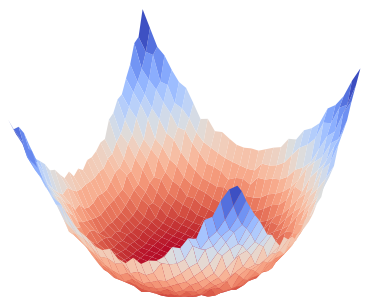}}
	\end{minipage}
	\caption{Global training loss surfaces for \texttt{FedSWA}, \texttt{FedMoSWA} and other baselines on Dirichlet-0.1 of CIFAR100 dataset with ResNet-18.  \texttt{FedMoSWA} could approach a more general and flat loss landscape which high generalization in FL.}
	\label{figure 4}
\end{figure*}
\textbf{Results on Convolutional Neural Network:}
From Table \ref{table 2}  and Figure \ref{figure 3}, we have the following observations:
(i) \texttt{FedSWA} can improve generalization ability by SWA, and it is better than FedSAM. Our \texttt{FedMoSWA} is better than FedSAM and its variant, MoFedSAM, in terms of generalization. 
(ii) \texttt{FedMoSWA} further improves the generalization ability by momentum-based stochastic control. Compared to the stochastic control algorithm, SCAFFOLD, our \texttt{FedMoSWA} demonstrates better generalization ability than SCAFFOLD. Compared to the momentum acceleration algorithm  FedACG, \texttt{FedMoSWA} shows significantly better generalization  combining momentum and variance reduction techniques. Compared to  FedSAM, MoFedSAM, SCAFFOLD and FedACG can not achieve desirable flatness. \texttt{FedMoSWA} further minimizes the global sharpness and
achieves much flatter loss landscape in Figure \ref{figure 4}.

\textbf{Impact of heterogeneity:} From Table \ref{table 3},  
we have the following observations: For different data heterogeneities, our \texttt{FedMoSWA} generalizes well. Especially when the degree of data heterogeneity is high (Dirichle-0.1), the final test results of our FedMoSWA are much better than all other algorithms.  
For example, as Non-IID levels increasing, \texttt{FedMoSWA} achieves  higher test accuracy
10.4\%, 8.5\% and 6.2\% and saving communication round
than MoFedSAM on the CIFAR-100 dataset.

\textbf{Impact of $\alpha$, $\gamma$:}  Figure \ref{figure 5} (a) compares the effects of $\alpha$ on FL aggregation.  As the $\alpha$ increases, the acceleration effect of our \texttt{FedMoSWA}  becomes more obvious. However, when $\alpha>1.5$, the performance of the algorithm decreases with increasing $\alpha$. We can find that the final performance of the algorithm is best when $\alpha=1.5$. Figure \ref{figure 5} (b) compares the effects of $\gamma$ in the same settings. When $\gamma=0.05$, \texttt{FedMoSWA} converges very slowly, and the final algorithm has low accuracy. As the $\gamma$ increases, the acceleration effect of our \texttt{FedMoSWA}  becomes more apparent. The final performance of the algorithm is best when $\gamma=0.2$.

\textbf{Impact of $\rho$: } Table \ref{tab:4}  and Figure \ref{figure 3} compare the effects of $\rho$ in the local training of \texttt{FedSWA} and \texttt{MoFedSWA}.  As $\rho$ decreases, the performance of the algorithms increases. We can find that the final performance of the algorithms is best when $\rho=0.1$. When $\rho=1$ and $\alpha=1$, our  \texttt{FedSWA} is equal to FedAvg, which means  cyclical learning rate  works in FL. When $\rho=1$ and $\alpha=1$, our  \texttt{FedMoSWA} is equal to FedAvg with momentum variance reduction without SWA (65.9\%) which means  momentum variance reduction  works in FL and it is better than SCAFFOLD (52.3\%).

\begin{figure}[tb]
	\centering
	\begin{minipage}[b]{0.235\textwidth}
		\centering
		\subcaptionbox{ Impact of  $\alpha$}{\includegraphics[width=\textwidth]{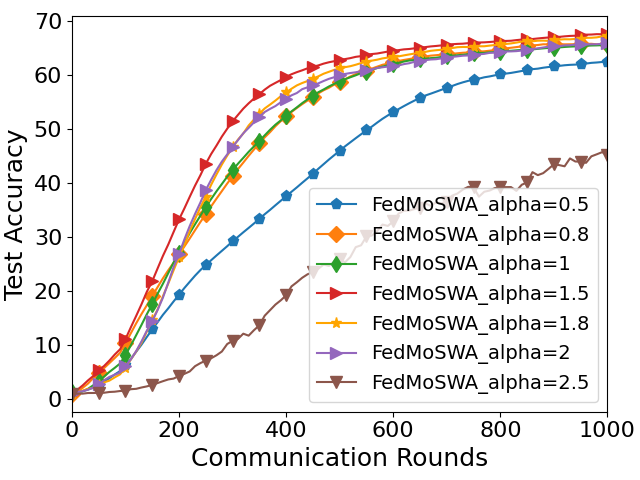}}
	\end{minipage}
	\begin{minipage}[b]{0.235\textwidth}
		\centering
		\subcaptionbox{Impact of $\gamma$}{\includegraphics[width=\textwidth]{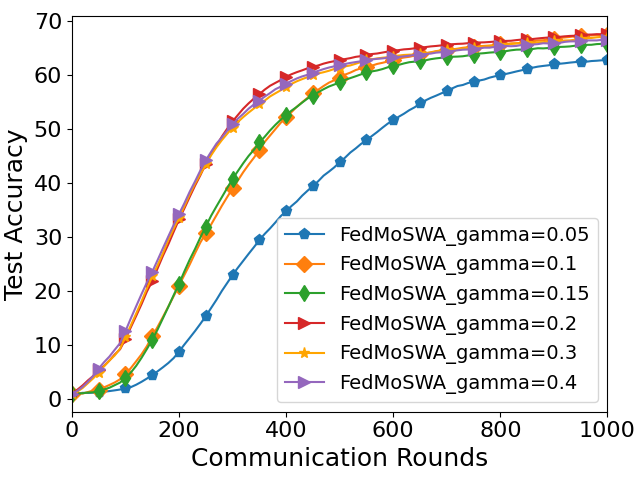}}
	\end{minipage}
	\caption{Convergence plots for \texttt{FedMoSWA} with different $\alpha$ and  $\gamma$ on CIFAR100 dataset using ResNet-18.}
	\label{figure 5}
\end{figure}

\begin{table}[tb]
	\centering
	\caption{Testing Acc (\%)  of \texttt{FedSWA} and \texttt{MoFedSWA} with different $\rho$ on CIFAR100 (ResNet-18, Dirichlet-0.6).}
	\vspace{-2mm}
	\label{tab:4}
	\setlength{\tabcolsep}{3pt}
	\begin{tabular}{lcccccccc}
		\midrule[1.5pt]
		$\rho$ &0.1 & 0.2 & 0.3 & 0.4& 0.7 &0.8& 1.0\\
		\midrule
		\texttt{FedSWA}  &\textbf{59.8}   & 58.5 & 57.9 &  56.6& 55.3 & 54.6 &54.2 \\
		\texttt{FedMoSWA}   & \textbf{67.9} & 67.3 &   67.0 & 66.8& 66.5&66.3&65.9\\
		\midrule[1.5pt]
	\end{tabular}
\end{table} 

\textbf{Results on Vision Transformer:}
In Table \ref{table 3}, we use the  ViT-Base model  on the Tiny ImageNet dataset. We used the pretrained model  from the official website, and conducted many experiments. All the experimental results verify that our \texttt{FedMoSWA}  can achieve excellent generalization on both the vision Transformer model and big datasets.

\subsection{Ablation Study}

\begin{table}[tb]
	\centering
	\caption{Full client participation on CIFAR-100 and ResNet-18 with 10 clients and 300 rounds.}
	\vspace{-2mm}
	\label{tab:5}
	\begin{tabular}{lccc}
		\midrule[1.5pt]
		Algorithm &Accuracy (\%)&	Improvement (\%)\\
		\midrule
        FedAvg  &58.2  & -\\
        SCAFFOLD  &59.9   & 	+1.7\\
        FedSAM  &48.3  & 	-9.9\\
        MoFedSAM  &37.9   & -20.3\\
        FedLESAM &59.2   & +1.0\\
		\texttt{FedSWA} (i)  &	60.2   & +2.0\\
        \texttt{FedMo} (ii)  &62.5  & +4.3 \\
		\texttt{FedMoSWA} (i+ii)  & 63.2 & 	+5.0 \\
		\midrule[1.5pt]
	\end{tabular}
\end{table} 

We conducted full client participation experiments in \ref{tab:5}. Here, FedMo denotes FedMoSWA without SWA (i) but only with Momentum Stochastic Control (ii). FedMo (62.5\%) outperforms SCAFFOLD (59.9\%) by 2.6\%. Even if the client is full participation, our method is still better than SCAFFOLD.

\begin{table}[tb]
	\centering
	\caption{10\% client participation on CIFAR-100 and ResNet-18 with 100 clients and 1000 rounds.}
	\vspace{-2mm}
	\label{tab:6}
	\begin{tabular}{lccc}
		\midrule[1.5pt]
		Algorithm &Accuracy (\%)&	Improvement (\%)\\
		\midrule
        FedAvg  &54.2  & -\\
        SCAFFOLD  &54.1  & 	-0.1\\
        FedSAM  &	47.8  & 	-6.4\\
		\texttt{FedSWA} (i)  &	59.8   & +5.6\\
        \texttt{FedMo} (ii)  &65.9  & +11.7\\
		\texttt{FedMoSWA} (i+ii)  & 67.9 & 	+13.7 \\
		\midrule[1.5pt]
	\end{tabular}
\end{table} 
When Dirichlet-0.6, 10\% participation clients, FedMo using only momentum stochastic control achieves 65.9\%, higher than SCAFFOLD (54.1\%), demonstrating that momentum variance reduction mitigates SCAFFOLD’s variance reduction delay. By combining SWA and momentum variance reduction, FedMoSWA achieves 67.9\% (+13.7\%), showing that both SWA (+5.6\%) and momentum control (+11.7\%) are effective, and momentum control has greater impact. Moreover, FedSWA is a simple algorithm like FedSAM and can be combined with other techniques. 

\section{Conclusion}
This study investigated  generalization concerns in FL and proposed two new algorithms, \texttt{FedSWA} and \texttt{FedMoSWA}, for the setting of highly heterogeneous data. In particular,  \texttt{FedMoSWA} addresses the essential issue of generalization  with highly heterogeneous data. By developing an analytical framework based on algorithmic stability, we theoretically established the superiority of \texttt{FedSWA} and \texttt{FedMoSWA} in terms of  generalization error.  Our \texttt{FedMoSWA} mitigates the effect of data heterogeneity on the generalization error, though it does not completely eliminate it. Our future works may consider  an approach that significantly eliminates the effect of high heterogeneity on the generalization error.

\section*{Acknowledgments}
This work was supported in the National Natural Science Foundation of China (Nos. 62276182), Peng
Cheng Lab Program (No. PCL2023A08), Tianjin Natural Science Foundation (Nos. 24JCYBJC01230, 24JCYBJC01460), and Tianjin Municipal Education Commission
Research Plan (No. 2024ZX008).

\section*{Impact Statement}
The proposed Stochastic Weight Averaging (SWA) and Momentum-Based Stochastic Controlled Weight Averaging (FedMoSWA) algorithms for federated learning aim to address data heterogeneity and improve model generalization. This paper presents one work whose goal is to advance the field of 
Machine Learning. There are many potential societal consequences 
of our work, none which we feel must be specifically highlighted here

\newpage

\bibliography{example_paper}
\bibliographystyle{icml2025}

\newpage
\appendix
\onecolumn


\section{Background}
In this section, we provide a concise overview of Sharpness-Aware Minimization (SAM) and Stochastic Weight Averaging (SWA).

\subsection{SAM: Overview}
$\boldsymbol{S A M}$ seeks to identify a solution $\theta$ that lies within a region characterized by uniformly low training loss $\mathcal{L}_{\mathcal{D}}(\theta)$, specifically within a flat minimum. The sharpness of the training loss function is quantified as:

$$
\max _{\|\epsilon\|_p \leq \rho} \mathcal{L}_{\mathcal{D}}(\theta+\epsilon)-\mathcal{L}_{\mathcal{D}}(\theta)
$$

Here, $\rho$ is a hyper-parameter that determines the size of the neighborhood, and $p \in[1, \infty)$. SAM focuses on reducing the sharpness of the loss by addressing the following minmax objective:

$$
\min _{\theta \in \mathbb{R}^d} \max _{\|\epsilon\|_p \leq \rho} \mathcal{L}_{\mathcal{D}}(\theta+\epsilon)+\lambda\|\theta\|_2^2
$$

where $\lambda$ is a hyper-parameter that balances the regularization term. According to \cite{foret2020sharpness}, the optimal choice for $p$ is typically 2. Therefore, we use the $\ell_2$-norm for the maximization over $\epsilon$ and simplify the expression by omitting the regularization term. To solve the inner maximization problem $\epsilon^* \triangleq \arg \max _{\|\epsilon\|_2 \leq \rho} \mathcal{L}(\theta+\epsilon)$, the authors suggest using a first-order approximation of $\mathcal{L}(\theta+\epsilon)$ around 0:

$$
\epsilon^* \approx \arg \max _{\|\epsilon\|_2 \leq \rho} \mathcal{L}_{\mathcal{D}}(\theta)+\epsilon^T \nabla_\theta \mathcal{L}_{\mathcal{D}}(\theta)=\rho \frac{\nabla_\theta \mathcal{L}_{\mathcal{D}}(\theta)}{\left\|\nabla_\theta \mathcal{L}_{\mathcal{D}}(\theta)\right\|_2}=: \hat{\epsilon}(\theta)
$$

This approximation, which is computationally efficient, results in $\hat{\epsilon}(\theta)$ being essentially a scaled gradient of the current parameters $\theta$. The sharpness-aware gradient is then defined as $\left.\nabla_\theta \mathcal{L}_{\mathcal{D}}(\theta)\right|_{\theta+\hat{\epsilon}(\theta)}$ and is used to update the model as follows:
$$
\theta_{t+1} \leftarrow \theta_t-\left.\gamma \nabla_{\theta_t} \mathcal{L}_{\mathcal{D}}\left(\theta_t\right)\right|_{\theta_t+\hat{\epsilon}_t},
$$

where $\gamma$ is the learning rate and $\hat{\epsilon}_t=\hat{\epsilon}\left(\theta_t\right)$. This two-step process is repeated iteratively to solve the minmax objective. Essentially, SAM first performs a gradient ascent step to estimate the point $\left(\theta_t+\hat{\epsilon}_t\right)$ where the loss is approximately maximized, followed by a gradient descent step at $\theta_t$ using the computed gradient.

\subsection{ Stochastic Weight Averaging: Overview}

SWA (Stochastic Weight Averaging) operates by averaging the weights proposed by Stochastic Gradient Descent (SGD) while employing a learning rate schedule to explore regions of the weight space that correspond to high-performing networks. At each step $i$ within a cycle of length $c$, the learning rate is gradually reduced from an initial value $\gamma_1$ to a final value $\gamma_2$. This learning rate adjustment is defined as:

$$
\gamma(i) = (1 - t(i)) \gamma_1 + t(i) \gamma_2, \quad t(i) = \frac{1}{c} \left(\bmod(i - 1, c) + 1\right)
$$

Here, $t(i)$ represents a linear interpolation factor that depends on the current step $i$ and the cycle length $c$. If $c = 1$, the learning rate remains constant at $\gamma_1$ throughout the training process. However, for $c > 1$, the learning rate follows a cyclical schedule, oscillating between $\gamma_1$ and $\gamma_2$ over the course of the cycle.

Starting from a pre-trained model $f_{\hat{\theta}}$, SWA captures the model weights $\theta$ at the end of each cycle and computes their running average. This averaging process is expressed as:

$$
\theta_{\mathrm{SWA}} \leftarrow \frac{\theta_{\mathrm{SWA}} \cdot n_{\mathrm{models}} + \theta}{n_{\mathrm{models}} + 1}
$$

In this equation, $n_{\mathrm{models}}$ tracks the number of completed cycles, and $\theta_{\mathrm{SWA}}$ represents the averaged weights. The final model, denoted as $f_{\theta_{\mathrm{SWA}}}$, is obtained by aggregating the weights from multiple cycles, effectively smoothing the optimization trajectory and enhancing generalization performance.

By iteratively averaging the weights, SWA encourages the model to converge to a flatter region of the loss landscape, which is often associated with better generalization. This approach leverages the exploration of diverse weight configurations during training, ultimately leading to a more robust and stable model.

\section{Implementation of the Experiments}
\subsection{Datasets}

\begin{table}[H]
	\centering
	\caption{A summary of CIFAR10/100, Tiny ImageNet datasets, including number of total images, number of classes, and the size of the images in the datasets.}
	\vspace{-2mm}
	\label{tab:4}
	\setlength{\tabcolsep}{3pt}
	\begin{tabular}{lcccccccc}
		\midrule[1.5pt]
		Datasets &Total Images &  Class & Image Size \\
		\midrule
		CIFAR10  &60,000   & 10 &3 × 32 × 32 \\
		CIFAR100   & 60,000  &100 & 3 × 32 × 32\\
		Tiny ImageNet   & 100,000 &200&3 × 224 × 224\\
		\midrule[1.5pt]
	\end{tabular}
\end{table}
CIFAR-10/100, Tiny ImageNet are 
popular benchmark datasets in  federated learning. Data samples in CIFAR10 and CIFAR100 are colorful images
of different categories with the resolution of 32 × 32. There are 10 classes and each class has 6,000 images in CIFAR10.
For CIFAR100, there are 100 classes and each class has 600 images. For Tiny ImageNet, there are 200 classes and each class has 500 images. As shown in Table 6, we summarize CIFAR10, CIFAR100, and Tiny ImageNet from the views of number of total images, number of classes, and the size of the images
in the datasets.

\subsection{More Loss Surface Visualization}
Here we show more global loss surface visualizations. As shown in Figure 7-11, we conduct experiments using LeNet-5 and ResNet-18 for CIFAR10 and CIFAR100, and visualize FedAvg, FedSAM, and FedSWA. Among these, our FedSWA achieves a flatter loss landscape.
\begin{figure}[H]
	\begin{minipage}[b]{0.32\textwidth}
		\centering
		\subcaptionbox{ FedAvg }{\includegraphics[width=\textwidth]{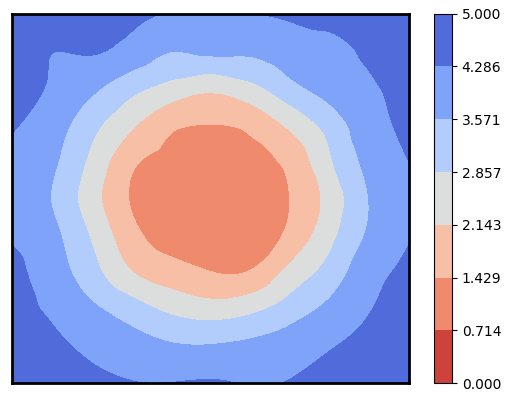}}
	\end{minipage}
	\begin{minipage}[b]{0.32\textwidth}
		\centering
		\subcaptionbox{FedSAM }{\includegraphics[width=\textwidth]{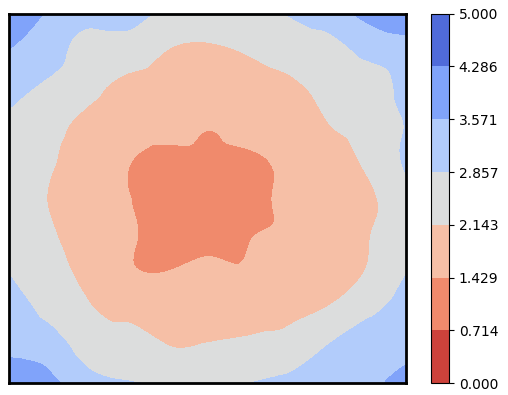}}
	\end{minipage}
	\begin{minipage}[b]{0.32\textwidth}
		\centering
		\subcaptionbox{FedSWA}{\includegraphics[width=\textwidth]{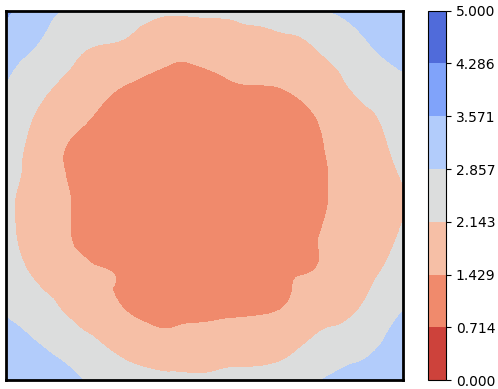}}
	\end{minipage}
	\caption{The training loss surfaces of FedAvg, FedSAM, and our FedSWA on CIFAR-100  with  ResNet-18 (Dirichlet-0.1). }
	\label{figure 7}
\end{figure}

\begin{figure}[H]
	\begin{minipage}[b]{0.32\textwidth}
		\centering
		\subcaptionbox{ FedAvg}{\includegraphics[width=\textwidth]{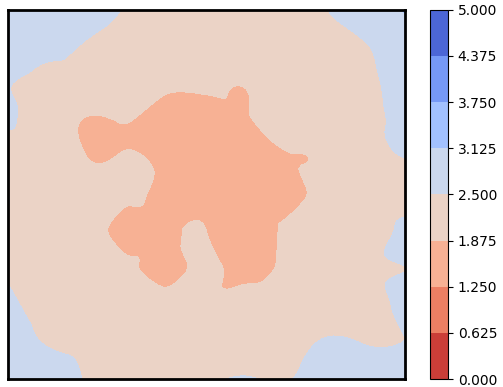}}
	\end{minipage}
	\begin{minipage}[b]{0.32\textwidth}
		\centering
		\subcaptionbox{FedSAM}{\includegraphics[width=\textwidth]{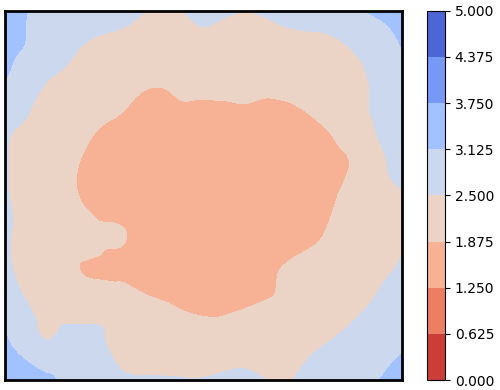}}
	\end{minipage}
	\begin{minipage}[b]{0.32\textwidth}
		\centering
		\subcaptionbox{FedSWA}{\includegraphics[width=\textwidth]{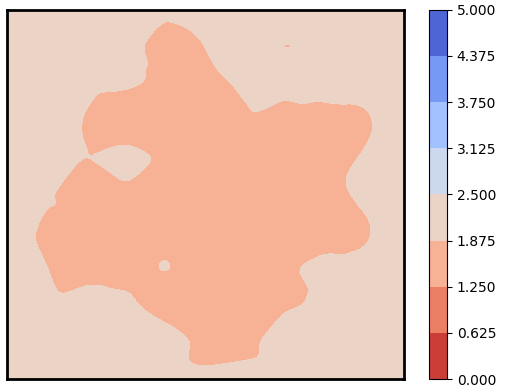}}
	\end{minipage}
	\caption{The training loss surfaces of FedAvg, FedSAM, and our FedSWA on CIFAR-100 with  LeNet-5 (Dirichlet-0.1). }
	\label{figure 8}
\end{figure}

\begin{figure}[H]
	\begin{minipage}[b]{0.32\textwidth}
		\centering
		\subcaptionbox{FedAvg}{\includegraphics[width=\textwidth]{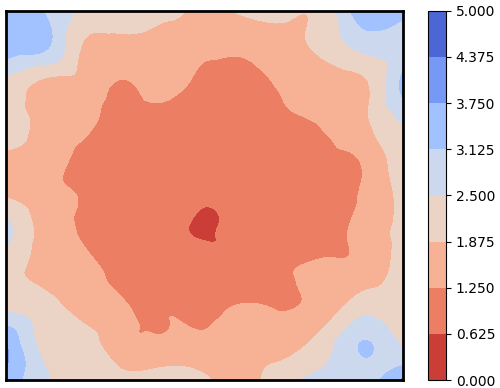}}
	\end{minipage}
	\begin{minipage}[b]{0.32\textwidth}
		\centering
		\subcaptionbox{FedSAM}{\includegraphics[width=\textwidth]{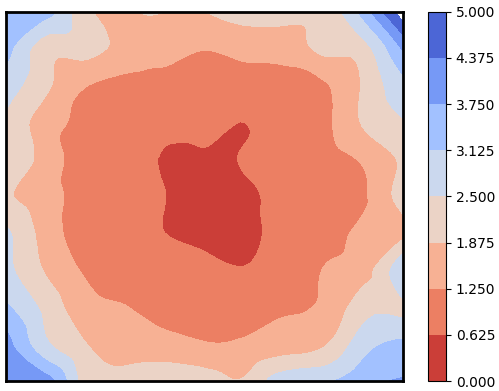}}
	\end{minipage}
	\begin{minipage}[b]{0.32\textwidth}
		\centering
		\subcaptionbox{FedSWA }{\includegraphics[width=\textwidth]{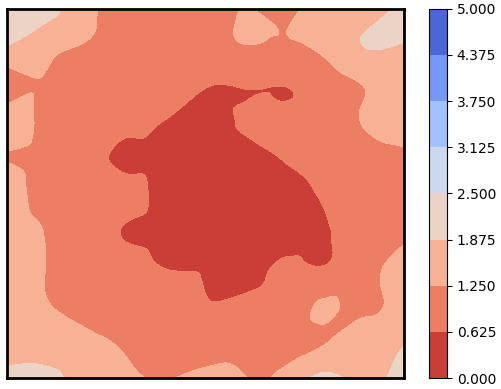}}
	\end{minipage}
	\caption{The training loss surfaces of FedAvg, FedSAM, and our FedSWA on CIFAR-10  with  LeNet-5 (Dirichlet-0.1). }
	\label{figure 9}
\end{figure}

\begin{figure}[H]
	\begin{minipage}[b]{0.32\textwidth}
		\centering
		\subcaptionbox{ FedSAM (0.1) }{\includegraphics[width=\textwidth]{image/FedSAM_0.1_1.png}}
	\end{minipage}
	\begin{minipage}[b]{0.32\textwidth}
		\centering
		\subcaptionbox{FedSAM (0.3) }{\includegraphics[width=\textwidth]{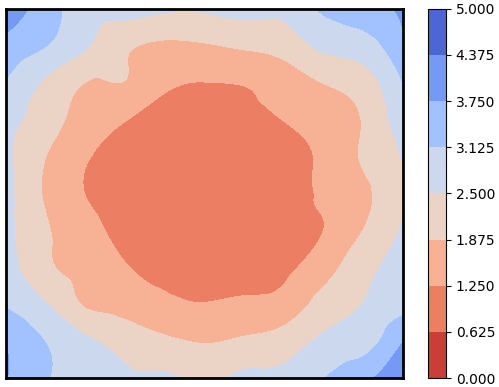}}
	\end{minipage}
	\begin{minipage}[b]{0.32\textwidth}
		\centering
		\subcaptionbox{FedSAM (0.6)}{\includegraphics[width=\textwidth]{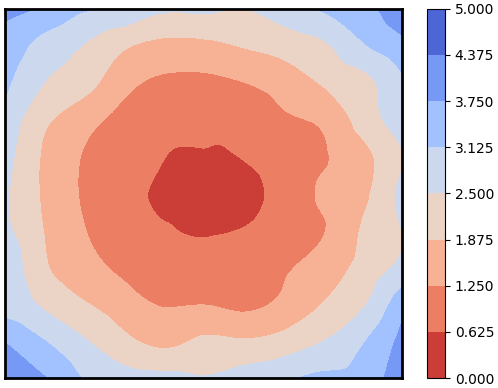}}
	\end{minipage}
	\caption{The training loss surfaces of FedSAM on CIFAR-100  with ResNet-18 (Dirichlet-0.1, Dirichlet-0.3, Dirichlet-0.6).  }
	\label{figure 10}
\end{figure}

\begin{figure}[H]
	\begin{minipage}[b]{0.32\textwidth}
		\centering
		\subcaptionbox{ FedSWA (0.1) }{\includegraphics[width=\textwidth]{image/FedSWA_0.1.png}}
	\end{minipage}
	\begin{minipage}[b]{0.32\textwidth}
		\centering
		\subcaptionbox{FedSWA (0.3) }{\includegraphics[width=\textwidth]{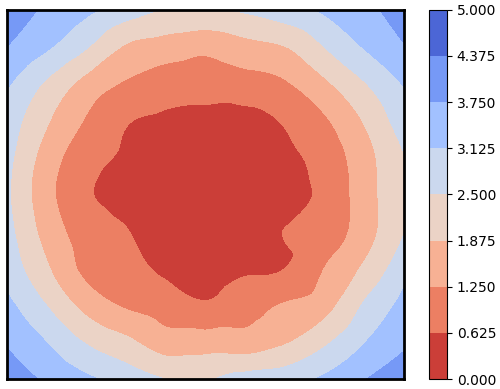}}
	\end{minipage}
	\begin{minipage}[b]{0.32\textwidth}
		\centering
		\subcaptionbox{FedSWA (0.6)}{\includegraphics[width=\textwidth]{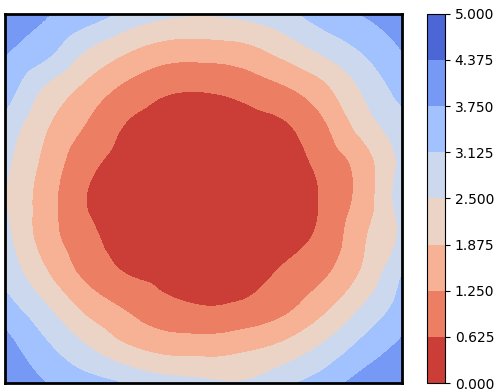}}
	\end{minipage}
	\caption{The training loss surfaces of FedSWA on CIFAR-100  with ResNet-18 (Dirichlet-0.1, Dirichlet-0.3, Dirichlet-0.6). }
	\label{figure 11}
\end{figure}

\subsection{Convergence Behavior Curves}

\begin{figure}[H]
	\centering
	\begin{minipage}[b]{0.32\textwidth}
		\centering
		\subcaptionbox{Lenet-5, CIFAR10}{\includegraphics[width=\textwidth]{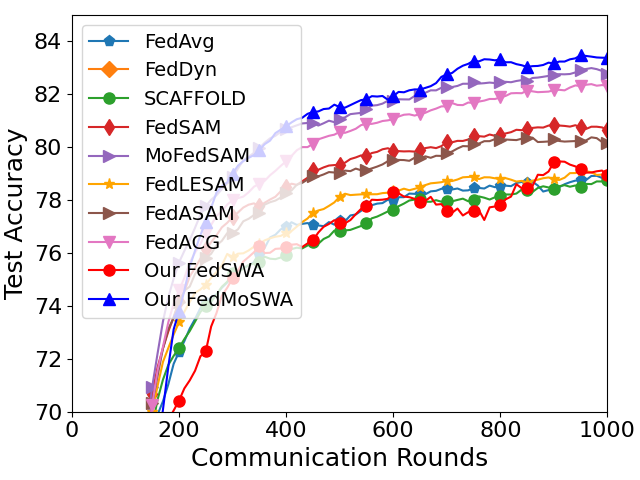}}
	\end{minipage}
	\hfill
	\begin{minipage}[b]{0.32\textwidth}
		\centering
		\subcaptionbox{Lenet-5, CIFAR100}{\includegraphics[width=\textwidth]{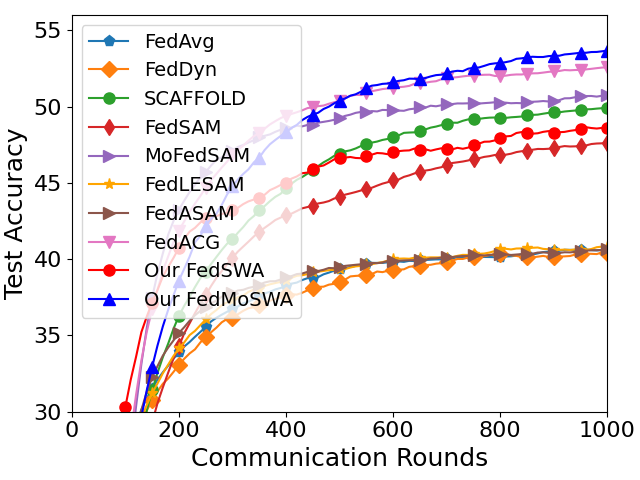}}
	\end{minipage}
	\hfill
	\begin{minipage}[b]{0.32\textwidth}
		\centering
		\subcaptionbox{VGG-11, CIFAR10}{\includegraphics[width=\textwidth]{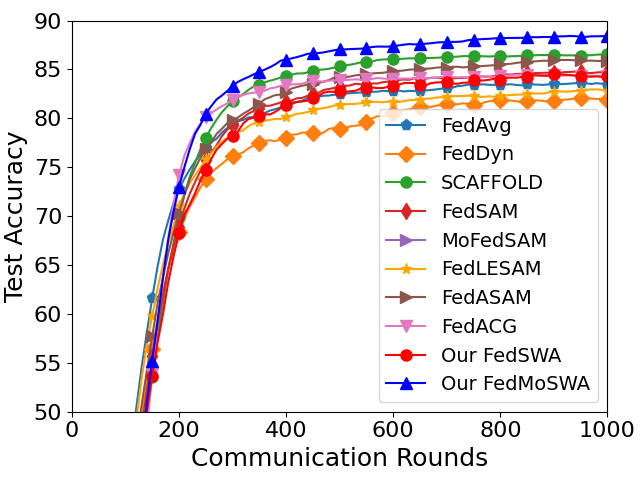}}
	\end{minipage}
	\hfill
	\begin{minipage}[b]{0.32\textwidth}
		\centering
		\subcaptionbox{VGG-11, CIFAR100}{\includegraphics[width=\textwidth]{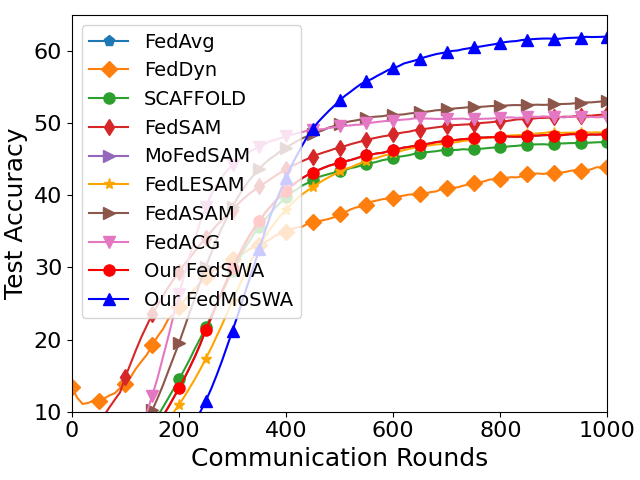}}
	\end{minipage}
	\hfill
	\begin{minipage}[b]{0.32\textwidth}
		\centering
		\subcaptionbox{Resnet-18, CIFAR10}{\includegraphics[width=\textwidth]{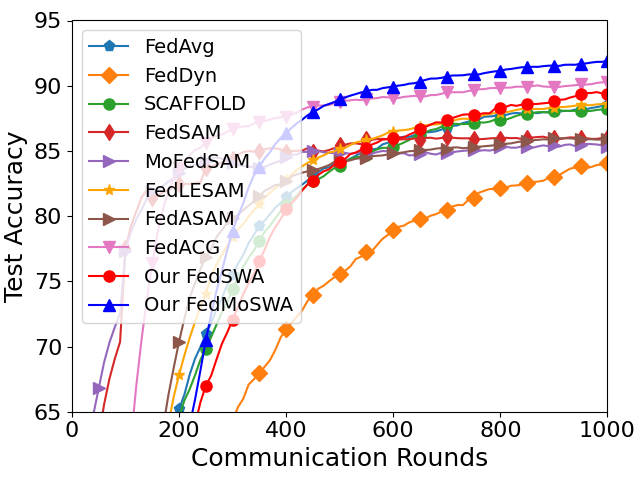}}
	\end{minipage}
	\hfill
	\begin{minipage}[b]{0.32\textwidth}
		\centering
		\subcaptionbox{ResNet-18, CIFAR100}{\includegraphics[width=\textwidth]{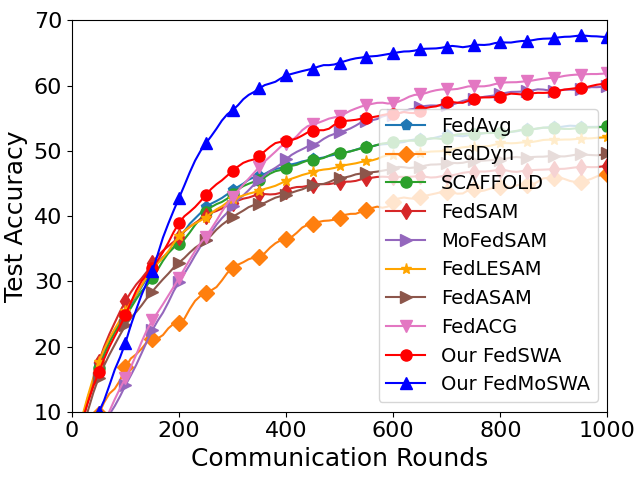}}
	\end{minipage}
	\vspace{-2pt}
	\caption{Convergence plots for FedSWA, FedMoSWA and other baselines in same settings with Dirichlet-0.6 on CIFAR10 and CIFAR100 using different neural network architectures.}
\end{figure}

\begin{figure}[H]
	\centering
	\begin{minipage}[b]{0.32\textwidth}
		\centering
		\subcaptionbox{ResNet-18, Dirichlet-0.6}{\includegraphics[width=\textwidth]{image/resnet18_cifar100_0.6.png}}
	\end{minipage}
	\hfill
	\begin{minipage}[b]{0.32\textwidth}
		\centering
		\subcaptionbox{ResNet-18, Dirichlet-0.3}{\includegraphics[width=\textwidth]{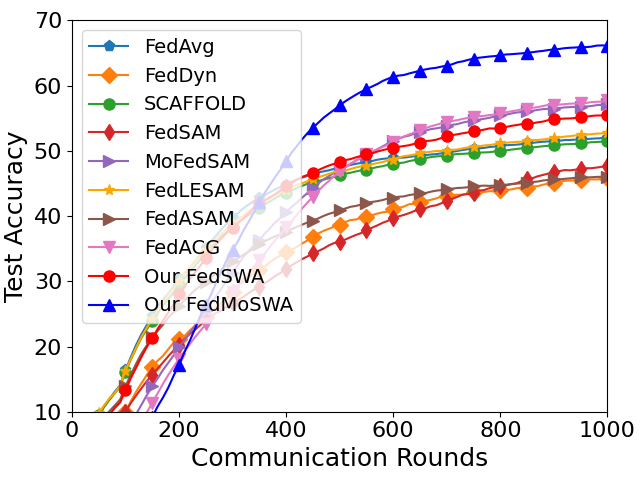}}
	\end{minipage}
	\hfill
	\begin{minipage}[b]{0.32\textwidth}
		\centering
		\subcaptionbox{ResNet-18, Dirichlet-0.1}{\includegraphics[width=\textwidth]{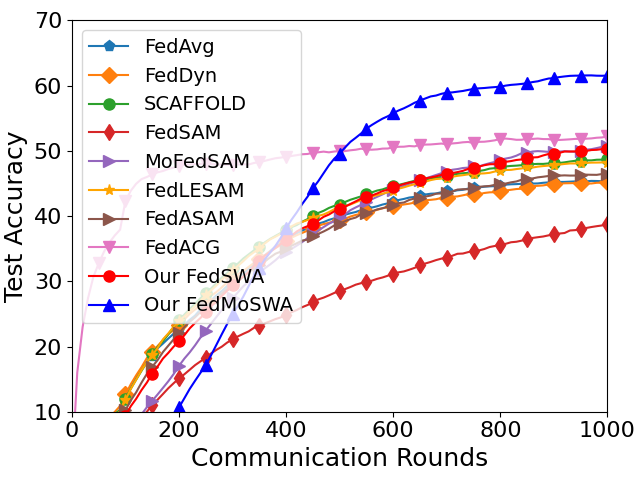}}
	\end{minipage}
	\caption{Convergence plots for FedSWA, FedMoSWA and  baselines with  Dirichlet-0.6, 0.3, 0.1 on CIFAR100  using ResNet-18.}
	\label{figure 13}
\end{figure}

\begin{table}[tb]
	\centering
	\caption{CIFAR-100 with ResNet-18 under different data heterogeneity levels (0.1, 0.3, 0.6).}
	\vspace{-2mm}
	\label{tab:8}
	\begin{tabular}{lcccc}
		\midrule[1.5pt]
		heterogeneity &0.1&	0.3& 0.6\\
		\midrule
        FedSMOO  &46.5	 & 47.8&	49.2\\
		FedGAMMA &48.4	& 51.8	 &52.6\\
		FedSWA (ours) &	50.3	  & 55.5	&59.8\\
		FedMoSWA (ours)	& 61.9	
 & 	66.2	 &67.9\\
		\midrule[1.5pt]
	\end{tabular}
\end{table}

\section{Basic Assumptions }

We analyze generalization based on  following assumptions: \\
\begin{assumption}
     \textit{The loss function $l(\cdot, z)$ is $\mu$-strongly convex for any sample $z$, $l(\boldsymbol{\theta} ; z) \geq l\left(\boldsymbol{\theta}^{\prime} ; z\right)+\left\langle\nabla l\left(\boldsymbol{\theta}^{\prime} ; z\right), \boldsymbol{\theta}-\boldsymbol{\theta}^{\prime}\right\rangle+\frac{\mu}{2}\left\|\boldsymbol{\theta}-\boldsymbol{\theta}^{\prime}\right\|^2$, for any $z, \boldsymbol{\theta}, \boldsymbol{\theta}^{\prime}$.}\label{A1}
\end{assumption}

\begin{assumption}
\textit{The loss function $l(\cdot, z)$ is $L$ Lipschitz continous, $\mid l(\boldsymbol{\theta} ; z)-$ $l\left(\boldsymbol{\theta}^{\prime} ; z\right) \mid \leq L\left\|\boldsymbol{\theta}-\boldsymbol{\theta}^{\prime}\right\|$, for any $z,\boldsymbol{\theta}, \boldsymbol{\theta}^{\prime}$. We note that under Assumption 2,}
$
\mathbb{E}_{\mathcal{A}, \mathcal{S}, z}\left|l\left(\theta_T ;z\right)-l\left(\theta_T^{\prime} ; z\right)\right| \leq L \mathbb{E}\left\|\theta_T-\theta_T^{\prime}\right\|
,\mathbb{E}\left\|\nabla F\left(\theta_t\right)\right\| \leq L$.\label{A2}
\end{assumption}

\begin{assumption}
\textit{There exists a constant $\sigma>0$ such that for any $\boldsymbol{\theta}, i \in[m]$, and $z_i \sim D_i$, $\mathbb{E}\left[\left\|\nabla l\left(\boldsymbol{\theta} ; z_i\right)-\nabla F_i(\boldsymbol{\theta})\right\|^2\right] \leq \sigma^2$.}\label{A3}
\end{assumption}

\begin{assumption}
\textit{The loss function $l(\cdot, z)$ is $\beta$-smooth, $\left\|\nabla l(\boldsymbol{\theta} ; z)-\nabla l\left(\boldsymbol{\theta}^{\prime} ; z\right)\right\| \leq \beta\left\|\boldsymbol{\theta}-\boldsymbol{\theta}^{\prime}\right\|$, for any $z, \boldsymbol{\theta}, \boldsymbol{\theta}^{\prime}$.}\label{A4}
\end{assumption}

\begin{assumption}
 (data heterogeneity).
\textit{Given $i \in[m]$, there exists a constant $\sigma_{g}>0$, for any $\boldsymbol{\theta}$ we have
}
$\left\|\nabla F_i(\theta)-\nabla F(\theta)\right\| \leq \sigma_{g, i},
\frac{1}{m} \sum_{i=1}^m\left\|\nabla F_i(\theta)-\nabla F(\theta)\right\| \leq \frac{1}{m} \sum_{i=1}^m \sigma_{g, i}=\sigma_g$\label{A5}
\end{assumption}

\section{Main Lemmas}
\begin{lemma} 
	Assume that $f$ is $\beta$-smooth. Then, the following properties hold.\\
	1. $G_{f, \eta}$ is $(1+\eta \beta)$-expansive, $G(v)=G_{f, \eta}(v)=v-\eta\nabla f(v)$\\
	2. Assume in addition that $f$ is $\mu$-strongly convex. Then, for $\eta \leq \frac{2}{\beta+\mu}, G_{f, \eta}$ is $\left(1-\frac{\eta \beta \mu}{\beta+\mu}\right)$ expansive.\\
	Henceforth we will no longer mention which random selection rule we use as the proofs are almost identical for both rules.\label{L1}
\end{lemma}
\begin{proof}
1. Let $G=G_{f, \eta}$. By triangle inequality and our smoothness assumption,
\begin{align}
	\begin{split}
		\|G(v)-G(w)\| & \leq\|v-w\|+\eta|\nabla f(w)-\nabla f(v)\| \\
		& \leq\|v-w\|+\eta \beta\|w-v\| \\
		& =(1+\eta\beta)\|v-w\| .
	\end{split}
\end{align}

2. First, note that if $f$ is $\mu$ strongly convex, then $\varphi(w)=f(w)-\frac{\mu}{2}\|w\|^2$ is convex with $(\beta-\mu)$-smooth. Hence, 

\begin{align}
	\begin{split}
		\langle\nabla f(v)-\nabla f(w), v-w\rangle \geq \frac{\beta \mu}{\beta+\mu}\|v-w\|^2+\frac{1}{\beta+\mu}\|\nabla f(v)-\nabla f(w)\|^2
	\end{split}
\end{align}
Using this inequality gives
\begin{align}
	\begin{split}
		\left\|G_{f,\eta}(v)-G_{f, \eta}(w)\right\|^2 & =\|v-w\|^2-2 \eta\langle\nabla f(v)-\nabla f(w), v-w\rangle+\eta^2\|\nabla f(v)-\nabla f(w)\|^2 \\
		& \leq\left(1-2 \frac{\eta \beta \mu}{\beta+\mu}\right)\|v-w\|^2-\eta\left(\frac{2}{\beta+\mu}-\eta\right)\|\nabla f(v)-\nabla f(w)\|^2 .
	\end{split}
\end{align}
With our assumption that $\eta\leq \frac{2}{\beta+\mu}$, this implies
\begin{align}
	\begin{split}
		\left\|G_{f, \eta}(v)-G_{f, \eta}(w)\right\| \leq\left(1-2 \frac{\eta \beta \mu}{\beta+\mu}\right)^{1 / 2}\|v-w\|\leq\left(1-\frac{\eta \beta \mu}{\beta+\mu}\right)\|v-w\| .
	\end{split}
\end{align}
The lemma follows by applying the inequality $\sqrt{1-x} \leq 1-x / 2$ which holds for $x \in[0,1]$.
\end{proof}

\section{ Theoretical Results of FedSWA}
\subsection{ Generalization Analysis for FedSWA under strongly convex setting }	
\begin{lemma}Suppose Assumptions \ref{A1}-\ref{A5} hold. Then for FedSWA with $\eta_k^{t} \leq 1 / \beta KT$, $\alpha = 1$,
	$$
	\mathbb{E}\left\|\theta_{i, k}^t-\theta_t\right\| \leq (1-b)^{K-1}\tilde{\eta}_{ t}\left(\mathbb{E}\left\|\nabla F\left(\theta_t\right)\right\|+ \sigma_{g, i}+\sigma\right), \forall k=1, \ldots, K,
	$$
	where $\tilde{\eta}_{ t}=\sum_{k=0}^{K-1} \eta_k^{t}$.\\
\end{lemma}\label{E1}
\begin{proof}
Considering local update  of FedSWA

\begin{align}
	\begin{split}
		\mathbb{E}\left\|\theta_{i, k+1}-\theta_t\right\| & =\mathbb{E}\left\|\theta_{i, k}^t-\eta_k^{t} g_i\left(\theta_{i, k}^t\right)-\theta_t\right\| \\
		& \leq \mathbb{E}\left\|\theta_{i, k}^t-\theta_t-\eta_k^{t}\left(g_i\left(\theta_{i, k}^t\right)-g_i\left(\theta_t\right)\right)\right\|+\eta_k^{t} \mathbb{E}\left\|g_i\left(\theta_t\right)\right\| \\
		& \stackrel{(a)}{\leq}\left(1-\frac{\eta_k^{t} \beta \mu}{\beta+\mu}\right) \mathbb{E}\left\|\theta_{i, k}^t-\theta_t\right\|+\eta_k^{t} \mathbb{E}\left\|g_i\left(\theta_t\right)\right\| \\
		& \leq\left(1-\frac{\eta_k^{t} \beta \mu}{\beta+\mu}\right) \mathbb{E}\left\|\theta_{i, k}^t-\theta_t\right\|+\eta_k^{t}\left(\mathbb{E}\left\|g_i\left(\theta_t\right)-\nabla F_i\left(\theta_t\right)\right\|+\mathbb{E}\left\|\nabla F_i\left(\theta_t\right)\right\|\right) \\
		& \stackrel{(b)}{\leq} \left(1-\frac{\eta_k^{t} \beta \mu}{\beta+\mu}\right)\mathbb{E}\left\|\theta_{i, k}^t-\theta_t\right\|+\eta_k^{t}\left(\mathbb{E}\left\|\nabla F_i\left(\theta_t\right)\right\|+\sigma\right),
	\end{split}
\end{align}
where (a) follows Lemma \ref{L1}; (b) follows Assumption \ref{A3}. Unrolling the above and noting $\theta_{i, 0}=\theta_t$, with $b=\left(1-\frac{\eta_K^{t} \beta \mu}{\beta+\mu}\right)$ for some $b>0$ yields
\begin{align}
	\begin{split}
		\mathbb{E}\left\|\theta_{i, k}^t-\theta_t\right\| & \leq \sum_{l=0}^{k-1} \eta_l^{t}\left(\mathbb{E}\left\|\nabla F_i\left(\theta_t\right)\right\|+\sigma\right)(1-b)^{k-1-l} \\
		& \leq \sum_{l=0}^{K-1}\eta_l^{t}\left(\mathbb{E}\left\|\nabla F_i\left(\theta_t\right)\right\|+\sigma\right)(1-b)^{K-1} \\
		& \leq(1-b)^{K-1} \tilde{\eta}_{t}\left(\mathbb{E}\left\|\nabla F\left(\theta_t\right)\right\|+ \sigma_{g, i}+\sigma\right),
	\end{split}
\end{align}
where the last inequality follows Assumption  \ref{A5}, $\tilde{\eta}_{ t}=\sum_{k=0}^{K-1} \eta_k^{t}$.
\end{proof}

\begin{lemma}\label{E2}
	Given Assumptions \ref{A1}-\ref{A4} and considering of FedSWA, for $\eta_{k}^t \leq 1 / \beta KT$ we have
	$$
	\mathbb{E}\left\|g_i\left(\theta_{i, k}^t\right)\right\| \leq\left(1+(1-b)^{K-1} \beta \tilde{\eta}_t\right)\left(\mathbb{E}\left\|\nabla F\left(\theta_t\right)\right\|+ \sigma_{g, i}+\sigma\right),
	$$
	where $g_i(\cdot)$ is the sampled gradient of client $i, \tilde{\eta}_{t}=\sum_{k=0}^{K-1} \eta_{ k}^t$.\\

\end{lemma}
\begin{proof}
    Using Assumptions  \ref{A3}, \ref{A4}, \ref{A5}, we obtain
	\begin{align}
		\begin{split}
			\mathbb{E}\left\|g_i\left(\theta_{i, k}^t\right)\right\| & \leq \mathbb{E}\left\|g_i\left(\theta_{i, k}^t\right)-\nabla F_i\left(\theta_{i, k}^t\right)\right\|+\mathbb{E}\left\|\nabla F_i\left(\theta_{i, k}^t\right)\right\| \\
			& \leq^a \mathbb{E}\left\|\nabla F_i\left(\theta_{i, k}^t\right)\right\|+\sigma \\
			& \leq \mathbb{E}\left\|\nabla F_i\left(\theta_t\right)\right\|+\mathbb{E}\left\|\nabla F_i\left(\theta_{i, k}^t\right)-\nabla F_i\left(\theta_t\right)\right\|+\sigma \\
			& \leq^b \mathbb{E}\left\|\nabla F\left(\theta_t\right)\right\|+\mathbb{E}\left\|\nabla F_i\left(\theta_t\right)-\nabla F\left(\theta_t\right)\right\|+\beta \mathbb{E}\left\|\theta_{i, k}^t-\theta_t\right\|+\sigma \\
			& \leq^c\left(1+(1-b)^{K-1} \beta \tilde{\eta}_t\right)\left(\mathbb{E}\left\|\nabla F\left(\theta_t\right)\right\|+ \sigma_{g, i}+\sigma\right),
		\end{split}
	\end{align}
$a$ is following Assumption \ref{A3}, $b$ is is following Assumption \ref{A4},
$c$ is is following Assumption \ref{A5} and \ref{L1}.
\end{proof}

\begin{theorem}
	(FedSWA). Suppose Assumptions 1-5 hold and consider FedSWA. Let $\left\{\theta_t\right\}_{t=0}^T$ and $\left\{\theta_t^{\prime}\right\}_{t=0}^T$ be two  datasets $\mathcal{S}$ and $\mathcal{S}^{(i)}$, respectively. Suppose $\theta_0=\theta_0^{\prime}$. Then,
	$$
	\varepsilon_{\text {gen }}\leq \frac{2L \tilde{b}}{m n} \frac{1}{\beta} e^{1-\frac{ \mu}{(\beta+\mu) T}}\left(\mathbb{E}\left\|\nabla F\left(\theta_t\right)\right\|+\sigma_g+\sigma\right),
	$$
	where $\tilde{\eta}_{k}^t=\sum_{k=0}^{K-1} \eta_{k}^t$, $\tilde{b}\leq1+(1-b)^{K-1} \beta \tilde{\eta}_{ t}$\\
\end{theorem}

\begin{proof}
For client $i$, there are two cases to consider. In the first case, SGD selects the index of an sample at local step $k$ on which is identical in $\mathcal{S}$ and $\mathcal{S}^{(i)}$. In this sense, we have
\begin{align}
	\begin{split}
		\left\|\theta_{i, k+1}^t-\theta_{i, k+1}^{t,\prime}\right\| \leq\left(1-\frac{\eta^t_{ k} \beta \mu}{\beta+\mu}\right)\left\|\theta_{i, k}^t-\theta_{i, k}^{t,\prime}\right\|.
	\end{split}
\end{align}
And this case happens with probability $1-1 / n$ (since only one sample is perturbed for client $i$ ).

In the second case, SGD encounters the perturbed sample at local time step $k$, which happens with probability $1 / n$. We denote the gradient of this perturbed sample as $g_i^{\prime}(\cdot)$. Then,
\begin{align}
	\begin{split}
		\left\|\theta_{i, k+1}^t-\theta_{i, k+1}^{t,\prime}\right\| & =\left\|\theta_{i, k}^t-\theta_{i, k}^{t,\prime}-\eta_{k}^t\left(g_i\left(\theta_{i, k}^t\right)-g_i^{\prime}\left(\theta_{i, k}^{t,\prime}\right)\right)\right\| \\
		& \leq\left\|\theta_{i, k}^t-\theta_{i, k}^{t,\prime}-\eta_{k}^t\left(g_i\left(\theta_{i, k}^t\right)-g_i\left(\theta_{i, k}^{t,\prime}\right)\right)\right\|+\eta_{k}^t\left\|g_i\left(\theta_{i, k}^{t}\right)-g_i^{\prime}\left(\theta_{i, k}^{t,\prime}\right)\right\| \\
		& \leq\left(1-\frac{\eta^t_{ k} \beta \mu}{\beta+\mu}\right)\left\|\theta_{i, k}^t-\theta_{i, k}^{t,\prime}\right\|+\eta_{k}^t\left\|g_i\left(\theta_{i, k}^{t}\right)-g_i^{\prime}\left(\theta_{i, k}^{t,\prime}\right)\right\| .
	\end{split}
\end{align}
Combining these two cases we have for client $i$,
\begin{align}
	\begin{split}
		\mathbb{E}\left\|\theta_{i, k+1}^t-\theta_{i, k+1}^{t,\prime}\right\| & \leq \left(1-\frac{\eta^t_{ k} \beta \mu}{\beta+\mu}\right)\mathbb{E}\left\|\theta_{i, k}^t-\theta_{i, k}^{t,\prime}\right\|+\frac{\eta_{k}^t}{n} \mathbb{E}\left\|g_i\left(\theta_{i, k}^{t}\right)-g_i^{\prime}\left(\theta_{i, k}^{t,\prime}\right)\right\| \\
		& \leq \left(1-\frac{\eta^t_{ k} \beta \mu}{\beta+\mu}\right)\mathbb{E}\left\|\theta_{i, k}^t-\theta_{i, k}^{t,\prime}\right\|+\frac{2 \eta_{k}^t}{n} \mathbb{E}\left\|g_i\left(\theta_{i, k}^t\right)\right\|,
	\end{split} \label{eq 19}
\end{align}
where the last inequality follows that $g_i(\cdot)$ and $g_i^{\prime}(\cdot)$ are sampled from the same distribution.

Iterating the above over $t$ and noting $\theta_0=\theta_0^{\prime}$, we conclude the proof.
We let $\tilde{b}$ be an upper bound of $1+(1-b)^{K-1} \beta \tilde{\eta}_{ t}$ since $\tilde{\eta}_{ t}$ is bounded above. Then unrolling \ref{eq 19}, and follow \ref{E2},

\begin{align}
	\begin{split}
		& \mathbb{E}\left\|\theta_{i, K}^t-\theta_{i, K}^{t,\prime}\right\| \\
        &\leq \prod_{k=0}^{K-1}\left(1-\frac{\eta^t_{ k} \beta \mu}{\beta+\mu}\right) \mathbb{E}\left\|\theta_t-\theta_t^{\prime}\right\|+\left(\frac{2}{n} \sum_{k=0}^{K-1} \eta_{k}^t \tilde{b} \prod_{l=k+1}^{K-1}\left(1-\frac{\eta^t_{ k} \beta \mu}{\beta+\mu}\right)\left(\mathbb{E}\left\|\nabla F\left(\theta_t\right)\right\|+ \sigma_{g, i}+\sigma\right)\right) \\
		& \leq^a e^{\frac{-\tilde{\eta}_{ t} \beta \mu}{\beta+\mu} } \mathbb{E}\left\|\theta_t-\theta_t^{\prime}\right\|+\frac{2}{n} \tilde{b} \tilde{\eta}_{ t} e^{\frac{-\tilde{\eta}_{ t} \beta \mu}{\beta+\mu} }\left(\mathbb{E}\left\|\nabla F\left(\theta_t\right)\right\|+ \sigma_{g, i}+\sigma\right) .
	\end{split}
\end{align}
$a$ is following $1-x\leq e^{-x}$.
With $\theta_{t+1}= \frac{1}{m}\sum_{i=1}^m \theta_{i, K}^t$ and $\theta_{t+1}^{\prime}= \frac{1}{m}\sum_{i=1}^m \theta_{i, K}^{t,\prime}$, we have
\begin{align}
	\begin{split}
		& \mathbb{E}\left\|\theta_{t+1}-\theta_{t+1}^{\prime}\right\| \leq \sum_{i=1}^m \frac{1}{m} \mathbb{E}\left\|\theta_{i, K}^t-\theta_{i, K}^{t,\prime}\right\| \\
		& \leq e^{\frac{-\tilde{\eta}_{ t} \beta \mu}{\beta+\mu} } \mathbb{E}\left\|\theta_t-\theta_t^{\prime}\right\|+\frac{2}{mn} \tilde{b} \tilde{\eta}_{ t} e^{\frac{-\tilde{\eta}_{ t} \beta \mu}{\beta+\mu} }\left(\mathbb{E}\left\|\nabla F\left(\theta_t\right)\right\|+ \sigma_{g}+\sigma\right) \\
	\end{split}
\end{align}
Further, unrolling the above over $t$ and noting $\theta_0=\theta_0^{\prime}$, we obtain
\begin{align}
	\begin{split}
		&\mathbb{E}\left\|\theta_T-\theta_T^{\prime}\right\| \leq \frac{2 \tilde{b}}{mn} \sum_{t=0}^{T-1} \exp \left(\beta \sum_{l=t+1}^{T-1} \tilde{\eta}_{ t}\right) \tilde{\eta}_{ t} e^{-\frac{\tilde{\eta}_{ t} \beta \mu}{\beta+\mu} }\left(\mathbb{E}\left\|\nabla F\left(\theta_t\right)\right\|+ \sigma_{g}+\sigma\right) \\
		&\leq \frac{2 \tilde{b}}{m n} \frac{1}{\beta} e^{1-\frac{ \mu}{(\beta+\mu) T}}\left(\mathbb{E}\left\|\nabla F\left(\theta_t\right)\right\|+\sigma_g+\sigma\right),
	\end{split}
\end{align}
With Assumption \ref{A2},
$$
\varepsilon_{\text {gen }}\leq \frac{2L \tilde{b}}{m n} \frac{1}{\beta} e^{1-\frac{ \mu}{(\beta+\mu) T}}\left(\mathbb{E}\left\|\nabla F\left(\theta_t\right)\right\|+\sigma_g+\sigma\right),
$$
$$
\tilde{b}=1+\left(\frac{\mu}{(\beta+\mu) K}\right)^{K-1} \frac{1}{T} \gg 1
$$
When the diminishing stepsizes $\eta_{k}^t \leq 1 / \beta KT$ are chosen in the statement of the theorem, we conclude the proof.
\end{proof}

\subsection{ Generalization Analysis for FedSWA under non-convex setting }	

\begin{lemma}\label{E4}
	Suppose Assumptions \ref{A2}-\ref{A5} hold. Then for FedSWA with $\eta_{k}^t \leq  / \beta KT$,
	$$
	\mathbb{E}\left\|\theta_{i, k}^t-\theta_t\right\| \leq(1+c)^{K -1} \tilde{\eta}_{ t}\left(\mathbb{E}\left\|\nabla F\left(\theta_t\right)\right\|+ \sigma_{g, i}+\sigma\right), \quad \forall k=1, \ldots, K,
	$$
	where $\tilde{\eta}_{ t}=\sum_{k=0}^{K-1} \eta_{k}^t$.
\end{lemma}

\begin{proof}

Considering local update  of FedSWA
\begin{align}
	\begin{split}
		\mathbb{E}\left\|\theta_{i, k+1}-\theta_t\right\| & =\mathbb{E}\left\|\theta_{i, k}^t-\eta_{k}^t g_i\left(\theta_{i, k}^t\right)-\theta_t\right\| \\
		& \leq \mathbb{E}\left\|\theta_{i, k}^t-\theta_t-\eta_{k}^t\left(g_i\left(\theta_{i, k}^t\right)-g_i\left(\theta_t\right)\right)\right\|+\eta_{k}^t \mathbb{E}\left\|g_i\left(\theta_t\right)\right\| \\
		& \leq^a\left(1+\beta \eta_{k}^t\right) \mathbb{E}\left\|\theta_{i, k}^t-\theta_t\right\|+\eta_{k}^t \mathbb{E}\left\|g_i\left(\theta_t\right)\right\| \\
		& \leq\left(1+\beta \eta_{k}^t\right) \mathbb{E}\left\|\theta_{i, k}^t-\theta_t\right\|+\eta_{k}^t\left(\mathbb{E}\left\|g_i\left(\theta_t\right)-\nabla F_i\left(\theta_t\right)\right\|+\mathbb{E}\left\|\nabla F_i\left(\theta_t\right)\right\|\right) \\
		& \leq^b\left(1+\beta \eta_{k}^t\right) \mathbb{E}\left\|\theta_{i, k}^t-\theta_t\right\|+\eta_{k}^t\left(\mathbb{E}\left\|\nabla F_i\left(\theta_t\right)\right\|+\sigma\right),
	\end{split}
\end{align}
where $a$ is following \ref{L1}, and we use Assumptions  \ref{A3} in $b$ . With $\left(1+\beta \eta_{k}^t\right) \leq c$, Unrolling the above and noting $\theta_{i, 0}=\theta_t$ yields
\begin{align}
	\begin{split}
		\mathbb{E}\left\|\theta_{i, k}^t-\theta_t\right\| & \leq \sum_{l=0}^{k-1} \eta_{l}^t\left(\mathbb{E}\left\|\nabla F_i\left(\theta_t\right)\right\|+\sigma\right)(1+c)^{k-1-l} \\
		& \leq \sum_{l=0}^{K-1} \eta_{l}^t\left(\mathbb{E}\left\|\nabla F_i\left(\theta_t\right)\right\|+\sigma\right)(1+c)^{K-1} \\
		& \leq(1+c)^{K-1} \tilde{\eta}_{ t}\left(\mathbb{E}\left\|\nabla F\left(\theta_t\right)\right\|+ \sigma_{g, i}+\sigma\right),
	\end{split}
\end{align}
where the last inequality follows Assumption \ref{A5}, $\tilde{\eta}_{ t}=\sum_{k=0}^{K-1} \eta_{k}^t$.
\end{proof}

\begin{lemma}\label{E5}
	Given Assumptions \ref{A2}-\ref{A5}, for $\eta_{k}^t \leq  / \beta KT$, we have
	$$
	\mathbb{E}\left\|g_i\left(\theta_{i, k}^t\right)\right\| \leq\left(1+(1+c)^{K-1} \beta \tilde{\eta}_{ t}\right)\left(\mathbb{E}\left\|\nabla F\left(\theta_t\right)\right\|+ \sigma_{g, i}+\sigma\right)
	$$
	where $g_i(\cdot)$ is the sampled gradient of client $i, \tilde{\eta}_{ t}=\sum_{k=0}^{K-1} \eta_{k}^t$.
\end{lemma}
\begin{proof}

we obtain
\begin{align}
	\begin{split}
		\mathbb{E}\left\|g_i\left(\theta_{i, k}^t\right)\right\| & \leq \mathbb{E}\left\|g_i\left(\theta_{i, k}^t\right)-\nabla F_i\left(\theta_{i, k}^t\right)\right\|+\mathbb{E}\left\|\nabla F_i\left(\theta_{i, k}^t\right)\right\| \\
		& \leq^a \mathbb{E}\left\|\nabla F_i\left(\theta_{i, k}^t\right)\right\|+\sigma \\
		& \leq \mathbb{E}\left\|\nabla F_i\left(\theta_t\right)\right\|+\mathbb{E}\left\|\nabla F_i\left(\theta_{i, k}^t\right)-\nabla F_i\left(\theta_t\right)\right\|+\sigma \\
		& \leq^b \mathbb{E}\left\|\nabla F\left(\theta_t\right)\right\|+\mathbb{E}\left\|\nabla F_i\left(\theta_t\right)-\nabla F\left(\theta_t\right)\right\|+\beta \mathbb{E}\left\|\theta_{i, k}^t-\theta_t\right\|+\sigma \\
		& \leq^c\left(1+(1+c)^{K-1} \beta \tilde{\alpha}_i\right)\left(\mathbb{E}\left\|\nabla F\left(\theta_t\right)\right\|+ \sigma_{g, i}+\sigma\right) .
	\end{split}
\end{align}
$a$ is following Assumption \ref{A3}, $b$ is is following Assumption \ref{A4},
$c$ is is following Assumption \ref{A5} and \ref{L1} and \ref{E4}.
\end{proof}

\begin{theorem} (FedSWA). Suppose Assumptions \ref{A2}-\ref{A5} hold and consider FedSWA. Let $k=K, \forall i \in[m]$ and $\eta_{k}^t \leq \frac{1}{\beta KT}$. Then,
	$$
	\varepsilon_{\text {gen }}\leq \frac{2L \tilde{c}}{m n} \frac{1}{\beta} e^{1+\frac{1}{ T}}\left(\mathbb{E}\left\|\nabla F\left(\theta_t\right)\right\|+\sigma_g+\sigma\right),
	$$
	$\tilde{c}\leq 1+(1+c)^{K-1} \beta \tilde{\eta}_{ t}$.
\end{theorem}

\begin{proof}

For client $i$, there are two cases to consider. In the first case, SGD selects non-perturbed samples in $\mathcal{S}$ and $\mathcal{S}^{(i)}$, which happens with probability $1-1 / n$. Then, with \ref{L1}, we have
\begin{align}
	\begin{split}
		\left\|\theta_{i, k+1}^t-\theta_{i, k+1}^{t,\prime}\right\| \leq\left(1+\beta \eta_{k}^t\right)\left\|\theta_{i, k}^t-\theta_{i, k}^{t,\prime}\right\| .
	\end{split}
\end{align}
In the second case, SGD encounters the perturbed sample at time step $k$, which happens with probability $1 / n$. Then, we have
\begin{align}
	\begin{split}
		\left\|\theta_{i, k+1}^t-\theta_{i, k+1}^{t,\prime}\right\| & =\left\|\theta_{i, k}^t-\theta_{i, k}^{t,\prime}-\eta_{k}^t\left(g_i\left(\theta_{i, k}^t\right)-g_i^{\prime}\left(\theta_{i, k}^{t,\prime}\right)\right)\right\| \\
		& \leq\left\|\theta_{i, k}^t-\theta_{i, k}^{t,\prime}-\eta_{k}^t\left(g_i\left(\theta_{i, k}^t\right)-g_i\left(\theta_{i, k}^{t,\prime}\right)\right)\right\|+\eta_{k}^t\left\|g_i\left(\theta_{i, k}^{t,\prime}\right)-g_i^{\prime}\left(\theta_{i, k}^{t,\prime}\right)\right\| \\
		& \leq\left(1+\beta \eta_{k}^t\right)\left\|\theta_{i, k}^t-\theta_{i, k}^{t,\prime}\right\|+\eta_{k}^t\left\|g_i\left(\theta_{i, k}^{t,\prime}\right)-g_i^{\prime}\left(\theta_{i, k}^{t,\prime}\right)\right\| .
	\end{split}
\end{align}

Combining these two cases for client $i$ we have
\begin{align}
	\begin{split}
		\mathbb{E}\left\|\theta_{i, k+1}^t-\theta_{i, k+1}^{t,\prime}\right\| \leq & \left(1+\beta \eta_{k}^t\right) \mathbb{E}\left\|\theta_{i, k}^t-\theta_{i, k}^{t,\prime}\right\|+\frac{\eta_{k}^t}{n} \mathbb{E}\left\|g_i\left(\theta_{i, k}^{t,\prime}\right)-g_i^{\prime}\left(\theta_{i, k}^{t,\prime}\right)\right\| \\
		\leq & \left(1+\beta \eta_{k}^t\right) \mathbb{E}\left\|\theta_{i, k}^t-\theta_{i, k}^{t,\prime}\right\|+\frac{\eta_{k}^t}{n} \mathbb{E}\left\|g_i\left(\theta_{i, k}^t\right)\right\| \\
		\leq & \left(1+\beta \eta_{k}^t\right) \mathbb{E}\left\|\theta_{i, k}^t-\theta_{i, k}^{t,\prime}\right\|+\frac{2 \eta_{k}^t}{n}\left(1+(1+c)^{K-1} \beta \tilde{\eta}_{ t}\right)(\sigma +\mathbb{E}\left\|\nabla F\left(\theta_t\right)\right\|+ \sigma_{g, i}) \\
		\leq & \left(1+\beta \eta_{k}^t\right) \mathbb{E}\left\|\theta_{i, k}^t-\theta_{i, k}^{t,\prime}\right\|+\frac{2 \eta_{k}^t \tilde{c}}{n}\left(\mathbb{E}\left\|\nabla F\left(\theta_t\right)\right\|+ \sigma_{g, i}+\sigma\right)
	\end{split}\label{eq28}
\end{align}
where the last inequality follows Assumption \ref{L1}.
We let $ 1+(1+c)^{K-1} \beta \tilde{\eta}_{ t}\leq\tilde{c}$ and $\tilde{\eta}_{ t}=\sum_{k=0}^{K-1} \eta_{k}^t$. Then unrolling (\ref{eq28}) with \ref{E5} gives
\begin{align}
	\begin{split}
		& \mathbb{E}\left\|\theta_{i, K}^t-\theta_{i, K}^{t,\prime}\right\| \leq \prod_{k=0}^{K-1}\left(1+\beta \eta_{k}^t\right) \mathbb{E}\left\|\theta_t-\theta_t^{\prime}\right\|+\left(\frac{2}{n} \sum_{k=0}^{K-1} \eta_{k}^t \tilde{c} \prod_{l=k+1}^{K-1}\left(1+\beta \eta_{l}^t\right)\left(\mathbb{E}\left\|\nabla F\left(\theta_t\right)\right\|+ \sigma_{g, i}+\sigma\right)\right) \\
		& \leq e^{\beta \tilde{\eta}_{ t}} \mathbb{E}\left\|\theta_t-\theta_t^{\prime}\right\|+\frac{2}{n} \tilde{c} \tilde{\eta}_{ t} e^{\beta \tilde{\eta}_{ t}}\left(\mathbb{E}\left\|\nabla F\left(\theta_t\right)\right\|+ \sigma_{g, i}+\sigma\right) . \\
	\end{split}
\end{align}
With $\theta_{t+1}= \frac{1}{m}\sum_{i=1}^m \theta_{i, K}^t$ and $\theta_{t+1}^{\prime}= \frac{1}{m}\sum_{i=1}^m \theta_{i, K}^{t,\prime}$, we have
\begin{align}
	\begin{split}
		& \mathbb{E}\left\|\theta_{t+1}-\theta_{t+1}^{\prime}\right\| \leq \sum_{i=1}^m \frac{1}{m} \mathbb{E}\left\|\theta_{i, K}^t-\theta_{i, K}^{t,\prime}\right\| \\
		& \leq e^{\beta \tilde{\eta}_{ t}} \mathbb{E}\left\|\theta_t-\theta_t^{\prime}\right\|+\frac{2}{mn} \tilde{c} \tilde{\eta}_{ t} e^{\beta \tilde{\eta}_{ t}}\left(\mathbb{E}\left\|\nabla F\left(\theta_t\right)\right\|+ \sigma_{g, i}+\sigma\right) \\
	\end{split}
\end{align}
Further, unrolling the above over $t$ and noting $\theta_0=\theta_0^{\prime}$, we obtain
\begin{align}
	\begin{split}
		\mathbb{E}\left\|\theta_T-\theta_T^{\prime}\right\| \leq \frac{2 \tilde{c}}{mn} \sum_{t=0}^{T-1} \exp \left(\beta \sum_{l=t+1}^{T-1} \tilde{\eta}_{ t}\right) \tilde{\eta}_{ t} e^{\beta \tilde{\eta}_{ t}}\left(\mathbb{E}\left\|\nabla F\left(\theta_t\right)\right\|+ \sigma_{g}+\sigma\right) ,
	\end{split}
\end{align}
We simplify this inequality
\begin{align}
	\begin{split}
		&\mathbb{E}\left\|\theta_T-\theta_T^{\prime}\right\| \\
        &\leq \frac{2 \tilde{c}}{m n} \frac{1}{\beta} e^{\frac{1}{T}+1}\left(\mathbb{E}\left\|\nabla F\left(\theta_t\right)\right\|+\sigma_g+\sigma\right)\\
        &\leq \frac{2 \tilde{c}}{m n} \frac{1}{\beta} e^{\frac{1}{T}+1}\left(L+\sigma_g+\sigma\right),
	\end{split}
\end{align}
With the Assumption \ref{A2},
$$
\varepsilon_{\text {gen }}\leq \frac{2L \tilde{c}}{m n} \frac{1}{\beta} e^{1+\frac{1}{ T}}\left(\mathbb{E}\left\|\nabla F\left(\theta_t\right)\right\|+\sigma_g+\sigma\right),
$$
With
$$
\tilde{c}=1+\left(2+\frac{1}{T K}\right)^{K-1} \frac{1}{T} \gg 1
$$
When the diminishing stepsizes are chosen in the statement of the theorem, we conclude the proof.
    
\end{proof}

\section{Theoretical Results of FedMoSWA}
\subsection{ Generalization Analysis for FedMoSWA under strongly convex setting }	

where $\boldsymbol{g}_{i,k}^{(t)}$ denotes the stochastic gradient at the point $\boldsymbol{\theta}_{i,k}^{(t)}$.
\begin{lemma}\label{F1}
Suppose Assumptions \ref{A1}-\ref{A5} hold. Then for FedMoSWA with $\eta_k^{t} \leq 1 / \beta  KT$, $\alpha=1$,
	$$
	\mathbb{E}\left\|\theta_{i, K}^t-\theta_t\right\|  \leq \tilde{\eta}_t(1-b)^{K-1} \left( L+\sigma\right)
	$$
	where $\tilde{\eta}_{ t}=\sum_{k=0}^{K-1} \eta_k^{t}$, $b=\frac{\eta^{t} \beta \mu}{\beta+\mu}$.

\end{lemma}
\begin{proof}

 Considering local update  of FedMoSWA, $\boldsymbol{c}_i^t=g_i\left(\theta_t\right)$, $\boldsymbol{m}^t=(1-\gamma)\boldsymbol{m}^{t-1}+\gamma \frac{1}{m} \sum g_i\left(\theta_t\right)$,
	\begin{align}
		\begin{split}
			& \mathbb{E}\left\|\theta_{i, k+1}^t-\theta_t\right\| \\
			& =\mathbb{E}\left\|\theta_{i, k+1}^t-\eta_k^t\left(g_i\left(\theta_{i, k+1}^t\right)-\boldsymbol{c}_i^t+\boldsymbol{m}^t\right)-\theta_t\right\| \\
			& =\mathbb{E}\left\|\theta_{i, k+1}^t-\theta_t-\eta_k^t g_i\left(\theta_{i, k+1}^t\right)+\eta_k^t\left(g_i\left(\theta_t\right)+\boldsymbol{m}^t\right)\right\| \\
			& =\mathbb{E}\left\|\theta_{i, k+1}^t-\theta_t-\eta_k^t g_i\left(\theta_{i, k+1}^t\right)+\eta_k^t g_i\left(\theta_t\right)\right\|+\eta_k^t\mathbb{E}\left\|\boldsymbol{m}^t\right\| \\
			& \leq^a \mathbb{E}\left\|\theta_{i, k+1}^t-\theta_t-\eta_k^t\left(g_i\left(\theta_{i k+1}^t\right)-g_i\left(\theta_t\right)\right)\right\|+\eta_k^t \sum_{l=0}^t\left(1-\gamma\right)^l\left(\gamma \mathbb{E}\left\|\nabla F\left(\theta_l\right)\right\|+\sigma\right) \\
			& \leq^b\left(1-\frac{\eta_k^{t} \beta \mu}{\beta+\mu}\right) \mathbb{E}\left\|\theta_{i, k+1}^t-\theta_t\right\|+\eta_k^t\mathbb{E}\left\|\boldsymbol{m}^t\right\|\\
			& \leq\left(1-\frac{\eta_k^{t} \beta \mu}{\beta+\mu}\right) \mathbb{E}\left\|\theta_{i, k+1}^t-\theta_t\right\|+\eta_k^t \sum_{l=0}^t\left(1-\gamma\right)^l\left(\gamma \mathbb{E}\left\|\nabla F\left(\theta_l\right)\right\|+\sigma\right)\\
			& \leq^c\left(1-\frac{\eta_k^{t} \beta \mu}{\beta+\mu}\right) \mathbb{E}\left\|\theta_{i, k+1}^t-\theta_t\right\|+\eta_k^t \sum_{l=0}^t\left(1-\gamma\right)^l\left(\gamma L+\sigma\right)
		\end{split}
	\end{align}\label{eq33}
    where, $a$ is following Assumption \ref{A3}, $b$ is following Lemma \ref{L1}, c is following Assumption  \ref{A2}.
	Below we find an upper bound for $\mathbb{E}\left\|\boldsymbol{m}^t\right\|$,
	\begin{align}
		\begin{split}
			& \mathbb{E}\left\|\boldsymbol{m}^t\right\|=\mathbb{E}\left\|\left(1-\gamma\right) \boldsymbol{m}^{t-1}+\gamma \frac{1}{m} \sum_{i=1}^m g_i\left(\theta_t\right)\right\| \\
			& \leq\left(1-\gamma\right) \mathbb{E}\left\|\boldsymbol{m}^{t-1}\right\|+\gamma \mathbb{E}\left\|\frac{1}{m} \sum_{i=1}^m g_i\left(\theta_t\right)\right\| \\
			& \leq\left(1-\gamma\right) \mathbb{E}\left\|\boldsymbol{m}^{t-1}\right\|+\gamma \mathbb{E}\left\|\nabla F\left(\theta_t\right)\right\|+\gamma\sigma \\
			& \leq^a \sum_{l=0}^t\left(1-\gamma\right)^l\left(\gamma \mathbb{E}\left\|\nabla F\left(\theta_l\right)\right\|+\gamma\sigma\right)\\
			& \leq^b L+\sigma
		\end{split}
	\end{align}
	where $a$ follows Assumption \ref{A3}; $b$ follows Lemma \ref{L1}. 
	Unrolling the (\ref{eq33}) and noting $\theta_{i, 0}=\theta_t$, with $b=\frac{\eta^{t} \beta \mu}{\beta+\mu}$ for some $b>0$, $\tilde{\eta}_{ t}=\sum_{k=0}^{K-1} \eta_k^{t}$ yields
	
	\begin{align}
		\begin{split}
			& \mathbb{E}\left\|\theta_{i, K}^t-\theta_t\right\| \\
			&\leq(1-b)^{K-1} \tilde{\eta}_t\mathbb{E}\left\|\boldsymbol{m}^t\right\|\\
			& \leq \sum_{l=0}^{K-1} \eta_l^{t} \sum_{p=0}^t\left(1-\gamma\right)^p\left(\gamma \mathbb{E}\left\|\nabla F\left(\theta_p\right)\right\|+\gamma\sigma\right)(1-b)^{K-1-l}\\
			& \leq \sum_{l=0}^{K-1} \eta_l^{t} \left( L+\sigma\right)(1-b)^{K-1-l}\\
			& \leq\tilde{\eta}_t(1-b)^{K-1} \left( L+\sigma\right)
		\end{split}
	\end{align}
\end{proof}

\begin{lemma}\label{F2}
	Given Assumptions \ref{A1}-\ref{A5} , for $\eta_{k}^t \leq 1 /  \beta KT$ we have
	$$
	\mathbb{E}\left\|g_i\left(\theta_{i, k}^t\right)\right\| \leq\left(1+\beta \tilde{\eta}_{t}(1+b)^{K-1}\right)\left(L+\sigma\right)+\sigma_{g,i}
	$$

	where $g_i(\cdot)$ is the sampled gradient of client $i, \tilde{\eta}_{t}=\sum_{k=0}^{K-1} \eta_{ k}^t$.
\end{lemma}
\begin{proof} 
    we obtain
	\begin{align}
		\begin{split}
			\mathbb{E}\left\|g_i\left(\theta_{i, k}^t\right)\right\| & \leq \mathbb{E}\left\|g_i\left(\theta_{i, k}^t\right)-\nabla F_i\left(\theta_{i, k}^t\right)\right\|+\mathbb{E}\left\|\nabla F_i\left(\theta_{i, k}^t\right)\right\| \\
			& \leq^a \mathbb{E}\left\|\nabla F_i\left(\theta_{i, k}^t\right)\right\|+\sigma \\
			& \leq \mathbb{E}\left\|\nabla F_i\left(\theta_t\right)\right\|+\mathbb{E}\left\|\nabla F_i\left(\theta_{i, k}^t\right)-\nabla F_i\left(\theta_t\right)\right\|+\sigma \\
			& \leq^b \mathbb{E}\left\|\nabla F\left(\theta_t\right)\right\|+\mathbb{E}\left\|\nabla F_i\left(\theta_t\right)-\nabla F\left(\theta_t\right)\right\|+\beta \mathbb{E}\left\|\theta_{i, k}^t-\theta_t\right\|+\sigma \\
			& \leq^c\left(\mathbb{E}\left\|\nabla F\left(\theta_t\right)\right\|+ \sigma_{g, i}+\sigma\right)+\sum_{l=0}^{K-1} \eta_l^{t} \sum_{p=0}^l\left(1-\gamma\right)^p\left(\gamma \mathbb{E}\left\|\nabla F\left(\theta_p\right)\right\|+\sigma\right)(1-b)^{K-1-l}\\
			&\leq^d \left(1+\beta \tilde{\eta}_{t}(1+b)^{K-1}\right)\left(L+\sigma\right)+\sigma_{g,i}
		\end{split}
	\end{align}
$a$ is following Assumption \ref{A3}, $b$ is is following Assumption \ref{A4},
$c$ is is following Assumption \ref{A5} and Lemma \ref{L1} and Lemma \ref{F1}, $d$ is following $\left\|\nabla F\left(\theta_t\right)\right\|\leq L$ .
Similarly, using same techniques we have
	\begin{align}
		\begin{split}
			\mathbb{E}\left\|g_i\left(\theta_t\right)\right\| & \leq \mathbb{E}\left\|\nabla F_i\left(\theta_t\right)\right\|+\sigma \\
			& \leq \mathbb{E}\left\|\nabla F_i\left(\theta_t\right)-\nabla F\left(\theta_t\right)\right\|+\mathbb{E}\left\|\nabla F\left(\theta_t\right)\right\|+\sigma \\
			& \leq \sigma_{g,i}+\mathbb{E}\left\|\nabla F\left(\theta_t\right)\right\|+\sigma .
		\end{split}
	\end{align}
\end{proof}

\begin{theorem}
	Suppose Assumptions \ref{A1}-\ref{A5} hold and consider FedMoSWA (Algorithm 1). Let $\left\{\theta_t\right\}_{t=0}^T$ and $\left\{\theta_t^{\prime}\right\}_{t=0}^T$ be two trajectories of the server induced by neighboring datasets $\mathcal{S}$ and $\mathcal{S}^{(i)}$, respectively. Suppose $\theta_0=\theta_0^{\prime}$. Then,
	$$
	\mathbb{E}\left\|\theta_T-\theta_T^{\prime}\right\| \leq \frac{2}{m n} \frac{1}{\beta} e^{1-\frac{\mu}{(\beta+\mu)}}(\tilde{b} L+ \sigma_g+\tilde{b} \sigma),
	$$
	$$
	\varepsilon_{\text {gen }} \leq \frac{2 L }{m n} \frac{1}{\beta} e^{1-\frac{\mu}{(\beta+\mu) T}}\left(\tilde{b}L+\sigma_g+\tilde{b}\sigma\right)
	$$
	where $\tilde{\eta}_{t}=\sum_{k=0}^{K-1} \eta^t_{ k}$. $\tilde{b}$ be an upper bound of $1+(1-b)^{K-1} \beta\tilde{\eta}_{t}$.
\end{theorem}
\begin{proof}
For client $i$, there are two cases to consider. In the first case, SGD selects the index of an sample at local step $k$ on which is identical in $\mathcal{S}$ and $\mathcal{S}^{(i)}$. In this sense, we have,
\begin{align}
	\begin{split}
		& \left\|\theta_{j, k+1}^t-\theta_{j, k+1}^{t, \prime}\right\| \\
		& \leq\left\|\theta_{j, k+1}^t-\theta_{j, k+1}^{t, \prime}-\eta_k^t\left(g_j\left(\theta_{j, k+1}^t\right)-g_j\left(\theta_{j, k}^{t, \prime}\right)-\boldsymbol{c}_i^t+\boldsymbol{m}^t-\boldsymbol{c}_i^{t}+\boldsymbol{m}^{t}\right)\right\| \\
		& \leq^a\left(1-\frac{\eta^t_{ k} \beta \mu}{\beta+\mu}\right)\left\|\theta_{i, k}^t-\theta_{i, k}^{t,\prime}\right\|
	\end{split}
\end{align}
where $a$ is following Lemma \ref{L1}.
And this case happens with probability $1-1 / n$ (since only one sample is perturbed for client $i$ ).

In the second case, SGD encounters the perturbed sample at local time step $k$, which happens with probability $1 / n$. We denote the gradient of this perturbed sample as $g_i^{\prime}(\cdot)$. Then,
\begin{align}
	\begin{split}
		\left\|\theta_{i, k+1}^t-\theta_{i, k+1}^{t,\prime}\right\|
		&\leq\left\|\theta_{j, k+1}^t-\theta_{j, k+1}^{t, \prime}-\eta_k^t\left(g_j\left(\theta_{j, k+1}^t\right)-g_j^{\prime}\left(\theta_{j, k}^{t, \prime}\right)-\boldsymbol{c}_i^t+\boldsymbol{m}^t-\boldsymbol{c}_i^{t}+\boldsymbol{m}^{t}\right)\right\| \\
		& =\left\|\theta_{i, k}^t-\theta_{i, k}^{t,\prime}-\eta^t_{ k}\left(g_i\left(\theta_{i, k}^t\right)-g_i^{\prime}\left(\theta_{i, k}^{t,\prime}\right)\right)\right\| \\
		& \leq\left\|\theta_{i, k}^t-\theta_{i, k}^{t,\prime}-\eta^t_{ k}\left(g_i\left(\theta_{i, k}^t\right)-g_i\left(\theta_{i, k}^{t,\prime}\right)\right)\right\|+\eta^t_{ k}\left\|g_i\left(\theta_{i, k}^{t,\prime}\right)-g_i^{\prime}\left(\theta_{i, k}^{t,\prime}\right)\right\| \\
		& \leq\left(1-\frac{\eta^t_{ k} \beta \mu}{\beta+\mu}\right)\left\|\theta_{i, k}^t-\theta_{i, k}^{t,\prime}\right\|+\eta^t_{ k}\left\|g_i\left(\theta_{i, k}^{t,\prime}\right)-g_i^{\prime}\left(\theta_{i, k}^{t,\prime}\right)\right\| .
	\end{split}
\end{align}
Combining these two cases we have for client $i$,
\begin{align}
	\begin{split}
		& \mathbb{E}\left\|\theta_{i, k+1}^t-\theta_{i, k+1}^{t, \prime}\right\| \leq\left(1-\frac{\eta^t_{ k} \beta \mu}{\beta+\mu}\right) \mathbb{E}\left\|\theta_{i, k}^t-\theta_{i, k}^{t, \prime}\right\|+\frac{\eta_k^t}{n} \mathbb{E}\left\|g_i\left(\theta_{i, k}^{t, \prime}\right)-g_i^{\prime}\left(\theta_{i, k}^{t, \prime}\right)\right\| \\
		& \leq^a\left(1-\frac{\eta^t_{ k} \beta \mu}{\beta+\mu}\right) \mathbb{E}\left\|\theta_{i, k}^t-\theta_{i, k}^{t, \prime}\right\|+\frac{2\eta_k^t}{n} \mathbb{E}\left\|g_i\left(\theta_{i, k}^t\right)\right\| \\
		& \leq^b\left(1-\frac{\eta^t_{ k} \beta \mu}{\beta+\mu}\right) \mathbb{E}\left\|\theta_{i, k}^t-\theta_{i, k}^{t, \prime}\right\|+\frac{2\eta_k^t}{n}\left(1+\beta \tilde{\eta}_{t}(1-b)^{K-1}\right)\left(L+\sigma\right)+\frac{2\eta_k^t}{n} \sigma_{g},\\
	\end{split}
\end{align}
where the $a$ inequality follows that $g_i(\cdot)$ and $g_i^{\prime}(\cdot)$ are sampled from the same distribution, $b$ is following Lemma \ref{F2}. We let $1+(1-b)^{K-1} \beta\tilde{\eta}_{t}\leq \tilde{b}$, $\tilde{\eta}_{t}=\sum_{k=0}^{K-1} \eta^t_{ k}$. Then unrolling it we have,
\begin{align}
	\begin{split}
		& \mathbb{E}\left\|\theta_{i, K}^t-\theta_{i, K}^{t, \prime}\right\| \leq \prod_{k=0}^{K-1}\left(1-\frac{\eta^t_{ k} \beta \mu}{\beta+\mu}\right) \mathbb{E}\left\|\theta_t-\theta_t^{\prime}\right\|+\left(\frac{2}{n} \sum_{k=0}^{K-1} \eta_k^t \prod_{l=k+1}^{K-1}\left(1-\frac{\eta^t_{ k} \beta \mu}{\beta+\mu}\right)\right. \\
		& \left.\cdot\left(\left(\mathbb{E}\left\|\nabla F\left(\theta_t\right)\right\|+ \sigma_{g, i}+\sigma\right)+(1-b)^{K-1} \tilde{\eta}_t\mathbb{E}\left\|\boldsymbol{m}^t\right\|\right)\right) \\
		& \leq e^{-\frac{\tilde{\eta}_t \beta \mu}{\beta+\mu} } \mathbb{E}\left\|\theta_t-\theta_t^{\prime}\right\|+\frac{2}{n}  \tilde{\eta}_t e^{-\frac{\tilde{\eta}_t \beta \mu}{\beta+\mu} }\left(\left(\mathbb{E}\left\|\nabla F\left(\theta_t\right)\right\|+ \sigma_{g, i}+\sigma\right)+(1-b)^{K-1} \tilde{\eta}_t\mathbb{E}\left\|\boldsymbol{m}^t\right\|\right) .
	\end{split}
\end{align}
Then, with $\theta_{t+1}= \frac{1}{m}\sum_{i=1}^m \theta_{i, K}^t$ and $\theta_{t+1}^{\prime}= \frac{1}{m}\sum_{i=1}^m \theta_{i, K}^{t,\prime}$, we have
\begin{align}
	\begin{split}
		& \mathbb{E}\left\|\theta_{t+1}-\theta_{t+1}^{\prime}\right\| \leq \sum_{i=1}^m \frac{1}{m} \mathbb{E}\left\|\theta_{i, K}^t-\theta_{i, K}^{t, \prime}\right\| \\
		& \leq e^{-\frac{\tilde{\eta}_t \beta \mu}{\beta+\mu} } \mathbb{E}\left\|\theta_t-\theta_t^{\prime}\right\|+\frac{2}{m n} \tilde{\eta}_t e^{-\frac{\tilde{\eta}_t \beta \mu}{\beta+\mu} }\left(\left(\mathbb{E}\left\|\nabla F\left(\theta_t\right)\right\|+\sigma_{g, i}+\sigma\right)+(1-b)^{K-1} \tilde{\eta}_t\mathbb{E}\left\|\boldsymbol{m}^t\right\|\right)
	\end{split}
\end{align}
where we use Lemma \ref{F2} in the last step. 
Further, unrolling the above over $t$ and noting $\theta_0=\theta_0^{\prime}$, $b=\frac{ \eta^t \beta \mu}{\beta+\mu}=\frac{ \eta_k^t \beta \mu}{\beta+\mu}$  we obtain,
\begin{align}
	\begin{split}
		& \mathbb{E}\left\|\theta_T-\theta_T^{\prime}\right\| \\
		& \leq \frac{2}{m n} \sum_{t=0}^{T-1} \exp \left(-\frac{ \beta \mu}{\beta+\mu} \sum_{l=t+1}^{T-1} \tilde{\eta}_t\right) \tilde{\eta}_t e^{-\frac{ \tilde{\eta}_t \beta \mu}{\beta+\mu}}\left(\left(\mathbb{E}\left\|\nabla F\left(\theta_t\right)\right\|+\sigma_g+\sigma\right)+(1-b)^{K-1}  \tilde{\eta}_t \mathbb{E}\left\|\boldsymbol{m}^t\right\|\right) \\
		& \leq \frac{2}{m n} \sum_{t=0}^{T-1} \exp \left(-\frac{ \beta \mu}{\beta+\mu} \sum_{l=t+1}^{T-1} \tilde{\eta}_t\right)  \tilde{\eta}_t e^{-\frac{ \tilde{\eta}_t\beta \mu}{\beta+\mu}}\left(\left(1+\beta \tilde{\eta}_t(1-b)^{K-1}\right)(L+\sigma)+\sigma_g\right) \\
		&
	\end{split}
\end{align}
When the diminishing stepsizes are chosen in the statement of the theorem, we conclude the proof.
\end{proof}
\subsection{ Generalization Analysis for FedMoSWA under non-convex setting }	

\begin{lemma}\label{F4}
	Suppose Assumptions \ref{A1}-\ref{A5} hold. Then for FedMoSWA with $\eta_{k}^t \leq 1/ \beta KT$,
	$$
	\mathbb{E}\left\|\theta_{i, k}^t-\theta_t\right\|\leq \tilde{\eta}_t(1+c)^{k-1}(L+\sigma)
	$$
	where $\tilde{\eta}_{ t}=\sum_{k=0}^{K-1} \eta_{k}^t$.
\end{lemma}
\begin{proof}
Considering local update  of FedMoSWA, $\boldsymbol{c}_i^t=g_i\left(\theta_t\right)$, $\boldsymbol{m}^t=(1-\gamma)\boldsymbol{m}^{t-1}+\gamma \frac{1}{m} \sum g_i\left(\theta_t\right)$,
\begin{align}
	\begin{split}
		& \mathbb{E}\left\|\theta_{i, k}^t-\theta_t\right\| \\
		& =\mathbb{E}\left\|\theta_{i, k}^t-\eta_k^t\left(g_i\left(\theta_{i, k}^t\right)-\boldsymbol{c}_i^t+\boldsymbol{m}^t\right)-\theta_t\right\| \\
		& =\mathbb{E}\left\|\theta_{i, k}^t-\theta_t-\eta_k^t g_i\left(\theta_{i, k}^t\right)+\eta_k^t\left(g_i\left(\theta_t\right)+\boldsymbol{m}^t\right)\right\| \\
		& =\mathbb{E}\left\|\theta_{i, k}^t-\theta_t-\eta_k^t g_i\left(\theta_{i, k}^t\right)+\eta_k^t g_i\left(\theta_t\right)\right\|+\eta_k^t\mathbb{E}\left\|\boldsymbol{m}^t\right\| \\
		& \leq^a \mathbb{E}\left\|\theta_{i, k}^t-\theta_t-\eta_k^t\left(g_i\left(\theta_{i, k}^t\right)-g_i\left(\theta_t\right)\right)\right\|+\eta_k^t \sum_{l=0}^t\left(1-\gamma\right)^l\left(\gamma \mathbb{E}\left\|\nabla F\left(\theta_l\right)\right\|+\gamma\sigma\right) \\
		& \leq^b\left(1+\beta \eta_k^t\right) \mathbb{E}\left\|\theta_{i, k+1}^t-\theta_t\right\|+\eta_k^t \sum_{l=0}^t\left(1-\gamma\right)^l\left(\gamma \mathbb{E}\left\|\nabla F\left(\theta_l\right)\right\|+\gamma\sigma\right)
	\end{split}
\end{align}
    where, $a$ is following Assumption \ref{A3}, $b$ is following Lemma \ref{L1}.
then we can get,
\begin{align}
	\begin{split}
		& \mathbb{E}\left\|\theta_{i, k+1}^t-\theta_t\right\| \leq\left(1+ \beta \eta_k^t\right) \mathbb{E}\left\|\theta_{i, k}^t-\theta_t\right\|+ \eta_k^t \sum_{l=0}^t\left(1-\gamma\right)^{l}\left(\gamma \mathbb{E}\left\|\nabla F\left(\theta_l\right)\right\|+\gamma \sigma\right) \\
		& \leq\left(1+\beta \eta_k^t\right) \mathbb{E}\left\|\theta_{i, k}^t-\theta_t\right\|+ \eta_k^t(L+\sigma) \\
		& \leq^c \sum_{k=0}^K \prod_{l=k}^K\left(1+ \beta \eta_l^t\right)  \eta_l^t(L+\sigma) \\
		& \leq  \tilde{\eta}_t(1+c)^{k-1}(L+\sigma)
	\end{split}
\end{align}
, $c$ is following Assumption  \ref{A2}.
Below we find an upper bound for $\mathbb{E}\left\|\boldsymbol{m}^t\right\|$,
\begin{align}
	\begin{split}
		& \mathbb{E}\left\|\boldsymbol{m}^t\right\|=\mathbb{E}\left\|\left(1-\gamma\right) \boldsymbol{m}^{t-1}+\gamma \frac{1}{m} \sum_{i=1}^m g_i\left(\theta_t\right)\right\| \\
		& \leq\left(1-\gamma\right) \mathbb{E}\left\|\boldsymbol{m}^{t-1}\right\|+\gamma \mathbb{E}\left\|\frac{1}{m} \sum_{i=1}^m g_i\left(\theta_t\right)\right\| \\
		& \leq\left(1-\gamma\right) \mathbb{E}\left\|\boldsymbol{m}^{t-1}\right\|+\gamma \mathbb{E}\left\|\nabla F\left(\theta_t\right)\right\|+\gamma\sigma \\
		& \leq \sum_{l=0}^t\left(1-\gamma\right)^l\left(\gamma \mathbb{E}\left\|\nabla F\left(\theta_l\right)\right\|+\gamma\sigma\right)\\
		& \leq L+\sigma
	\end{split}
\end{align}
where we use Assumptions \ref{A5}.With $\left(1+\beta \eta_{k}^t\right) \leq c$, unrolling the above and noting $\theta_{i, 0}=\theta_t$ yields
\begin{align}
	\begin{split}
		& \mathbb{E}\left\|\theta_{i, K}^t-\theta_t\right\| \\
		&\leq(1+c)^{K-1} \tilde{\eta}_t\mathbb{E}\left\|\boldsymbol{m}^t\right\|\\
		& \leq \tilde{\eta}_t(1+c)^{K-1}\left(L+\sigma\right)
	\end{split}
\end{align}
\end{proof}
\begin{lemma}\label{F5}
	Given Assumptions \ref{A2}-\ref{A5} and considering of FedMoSWA, for $\eta_{k}^t \leq 1 / \beta KT$, we have
	$$
	\mathbb{E}\left\|g_i\left(\theta_{i, k}^t\right)\right\| \leq \left(1+\beta \tilde{\eta_t}(1+c)^{K-1}\right)\left(L+\sigma\right)+ \sigma_{g, i}
	$$
	where $g_i(\cdot)$ is the sampled gradient of client $i, \tilde{\eta}_{ t}=\sum_{k=0}^{K-1} \eta_{k}^t$.
\end{lemma}
\begin{proof}
\begin{align}
	\begin{split}
		\mathbb{E}\left\|g_i\left(\theta_{i, k}^t\right)\right\| & \leq \mathbb{E}\left\|g_i\left(\theta_{i, k}^t\right)-\nabla F_i\left(\theta_{i, k}^t\right)\right\|+\mathbb{E}\left\|\nabla F_i\left(\theta_{i, k}^t\right)\right\| \\
		& \leq^a \mathbb{E}\left\|\nabla F_i\left(\theta_{i, k}^t\right)\right\|+\sigma \\
		& \leq \mathbb{E}\left\|\nabla F_i\left(\theta_t\right)\right\|+\mathbb{E}\left\|\nabla F_i\left(\theta_{i, k}^t\right)-\nabla F_i\left(\theta_t\right)\right\|+\sigma \\
		& \leq^b \mathbb{E}\left\|\nabla F\left(\theta_t\right)\right\|+\mathbb{E}\left\|\nabla F_i\left(\theta_t\right)-\nabla F\left(\theta_t\right)\right\|+\beta \mathbb{E}\left\|\theta_{i, k}^t-\theta_t\right\|+\sigma \\
		& \leq^c\left(\mathbb{E}\left\|\nabla F\left(\theta_t\right)\right\|+ \sigma_{g, i}+\sigma\right)+\sum_{l=0}^{K-1} \eta_l^{t} \sum_{p=0}^l\left(1-\gamma\right)^p\left(\gamma \mathbb{E}\left\|\nabla F\left(\theta_p\right)\right\|+\sigma\right)(1+c)^{K-1-l}\\
		&\leq^d \left(1+\beta \tilde{\eta_t}(1+c)^{K-1}\right)\left(L+\sigma\right)+ \sigma_{g, i}
	\end{split}
\end{align}
$a$ is following Assumption \ref{A3}, $b$ is is following Assumption \ref{A4},
$c$ is is following Assumption \ref{A5} and Lemma \ref{L1} and Lemma \ref{F4}, $d$ is following $\left\|\nabla F\left(\theta_t\right)\right\|\leq L$ .
\end{proof}
\begin{theorem} (FedMoSWA). Suppose Assumptions \ref{A2}-\ref{A5} hold and consider FedMoSWA (Algorithm 1). Let $k=K, \forall i \in[m]$ and $\eta_{k}^t \leq \frac{1}{\beta KT}$, $\tilde{c}=\left(1+\beta \tilde{\eta}_t(1+c)^{K-1}\right)(L+\sigma)$. Then,
	$$
	\varepsilon_{\text {gen }} \leq \frac{2 L }{m n} \frac{1}{\beta} e^{1+\frac{1}{T}}\left(\tilde{c}L+\sigma_g+\tilde{c}\sigma\right)
	$$
	where , $\tilde{c}\leq 1+(1+c)^{K-1} \beta \tilde{\eta}_{ t}$, $\tilde{\eta}_{ t}=\sum_{k=0}^{K-1} \eta_{k}^t$.
\end{theorem}
\begin{proof}
For client $i$, there are two cases to consider. In the first case, SGD selects non-perturbed samples in $\mathcal{S}$ and $\mathcal{S}^{(i)}$, which happens with probability $1-1 / n$. Then, we have,
\begin{align}
	\begin{split}
		\left\|\theta_{i, k+1}^t-\theta_{i, k+1}^{t,\prime}\right\| \leq\left(1+\beta \eta_{k}^t\right)\left\|\theta_{i, k}^t-\theta_{i, k}^{t,\prime}\right\| .
	\end{split}
\end{align}
In the second case, SGD encounters the perturbed sample at time step $k$, which happens with probability $1 / n$. Then, we have
\begin{align}
	\begin{split}
		\left\|\theta_{i, k+1}^t-\theta_{i, k+1}^{t,\prime}\right\| & =\left\|\theta_{i, k}^t-\theta_{i, k}^{t,\prime}-\eta_{k}^t\left(g_i\left(\theta_{i, k}^t\right)-g_i^{\prime}\left(\theta_{i, k}^{t,\prime}\right)-\boldsymbol{c}_i^t+\boldsymbol{m}^t-\boldsymbol{c}_i^{t}+\boldsymbol{m}^{t}\right)\right\| \\
		& \leq\left\|\theta_{i, k}^t-\theta_{i, k}^{t,\prime}-\eta_{k}^t\left(g_i\left(\theta_{i, k}^t\right)-g_i\left(\theta_{i, k}^{t,\prime}\right)\right)\right\|+\eta_{k}^t\left\|g_i\left(\theta_{i, k}^{t,\prime}\right)-g_i^{\prime}\left(\theta_{i, k}^{t,\prime}\right)\right\| \\
		& \leq^a\left(1+\beta \eta_{k}^t\right)\left\|\theta_{i, k}^t-\theta_{i, k}^{t,\prime}\right\|+\eta_{k}^t\left\|g_i\left(\theta_{i, k}^{t,\prime}\right)-g_i^{\prime}\left(\theta_{i, k}^{t,\prime}\right)\right\| .
	\end{split}
\end{align}
where $a$ is following Lemma \ref{L1}.
Combining these two cases we have for client $i$,
\begin{align}
	\begin{split}
		& \mathbb{E}\left\|\theta_{i, k+1}^t-\theta_{i, k+1}^{t, \prime}\right\| \leq\left(1+\beta \eta_{k}^t\right) \mathbb{E}\left\|\theta_{i, k}^t-\theta_{i, k}^{t, \prime}\right\|+\frac{\eta_k^t}{n} \mathbb{E}\left\|g_i\left(\theta_{i, k}^{t, \prime}\right)-g_i^{\prime}\left(\theta_{i, k}^{t, \prime}\right)\right\| \\
		& \leq^a\left(1+\beta \eta_{k}^t\right) \mathbb{E}\left\|\theta_{i, k}^t-\theta_{i, k}^{t, \prime}\right\|+\frac{2\eta_k^t}{n} \mathbb{E}\left\|g_i\left(\theta_{i, k}^t\right)\right\| \\
		& \leq^c\left(1+\beta \eta_{k}^t\right) \mathbb{E}\left\|\theta_{i, k}^t-\theta_{i, k}^{t, \prime}\right\|+\frac{2 \eta_k^t}{n}\left(\left(\mathbb{E}\left\|\nabla F\left(\theta_t\right)\right\|+ \sigma_{g, i}+\sigma\right)+(1+c)^{K-1} \tilde{\eta}_t\mathbb{E}\left\|\boldsymbol{m}^t\right\|\right)\\
		&\leq^b \left(1+\beta \eta_{k}^t\right) \mathbb{E}\left\|\theta_{i, k}^t-\theta_{i, k}^{t, \prime}\right\| +\left(1+\beta \tilde{\eta_t}(1+c)^{K-1}\right)\left(L+\sigma\right)+ \sigma_{g, i},
	\end{split}
\end{align}
where the $a$ inequality follows that $g_i(\cdot)$ and $g_i^{\prime}(\cdot)$ are sampled from the same distribution, $b$ is following Lemma \ref{F4}, $c$ is following Lemma \ref{F5}. We let $1+(1-c)^{K-1} \beta\tilde{\eta}_{t}\leq \tilde{c}$, $\tilde{\eta}_{t}=\sum_{k=0}^{K-1} \eta^t_{ k}$. Then unrolling it we have,
\begin{align}
	\begin{split}
		& \mathbb{E}\left\|\theta_{i, K}^t-\theta_{i, K}^{t, \prime}\right\| \leq \prod_{k=0}^{K-1}\left(1+\beta \eta_{k}^t\right) \mathbb{E}\left\|\theta_t-\theta_t^{\prime}\right\| \\
		&+\left(\frac{2}{n} \sum_{k=0}^{K-1} \eta_k^t \prod_{l=k+1}^{K-1}\left(1+\beta \eta_{k}^t\right)\right.
		\left.\cdot\left(\left(\mathbb{E}\left\|\nabla F\left(\theta_t\right)\right\|+ \sigma_{g, i}+\sigma\right)+(1+c)^{K-1} \tilde{\eta}_t\mathbb{E}\left\|\boldsymbol{m}^t\right\|\right)\right) \\
		& \leq e^{\beta \tilde{\eta}_{t} } \mathbb{E}\left\|\theta_t-\theta_t^{\prime}\right\|+\frac{2}{n}  \tilde{\eta}_t e^{\beta \tilde{\eta}_{t} }\left(\left(\mathbb{E}\left\|\nabla F\left(\theta_t\right)\right\|+ \sigma_{g, i}+\sigma\right)+(1+c)^{K-1} \tilde{\eta}_t\mathbb{E}\left\|\boldsymbol{m}^t\right\|\right) .
	\end{split}
\end{align}
Then, with $\theta_{t+1}= \frac{1}{m}\sum_{i=1}^m \theta_{i, K}^t$ and $\theta_{t+1}^{\prime}= \frac{1}{m}\sum_{i=1}^m \theta_{i, K}^{t,\prime}$, we have
\begin{align}
	\begin{split}
		&\mathbb{E}\left\|\theta_{t+1}-\theta_{t+1}^{\prime}\right\| \leq \sum_{i=1}^m \frac{1}{m} \mathbb{E}\left\|\theta_{i, K}^t-\theta_{i, K}^{t, \prime}\right\| \\
		& \leq e^{\beta \tilde{\eta}_{t}} \mathbb{E}\left\|\theta_t-\theta_t^{\prime}\right\|+\frac{2}{m n} \tilde{\eta}_t e^{\beta \tilde{\eta}_{t}}\left(\left(\mathbb{E}\left\|\nabla F\left(\theta_t\right)\right\|+ \sigma_g+\sigma\right)+(1+c)^{K-1} \tilde{\eta}_t\mathbb{E}\left\|\boldsymbol{m}^t\right\|\right)
	\end{split}
\end{align}
where we use Lemma \ref{F4} in the last step. Further, unrolling the above over $t$ and noting $\theta_0=\theta_0^{\prime}$, we obtain
\begin{align}
	\begin{split}
		&\mathbb{E}\left\|\theta_T-\theta_T^{\prime}\right\| \leq \frac{2}{m n} \sum_{t=0}^{T-1} \exp \left(\beta \sum_{l=t+1}^{T-1} \tilde{\eta}_t\right) \tilde{\eta}_t e^{\beta \tilde{\eta}_{t}}\left(\left(\mathbb{E}\left\|\nabla F\left(\theta_t\right)\right\|+ \sigma_g+\sigma\right)+(1+c)^{K-1} \tilde{\eta}_t\mathbb{E}\left\|\boldsymbol{m}^t\right\|\right)\\
		& \leq \frac{2}{m n} \sum_{t=0}^{T-1} \exp \left(\beta \sum_{l=t+1}^{T-1} \tilde{\eta}_t\right) \tilde{\eta}_t e^{\beta \tilde{\eta}_{t}}\\
		&\cdot\left(\left(\mathbb{E}\left\|\nabla F\left(\theta_t\right)\right\|+ \sigma_g+\sigma\right)+(1+c)^{K-1} \tilde{\eta}_t\sum_{l=0}^t\left(1-\gamma\right)^l\left(\gamma \mathbb{E}\left\|\nabla F\left(\theta_l\right)\right\|+\gamma \sigma\right)\right)\\
		& \leq \frac{2}{m n} \sum_{t=0}^{T-1} \exp \left(\beta \sum_{l=t+1}^{T-1} \tilde{\eta}_t\right) \tilde{\eta}_t e^{\beta \tilde{\eta}_{t}} \\
		& \cdot\left(\mathbb{E}\left\|\nabla F\left(\theta_t\right)\right\|+ \sigma_{g, i}+\tilde{b} \sigma+(1+c)^{K-1} \tilde{\eta}_t \sum_{l=0}^t\left(1-\gamma\right)^{l}\left(\gamma \mathbb{E}\left\|\nabla F\left(\theta_l\right)\right\|\right)\right) \\
		& \leq \frac{2}{m n} \sum_{t=0}^{T-1} \exp \left(\beta \sum_{l=t+1}^{T-1} \tilde{\eta}_t\right) \tilde{\eta}_t e^{\beta \tilde{\eta}_{t}}(\left(1+\beta \tilde{\eta_t}(1+c)^{K-1}\right)\left(L+\sigma\right)+ \sigma_{g})\\
		& \leq \frac{2}{m n} \sum_{t=0}^{T-1} \exp \left(\beta \sum_{l=t+1}^{T-1} \tilde{\eta}_t\right) \tilde{\eta}_t e^{\beta \tilde{\eta}_{t}}(\tilde{c} L+ \sigma_g+\tilde{c} \sigma)\\
	\end{split}
\end{align}
$$
\varepsilon_{\text {gen }} \leq \frac{2 L }{m n} \frac{1}{\beta} e^{1+\frac{1}{T}}\left(\tilde{c}L+\sigma_g+\tilde{c}\sigma\right)
$$
When the diminishing stepsizes are chosen in the statement of the theorem,  we conclude the proof.
\end{proof}

\section{ Convergence Analysis for FedMoSWA under strongly convex setting }	
For the optimization problem, we consider:
$$
\min _{\boldsymbol{\theta} \in \mathbb{R}^d}\left\{F(\boldsymbol{\theta}):=\frac{1}{m} \sum_{i=1}^m\left(F_i(\boldsymbol{\theta}):=\mathbb{E}_{\zeta_i}\left[F_i\left(\boldsymbol{\theta} ; \zeta_i\right)\right]\right)\right\} .$$
Before giving our theoretical results, we first present the common assumptions.\\
\begin{assumption} \label{asm:strong-convexity}
    \textit{$F_i$ is $\mu$-strongly-convex for all $i \in[m]$, i.e.,
	\begin{equation}
		F_i(\boldsymbol{\theta_1}) \geq F_i(\boldsymbol{\theta_2})+\left\langle\nabla F_i(\boldsymbol{\theta_2}), \boldsymbol{\theta_1}-\boldsymbol{\theta_2}\right\rangle+\frac{\mu}{2}\|\boldsymbol{\theta_1}-\boldsymbol{\theta_2}\|^2
	\end{equation}
	for all $\boldsymbol{\theta}_2, \boldsymbol{\theta}_1$ in its domain and $i \in[m]$. We allow $\mu=0$, which corresponds to general convex functions.}
\end{assumption}

\begin{assumption}\label{asm:smoothness}
(Smoothness)\textit{	The gradient of the loss function is Lipschitz continuous with constant $\beta$, for all $\boldsymbol{\theta}_1, \boldsymbol{\theta}_2 \in \mathbb{R}^d$}
\begin{equation}
	\left\|\nabla F\left(\boldsymbol{\theta}_1\right)-\nabla F\left(\boldsymbol{\theta}_2\right)\right\| \leq \beta\left\|\boldsymbol{\theta}_1-\boldsymbol{\theta}_2\right\|
\end{equation}
\end{assumption}

\begin{assumption}\label{asm:variance} 
\textit{Let $\zeta$ be a mini-batch drawn uniformly at random from all samples. We assume that the data is distributed so that, for all $\boldsymbol{\theta} \in \mathbb{R}^d$}
\begin{equation}
	\mathbb{E}_{\zeta \mid \boldsymbol{\theta}}\left[\nabla F_i(\boldsymbol{\theta} ; \zeta)\right]=\nabla F_i(\boldsymbol{\theta}) .
\end{equation}
\textit{We also can get:}

$$
\mathbb{E}_{\zeta \mid \boldsymbol{\theta}}\left[\left\|\nabla F_i\left(\boldsymbol{\theta} ; \zeta_i\right)-\nabla F_i(\boldsymbol{\theta})\right\|^2\right] \leq \sigma^2.
$$
\end{assumption}

\begin{assumption}\label{asm:heterogeneity} 
(Bounded heterogeneity)
\textit{The dissimilarity of $F_i(\boldsymbol{\theta})$ and $f(\boldsymbol{\theta})$ is bounded as follows:
	\begin{equation}
		\frac{1}{m} \sum_{i=1}^m\left\|\nabla F_i(\boldsymbol{\theta})-\nabla F(\boldsymbol{\theta})\right\|^2 \leq \sigma_g^2.
\end{equation}}
\end{assumption}

Assumption \ref{asm:variance} bounds the variance of stochastic gradients, which is common in stochastic optimization analysis \cite{bubeck2015convex}. Assumption \ref{asm:heterogeneity} bounds the gradient difference between global and local loss functions, which is a widely-used approach to characterize client heterogeneity in federated optimization literature \cite{li2020federated,reddi2020adaptive}.

\subsection{Some technical lemmas}

Now we cover some technical lemmas which are useful for
computations later on.
The two lemmas below are useful to unroll recursions and derive convergence
rates.

\begin{lemma}[linear convergence rate]
\label{lem:constant}
For every non-negative sequence $\{d_{r-1}\}_{r \geq 1}$ and any
parameters $\mu > 0$, $\eta_{\max} \in (0, 1/\mu]$, $c \geq 0$,
$R \geq \frac{1}{2\eta_{\max} \mu}$, there exists a constant step-size
$\eta \leq \eta_{\max}$ and weights $w_r:=(1-\mu\eta)^{1-r}$ such that
for $W_R := \sum_{r=1}^{R+1} w_r$,
\begin{equation}
\Psi_R:=\frac{1}{W_R} \sum_{r=1}^{R+1}\left(\frac{w_r}{\eta}(1-\mu \eta) d_{r-1}-\frac{w_r}{\eta} d_r+c \eta w_r\right)=\tilde{\mathcal{O}}\left(\mu d_0 \exp \left(-\mu \eta_{\max } R\right)+\frac{c}{\mu R}\right)
\end{equation}

\end{lemma}
\begin{proof}
  By substituting the value of $w_r$, we observe that we end up with a
  telescoping sum and estimate
    \begin{align*}
     \Psi_R = \frac{1}{\eta W_R} \sum_{r=1}^{R+1} \left(w_{r-1}d_{r-1} - w_{r} d_{r} \right) + \frac{c\eta}{W_R}\sum_{r=1}^{R+1} w_r \leq \frac{d_0}{\eta W_R} +  c \eta \,.
    \end{align*}
    When $R \geq \frac{1}{2\mu\eta}$, $(1 - \mu \eta)^R \leq \exp(-\mu\eta R) \leq \frac{2}{3}$. For such an $R$, we can lower bound $\eta W_R$ using
    \[
\eta W_R = \eta (1 - \mu \eta)^{-R} \sum_{r=0}^{R} (1 - \mu \eta)^r = \eta (1 - \mu \eta)^{-R} \frac{1 - (1 - \mu \eta)^R}{\mu \eta} \geq (1 - \mu \eta)^{-R} \frac{1}{3\mu}\,.
    \]
    This proves that for all $R \geq \frac{1}{2\mu \eta}$,
    \[
 \Psi_R \leq 3\mu d_0 (1 - \mu \eta)^{R}  +  c \eta \leq 3 \mu d_o \exp(-\mu \eta R) + c\eta\,.
    \]
    The lemma now follows by carefully tuning $\eta$. Consider the following two cases depending on the magnitude of $R$ and $\eta_{\max}$:
    
 Suppose $\frac{1}{2 \mu R} \leq \eta_{\max } \leq \frac{\log \left(\max \left(1, \mu^2 R d_0 / c\right)\right)}{\mu R}$. Then we can choose $\eta=\eta_{\max }$,

$$
\Psi_R \leq 3 \mu d_0 \exp \left[-\mu \eta_{\max } R\right]+c \eta_{\max } \leq 3 \mu d_0 \exp \left[-\mu \eta_{\max } R\right]+\tilde{\mathcal{O}}\left(\frac{c}{\mu R}\right)
$$

 Instead if $\eta_{\max }>\frac{\log \left(\max \left(1, \mu^2 R d_0 / c\right)\right)}{\mu R}$, we pick $\eta=\frac{\log \left(\max \left(1, \mu^2 R d_0 / c\right)\right)}{\mu R}$ to claim that

$$
\Psi_R \leq 3 \mu d_0 \exp \left[-\log \left(\max \left(1, \mu^2 R d_0 / c\right)\right)\right]+\tilde{\mathcal{O}}\left(\frac{c}{\mu R}\right) \leq \tilde{\mathcal{O}}\left(\frac{c}{\mu R}\right)
$$

    \end{proof}

\begin{lemma}[sub-linear convergence rate]\label{lemma:general}
For every non-negative sequence $\{d_{r-1}\}_{r \geq 1}$ and any
parameters $\eta_{\max}  \geq 0$, $c \geq 0$,
$R \geq 0$, there exists a constant step-size
$\eta \leq \eta_{\max}$ and weights $w_r  = 1$ such that,
\begin{equation}
\Psi_R:=\frac{1}{R+1} \sum_{r=1}^{R+1}\left(\frac{d_{r-1}}{\eta}-\frac{d_r}{\eta}+c_1 \eta+c_2 \eta^2\right) \leq \frac{d_0}{\eta_{\max }(R+1)}+\frac{2 \sqrt{c_1 d_0}}{\sqrt{R+1}}+2\left(\frac{d_0}{R+1}\right)^{\frac{2}{3}} c_2^{\frac{1}{3}} .
\end{equation}
\end{lemma}
\begin{proof} Unrolling the sum, we can simplify
\[
    \Psi_R \leq \frac{d_0}{\eta (R+1)} + c_1 \eta + c_2 \eta^2\,.
\]
Similar to the strongly convex case (Lemma~\ref{lem:constant}), we distinguish the following cases:
\begin{itemize}
    \item When $R+1 \leq \frac{d_0}{c_1 \eta_{\max}^2}$, and $R+1 \leq \frac{d_0}{c_2 \eta_{\max}^3}$ we pick $\eta = \eta_{\max}$ to claim
\begin{equation}
\Psi_R \leq \frac{d_0}{\eta_{\max }(R+1)}+c_1 \eta_{\max }+c_2 \eta_{\max }^2 \leq \frac{d_0}{\eta_{\max }(R+1)}+\frac{\sqrt{c_1 d_0}}{\sqrt{R+1}}+\left(\frac{d_0}{R+1}\right)^{\frac{2}{3}} c_2^{\frac{1}{3}}
\end{equation}
    \item In the other case, we have $\eta_{\max}^2 \geq \frac{d_0}{c_1(R+1)}$ or $\eta_{\max}^3 \geq \frac{d_0}{c_2(R+1)}$. We choose $\eta = =min \left\{\sqrt{\frac{d_0}{c_1(R+1)}}, \sqrt[3]{\frac{d_0}{c_2(R+1)}}\right\}$ to prove
    \[
        \Psi_R \leq \frac{d_0}{\eta (R+1)} + c \eta = \frac{2\sqrt{c_1 d_0}}{\sqrt{R +1}} + 2\sqrt[3]{\frac{d_0^2 c_2}{(R+1)^2}} \,.\vspace{-5mm}
    \]
\end{itemize}
\end{proof}

    Next, we state a relaxed triangle inequality true for the squared
    $\ell_2$ norm.
\begin{lemma}[relaxed triangle inequality]\label{lem:norm-sum}
Lemma 3 (relaxed triangle inequality). Let $\left\{\boldsymbol{v}_1, \ldots, \boldsymbol{v}_\tau\right\}$ be $\tau$ vectors in $\mathbb{R}^d$. Then the following are true:

1. $\left\|\boldsymbol{v}_i+\boldsymbol{v}_j\right\|^2 \leq(1+a)\left\|\boldsymbol{v}_i\right\|^2+\left(1+\frac{1}{a}\right)\left\|\boldsymbol{v}_j\right\|^2$ for any $a>0$, and

2. $\left\|\sum_{i=1}^\tau \boldsymbol{v}_i\right\|^2 \leq \tau \sum_{i=1}^\tau\left\|\boldsymbol{v}_i\right\|^2$.
\end{lemma}
\begin{proof}
The proof of the first statement for any $a>0$ follows from the identity:

$$
\left\|\boldsymbol{v}_i+\boldsymbol{v}_j\right\|^2=(1+a)\left\|\boldsymbol{v}_i\right\|^2+\left(1+\frac{1}{a}\right)\left\|\boldsymbol{v}_j\right\|^2-\left\|\sqrt{a} \boldsymbol{v}_i+\frac{1}{\sqrt{a}} \boldsymbol{v}_j\right\|^2 .
$$

For the second inequality, we use the convexity of $\boldsymbol{x} \rightarrow\|\boldsymbol{x}\|^2$ and Jensen's inequality

$$
\left\|\frac{1}{\tau} \sum_{i=1}^\tau \boldsymbol{v}_i\right\|^2 \leq \frac{1}{\tau} \sum_{i=1}^\tau\left\|\boldsymbol{v}_i\right\|^2
$$

Next we state an elementary lemma about expectations of norms of random vectors.
\end{proof}

\begin{lemma}[separating mean and variance]\label{lem:independent}
 Let $\left\{\Xi_1, \ldots, \Xi_\tau\right\}$ be $\tau$ random variables in $\mathbb{R}^d$ which are not necessarily independent. First suppose that their mean is $\mathbb{E}\left[\Xi_i\right]=\xi_i$ and variance is bounded as $\mathbb{E}\left[\left\|\Xi_i-\xi_i\right\|^2\right] \leq \sigma^2$. Then, the following holds

$$
\mathbb{E}\left[\left\|\sum_{i=1}^\tau \Xi_i\right\|^2\right] \leq\left\|\sum_{i=1}^\tau \xi_i\right\|^2+\tau^2 \sigma^2
$$

Now instead suppose that their conditional mean is $\mathbb{E}\left[\Xi_i \mid \Xi_{i-1}, \ldots \Xi_1\right]=\xi_i$ i.e. the variables $\left\{\Xi_i-\xi_i\right\}$ form a martingale difference sequence, and the variance is bounded by $\mathbb{E}\left[\left\|\Xi_i-\xi_i\right\|^2\right] \leq \sigma^2$ as before. Then we can show the tighter bound

$$
\mathbb{E}\left[\left\|\sum_{i=1}^\tau \Xi_i\right\|^2\right] \leq 2\left\|\sum_{i=1}^\tau \xi_i\right\|^2+2 \tau \sigma^2 .
$$

\end{lemma}
\begin{proof}
For any random variable $X, \mathbb{E}\left[X^2\right]=(\mathbb{E}[X-\mathbb{E}[X]])^2+(\mathbb{E}[X])^2$ implying

$$
\mathbb{E}\left[\left\|\sum_{i=1}^\tau \Xi_i\right\|^2\right]=\left\|\sum_{i=1}^\tau \xi_i\right\|^2+\mathbb{E}\left[\left\|\sum_{i=1}^\tau \Xi_i-\xi_i\right\|^2\right] .
$$

Expanding the above expression using relaxed triangle inequality (Lemma 3) proves the first claim:

$$
\mathbb{E}\left[\left\|\sum_{i=1}^\tau \Xi_i-\xi_i\right\|^2\right] \leq \tau \sum_{i=1}^\tau \mathbb{E}\left[\left\|\Xi_i-\xi_i\right\|^2\right] \leq \tau^2 \sigma^2 .
$$

For the second statement, $\xi_i$ is not deterministic and depends on $\Xi_{i-1}, \ldots, \Xi_1$. Hence we have to resort to the cruder relaxed triangle inequality to claim

$$
\mathbb{E}\left[\left\|\sum_{i=1}^\tau \Xi_i\right\|^2\right] \leq 2\left\|\sum_{i=1}^\tau \xi_i\right\|^2+2 \mathbb{E}\left[\left\|\sum_{i=1}^\tau \Xi_i-\xi_i\right\|^2\right]
$$

and then use the tighter expansion of the second term:

$$
\mathbb{E}\left[\left\|\sum_{i=1}^\tau \Xi_i-\xi_i\right\|^2\right]=\sum_{i, j} \mathbb{E}\left[\left(\Xi_i-\xi_i\right)^{\top}\left(\Xi_j-\xi_j\right)\right]=\sum_i \mathbb{E}\left[\left\|\Xi_i-\xi_i\right\|^2\right] \leq \tau \sigma^2
$$

The cross terms in the above expression have zero mean since $\left\{\Xi_i-\xi_i\right\}$ form a martingale difference sequence.
\end{proof}

\subsection{Convergence of FedMoSWA for strongly convex functions}
We will first bound the variance of FedMoSWA update in Lemma~\ref{lem:server-variance-sample}, then see how sampling of clients effects our control variates in Lemma~\ref{lem:control-error-sample}, and finally bound the amount of client-drift in Lemma~\ref{lem:drift-bound-sampling}. We will then use these three lemmas to prove the progress in a single round in Lemma~\ref{lem:progress-sample}. Combining this progress with Lemmas~\ref{lem:constant} and \ref{lemma:general} gives us the desired rates.
Before proceeding with the proof of our lemmas, we need some additional definitions of the various errors we track. As before, we define the effective step-size to be
$$
\tilde{\eta}:=K \eta_k^{t} \alpha=K \eta_l \alpha .
$$
We define client-drift to be how much the clients move from their starting point:
$$
\mathcal{E}_t:=\frac{1}{K m} \sum_{k=1}^K \sum_{i=1}^m \mathbb{E}\left[\left\|\boldsymbol{\theta}_{i, k}^t-\boldsymbol{\theta}^{t-1}\right\|^2\right] .
$$
Because we are sampling the clients, not all the client control-variates get updated every round. This leads to some 'lag' which we call control-lag:
$$
\mathcal{C}_t:=\frac{1}{m} \sum_{j=1}^m \mathbb{E}\left\|\mathbb{E}\left[\boldsymbol{c}_i^t\right]-\nabla F_i\left(\boldsymbol{\theta}^{\star}\right)\right\|^2 .
$$
$$
\mathcal{D}_t:=\mathbb{E}\left\|\boldsymbol{m}^t-\nabla F\left(\theta^{\star}\right)\right\|^2=\mathbb{E}\left\|\boldsymbol{m}^t\right\|^2 .
$$

\paragraph{Variance of server update.} We study how the variance of the server update can be bounded.
\begin{lemma}\label{lem:server-variance-sample}
	For updates, we can bound the variance of the server update in any round $r$ and any $\tilde{\eta}:=K \eta_k^{t} \alpha \geq$ 0 as follows
	$$
	\mathbb{E}\left[\left\|\boldsymbol{\theta}^t-\boldsymbol{\theta}^{t-1}\right\|^2\right] \leq 8 \beta \tilde{\eta}^2\left(\mathbb{E}\left[f\left(\boldsymbol{\theta}^{t-1}\right)\right]-f\left(\boldsymbol{\theta}^{\star}\right)\right)+4 \tilde{\eta}^2 \mathcal{C}_{t-1}+4 \tilde{\eta}^2 \mathcal{D}_{t-1}+4 \tilde{\eta}^2 \beta^2 \mathcal{E}_t+\frac{12 \tilde{\eta}^2 \sigma^2}{K s}
	$$
\end{lemma}
\begin{proof}

The server update in round $t$ can be written as follows (dropping the superscript $t$ everywhere)
\begin{align}
	\begin{split}
		\mathbb{E}\|\Delta \boldsymbol{\theta}\|^2=\mathbb{E}\left\|-\frac{\tilde{\eta}}{K s} \sum_{k, i \in \mathcal{S}} \boldsymbol{\theta}_{i, k}\right\|^2=\mathbb{E}\left\|\frac{\tilde{\eta}}{K s} \sum_{k, i \in \mathcal{S}}\left(g_i\left(\boldsymbol{\theta}_{i, k-1}\right)+\boldsymbol{m}-\boldsymbol{c}_i\right)\right\|^2,
	\end{split}
\end{align}
which can then be expanded as
\begin{align}
	\begin{split}
		& \mathbb{E}\|\Delta \boldsymbol{\theta}\|^2 \leq \mathbb{E}\left\|\frac{\tilde{\eta}}{K s} \sum_{k, i \in \mathcal{S}}\left(g_i\left(\boldsymbol{\theta}_{i, k-1}\right)+\boldsymbol{m}-\boldsymbol{c}_i\right)\right\|^2 \\
		& \leq 4 \mathbb{E}\left\|\frac{\tilde{\eta}}{K s} \sum_{k, i \in \mathcal{S}} g_i\left(\boldsymbol{\theta}_{i, k-1}\right)-\nabla F_i(\boldsymbol{\theta})\right\|^2+4 \tilde{\eta}^2 \mathbb{E}\|\boldsymbol{m}\|^2+4 \mathbb{E}\left\|\frac{\tilde{\eta}}{K s} \sum_{k, i \in \mathcal{S}} \nabla F_i\left(\boldsymbol{\theta}^{\star}\right)-\boldsymbol{c}_i\right\|^2 \\
		& \quad+4 \mathbb{E}\left\|\frac{\tilde{\eta}}{K s} \sum_{k, i \in \mathcal{S}} \nabla F_i(\boldsymbol{\theta})-\nabla F_i\left(\boldsymbol{\theta}^{\star}\right)\right\|^2
	\end{split}
\end{align}
\begin{align}
	\begin{split}
		& \leq 4 \mathbb{E}\left\|\frac{\tilde{\eta}}{K s} \sum_{k, i \in \mathcal{S}} g_i\left(\boldsymbol{\theta}_{i, k-1}\right)-\nabla F_i(\boldsymbol{\theta})\right\|^2+4 \tilde{\eta}^2 \mathbb{E}\|\boldsymbol{m}\|^2+4 \mathbb{E}\left\|\frac{\tilde{\eta}}{s} \sum_{i \in \mathcal{S}} \nabla F_i\left(\boldsymbol{\theta}^{\star}\right)-\boldsymbol{c}_i\right\|^2 \\
		& +8 \beta \tilde{\eta}^2\left(\mathbb{E}[f(\boldsymbol{\theta})]-f\left(\boldsymbol{\theta}^{\star}\right)\right) \\
		& \leq 4 \mathbb{E}\left\|\frac{\tilde{\eta}}{K s} \sum_{k, i \in \mathcal{S}} \nabla F_i\left(\boldsymbol{\theta}_{i, k-1}\right)-\nabla F_i(\boldsymbol{\theta})\right\|^2+4 \tilde{\eta}^2\|\mathbb{E}[\boldsymbol{m}]\|^2+4\left\|\frac{\tilde{\eta}}{s} \sum_{i \in \mathcal{S}} \nabla F_i\left(\boldsymbol{\theta}^{\star}\right)-\mathbb{E}\left[\boldsymbol{c}_i\right]\right\|^2 \\
		& +8 \beta \tilde{\eta}^2\left(\mathbb{E}[f(\boldsymbol{\theta})]-f\left(\boldsymbol{\theta}^{\star}\right)\right)+\frac{12 \tilde{\eta}^2 \sigma^2}{K s} . \\
		&
	\end{split}
\end{align}
The inequality before the last used the smoothness of $\left\{F_i\right\}$. The variance of $\left(\frac{1}{K s} \sum_{k, i \in \mathcal{S}} g_i\left(\boldsymbol{\theta}_{i, k-1}\right)\right)$ is bounded by $\sigma^2 / K s$. Similarly, $\boldsymbol{c}_j$  for any $j \in[m]$ has variance smaller than $\sigma^2 / K$ and hence the variance of $\left(\frac{1}{m} \sum_{i \in \mathcal{S}} \boldsymbol{c}_i\right)$ is smaller than $\sigma^2 / K s$.
\begin{align}
	\begin{split}
		& \mathbb{E}\|\Delta \boldsymbol{\theta}\|^2 \leq \frac{4 \tilde{\eta}^2}{K m} \sum_{k, i} \mathbb{E}\left\|\nabla F_i\left(\boldsymbol{\theta}_{i, k-1}\right)-\nabla F_i(\boldsymbol{\theta})\right\|^2+4 \tilde{\eta}^2\|\mathbb{E}[\boldsymbol{m}]\|^2+\frac{4 \tilde{\eta}^2}{m} \sum_i\left\|\nabla F_i\left(\boldsymbol{\theta}^{\star}\right)-\mathbb{E}\left[\boldsymbol{c}_i\right]\right\|^2 \\
		& +8 \beta \tilde{\eta}^2\left(\mathbb{E}[f(\boldsymbol{\theta})]-f\left(\boldsymbol{\theta}^{\star}\right)\right)+\frac{12 \tilde{\eta}^2 \sigma^2}{K s} \\
		& \leq \underbrace{\frac{4 \tilde{\eta}^2}{K m} \sum_{k, i} \mathbb{E}\left\|\nabla F_i\left(\boldsymbol{\theta}_{i, k-1}\right)-\nabla F_i(\boldsymbol{\theta})\right\|^2}_{\mathcal{T}_1}+\frac{4 \tilde{\eta}^2}{m} \sum_i\left\|\nabla F_i\left(\boldsymbol{\theta}^{\star}\right)-\mathbb{E}\left[\boldsymbol{c}_i\right]\right\|^2 +4 \tilde{\eta}^2\|\mathbb{E}[\boldsymbol{m}]\|^2\\
		& +8 \beta \tilde{\eta}^2\left(\mathbb{E}[f(\boldsymbol{\theta})]-f\left(\boldsymbol{\theta}^{\star}\right)\right)+\frac{12 \tilde{\eta}^2 \sigma^2}{K s} . \\
		&
	\end{split}
\end{align}

Since the gradient of $F_i$ is $\beta$-Lipschitz, $\mathcal{T}_1 \leq \frac{\beta^2 4 \tilde{\eta}^2}{K m} \sum_{k, i} \mathbb{E}\left\|\boldsymbol{\theta}_{i, k-1}-\boldsymbol{\theta}\right\|^2=$ $4 \tilde{\eta}^2 \beta^2 \mathcal{E}$. The definition of the error in the control variate $\mathcal{C}_{t-1}:=\frac{1}{m} \sum_{j=1}^m \mathbb{E}\left\|\mathbb{E}\left[\boldsymbol{c}_i\right]-\nabla F_i\left(\boldsymbol{\theta}^{\star}\right)\right\|^2$ completes the proof.\\
\end{proof}

\paragraph{Change in control lag.}
We have previously related the variance of the server update to the control lag. We now examine how the control-lag grows each round.
\begin{lemma}\label{lem:control-error-sample}
	For updates  with the control update and Assumptions \ref{asm:strong-convexity}--\ref{asm:smoothness}, the following holds true for any $\tilde{\eta}:=\eta_k^{t} \alpha K \in[0,1 / \beta]:$
	$$
	\mathcal{C}_t \leq\left(1-\frac{s}{m}\right) \mathcal{C}_{t-1}+\frac{s}{m}\left(4 \beta\left(\mathbb{E}\left[f\left(\boldsymbol{\theta}^{t-1}\right)\right]-f\left(\boldsymbol{\theta}^{\star}\right)\right)+2 \beta^2 \mathcal{E}_t\right) .
	$$
\end{lemma}
\begin{proof}
    
Recall that after round $t$, the control update rule  implies that $\boldsymbol{c}_i^t$ is set as per
$$
\boldsymbol{c}_i^t= \begin{cases}\boldsymbol{c}_i^{t-1} & \text { if } i \notin \mathcal{S}^t \text { i.e. with probability }\left(1-\frac{s}{m}\right) . \\ \frac{1}{K} \sum_{k=1}^K g_i\left(\boldsymbol{\theta}_{i, k-1}^t\right) & \text { with probability } \frac{s}{m} .\end{cases}
$$
Taking expectations on both sides yields
$$
\mathbb{E}\left[\boldsymbol{c}_i^t\right]=\left(1-\frac{s}{m}\right) \mathbb{E}\left[\boldsymbol{c}_i^{t-1}\right]+\frac{s}{K m} \sum_{k=1}^K \mathbb{E}\left[\nabla F_i\left(\boldsymbol{\theta}_{i, k-1}^t\right)\right], \forall i \in[m] .
$$
Plugging the above expression in the definition of $\mathcal{C}_t$ we get,
\begin{align}
	\begin{split}
		\mathcal{C}_t & =\frac{1}{m} \sum_{i=1}^m\left\|\mathbb{E}\left[\boldsymbol{c}_i^t\right]-\nabla F_i\left(\boldsymbol{\theta}^{\star}\right)\right\|^2 \\
		& =\frac{1}{m} \sum_{i=1}^m\left\|\left(1-\frac{s}{m}\right)\left(\mathbb{E}\left[\boldsymbol{c}_i^{t-1}\right]-\nabla F_i\left(\boldsymbol{\theta}^{\star}\right)\right)+\frac{s}{m}\left(\frac{1}{K} \sum_{k=1}^K \mathbb{E}\left[\nabla F_i\left(\boldsymbol{\theta}_{i, k-1}^t\right)\right]-\nabla F_i\left(\boldsymbol{\theta}^{\star}\right)\right)\right\|^2 \\
		& \leq\left(1-\frac{s}{m}\right) \mathcal{C}_{t-1}+\frac{s}{m^2 K} \sum_{k=1}^K \mathbb{E}\left\|\nabla F_i\left(\boldsymbol{\theta}_{i, k-1}^t\right)-\nabla F_i\left(\boldsymbol{\theta}^{\star}\right)\right\|^2 .
	\end{split}
\end{align}
The final step applied Jensen's inequality twice. We can then further simplify using the relaxed triangle inequality as,
\begin{align}
	\begin{split}
		\mathbb{E}_{t-1}\left[\mathcal{C}_t\right] & \leq\left(1-\frac{s}{m}\right) \mathcal{C}_{t-1}+\frac{s}{m^2 K} \sum_{i, k} \mathbb{E}\left\|\nabla F_i\left(\boldsymbol{\theta}_{i, k-1}^t\right)-\nabla F_i\left(\boldsymbol{\theta}^{\star}\right)\right\|^2 \\
		& \leq\left(1-\frac{s}{m}\right) \mathcal{C}_{t-1}+\frac{2 s}{m^2} \sum_i \mathbb{E}\left\|\nabla F_i\left(\boldsymbol{\theta}^{t-1}\right)-\nabla F_i\left(\boldsymbol{\theta}^{\star}\right)\right\|^2+\frac{2 s}{m^2 K} \sum_{i, k} \mathbb{E}\left\|\nabla F_i\left(\boldsymbol{\theta}_{i, k-1}^t\right)-\nabla F_i\left(\boldsymbol{\theta}^{t-1}\right)\right\|^2 \\
		& \leq\left(1-\frac{s}{m}\right) \mathcal{C}_{t-1}+\frac{2 s}{m^2} \sum_i \mathbb{E}\left\|\nabla F_i\left(\boldsymbol{\theta}^{t-1}\right)-\nabla F_i\left(\boldsymbol{\theta}^{\star}\right)\right\|^2+\frac{2 s}{m^2 K} \beta^2 \sum_{i, k} \mathbb{E}\left\|\boldsymbol{\theta}_{i, k-1}^t-\boldsymbol{\theta}^{t-1}\right\|^2 \\
		&\leq\left(1-\frac{s}{m}\right) \mathcal{C}_{t-1}+\frac{s}{m}\left(4 \beta\left(\mathbb{E}\left[f\left(\boldsymbol{\theta}^{t-1}\right)\right]-f\left(\boldsymbol{\theta}^{\star}\right)\right)+\beta^2 \mathcal{E}_t\right) .
	\end{split}
\end{align}
The last two inequalities follow from smoothness of $\left\{F_i\right\}$ and the definition $\mathcal{E}_t=\frac{1}{m K} \beta^2 \sum_{i, k} \mathbb{E}\left\|\boldsymbol{\theta}_{i, k-1}^t-\boldsymbol{\theta}^{t-1}\right\|^2$.
\end{proof}
\begin{lemma}
	Assumptions  \ref{asm:strong-convexity}--\ref{asm:smoothness}, the following holds true for any $\tilde{\eta}:=\eta_k^{t} \alpha K \in[0,1 / \beta]:$
	$$
	\mathcal{D}_t \leq\left(1-\gamma\right) \mathcal{D}_{t-1}+\gamma   \left(4 \beta\left(\mathbb{E}\left[f\left(\boldsymbol{\theta}^{t-1}\right)\right]-f\left(\boldsymbol{\theta}^{\star}\right)\right)+2 \beta^2 \mathcal{E}_t\right) .
	$$
\end{lemma}
\begin{proof}

Recall that after round $t$, the control update rule  implies that $\boldsymbol{c}_i^t$ is set as 
\begin{align}
	\begin{split}
		& \boldsymbol{m}^t=\boldsymbol{m}^{t-1}+\beta \frac{1}{m} \sum_{i=1}^m\left(c_i^t-\boldsymbol{m}^{t-1}\right) =(1-\beta) \boldsymbol{m}^{t-1}+\beta \frac{1}{m} \sum_{i=1}^m c_i^t
	\end{split}
\end{align}
\begin{align}
	\begin{split}
		& \mathbb{E}\left\|\boldsymbol{m}^t-\nabla F\left(\boldsymbol{\theta}^{\star}\right)\right\|^2=\mathbb{E}\left\|\left(1-\gamma\right) \boldsymbol{m}^{t-1}+\gamma \frac{1}{m} \sum_{i=1}^m \frac{1}{K} \sum_{k=1}^K \nabla F_i\left(\boldsymbol{\theta}_{i,k-1}^{t} ; \zeta\right)-\nabla F\left(x^{\star}\right)\right\|^2 \\
		& \leq\left(1-\gamma\right) \mathbb{E}\left\|\boldsymbol{m}^{t-1}-\nabla F\left(\boldsymbol{\theta}^{\star}\right)\right\|^2+\gamma \mathbb{E}\left\|\frac{1}{m} \sum_{i=1}^m \frac{1}{K} \sum_{k=1}^K \nabla F_i\left(\boldsymbol{\theta}_{i,k-1}^{t} ; \zeta\right)-\nabla F\left(\boldsymbol{\theta}^{\star}\right)\right\|^2 \\
		& \leq\left(1-\gamma\right) \mathbb{E}\left\|\boldsymbol{m}^{t-1}-\nabla F\left(\boldsymbol{\theta}^{\star}\right)\right\|^2+\frac{2}{m} \gamma \mathbb{E} \sum_{i=1}^m\left\|\nabla F\left(\boldsymbol{\theta}^{t-1}\right)-\nabla F\left(\boldsymbol{\theta}^{\star}\right)\right\|^2 \\
		& +\frac{2}{m K} \gamma \sum_{i, k} \mathbb{E}\left\|\nabla F_i\left(\boldsymbol{\theta}_i^{t-1, t}\right)-\nabla F_i\left(\boldsymbol{\theta}^{t-1}\right)\right\|^2 \\
		& \leq\left(1-\gamma\right) \mathbb{E}\left\|\boldsymbol{m}^{t-1}-\nabla F\left(\boldsymbol{\theta}^{\star}\right)\right\|^2+\gamma\left(4 \beta\left(\mathbb{E}\left[f\left(\boldsymbol{\theta}^{t-1}\right)\right]-f\left(\boldsymbol{\theta}^{\star}\right)\right)+\beta^2 \mathcal{E}_t\right)
	\end{split}
\end{align}
\end{proof}

\paragraph{Bounding client-drift.} We will now bound the final source of error which is the client-drift.
\begin{lemma}\label{lem:drift-bound-sampling}
	Suppose our step-sizes satisfy $\eta_k^{t}=\eta_l \leq \frac{1}{81 \beta K \alpha}$ and $F_i$ satisfies Assumptions \ref{asm:strong-convexity}--\ref{asm:smoothness}. Then, for any global $\alpha \geq 1$ we can bound the drift as
	$$
	3 \beta \tilde{\eta} \mathcal{E}_t \leq \frac{ \tilde{\eta}^2}{3} \mathcal{C}_{t-1}+\frac{ \tilde{\eta}^2}{3} \mathcal{D}_{t-1}+\frac{\tilde{\eta}}{25 \alpha^2}\left(\mathbb{E}\left[f\left(\boldsymbol{\theta}^{t-1}\right)\right]-f\left(\boldsymbol{\theta}^{\star}\right)\right)+\frac{\tilde{\eta}^2}{K \alpha^2} \sigma^2 .
	$$
\end{lemma}
\begin{proof}

First, observe that if $K=1, \mathcal{E}_t=0$ since $\boldsymbol{\theta}_{i, 0}=\boldsymbol{\theta}$ for all $i \in[m]$ and that $\mathcal{C}_{t-1}$ and the right hand side are both positive.  For $K > 1$, we build a recursive bound of the drift. Starting from the definition of the update  and then applying the relaxed triangle inequality, we can expand
\begin{align}
	\begin{split}
		& \frac{1}{m} \mathbb{E}_{t-1}\left[\sum_{i \in \mathcal{S}}\left\|\boldsymbol{\theta}_i-\eta_l g_i\left(\boldsymbol{\theta}_i\right)+\eta_l \boldsymbol{m}-\eta_l \boldsymbol{c}_i-\boldsymbol{\theta}\right\|^2\right] \\
		& \leq \frac{1}{m} \mathbb{E}_{t-1}\left[\sum_{i \in \mathcal{S}}\left\|\boldsymbol{\theta}_i-\eta_l \nabla F_i\left(\boldsymbol{\theta}_i\right)+\eta_l \boldsymbol{m}-\eta_l \boldsymbol{c}_i-\boldsymbol{\theta}\right\|^2\right]+\eta_l^2 \sigma^2 \\
		& \leq \frac{(1+a)}{s} \mathbb{E}_{t-1}[\sum_{i \in \mathcal{S}} \underbrace{\left\|\boldsymbol{\theta}_i-\eta_l \nabla F_i\left(\boldsymbol{\theta}_i\right)+\eta_l \nabla F_i(\boldsymbol{\theta})-\boldsymbol{\theta}\right\|^2}_{\mathcal{T}_2}] \\
		& +\left(1+\frac{1}{a}\right) \eta_l^2 \underbrace{\mathbb{E}_{t-1}\left[\frac{1}{m} \sum_{i \in \mathcal{S}}\left\|\boldsymbol{m}-\boldsymbol{c}_i+\nabla F_i(\boldsymbol{\theta})\right\|^2\right]}_{\mathcal{T}_3}+\eta_l^2 \sigma^2 . \\
		&
	\end{split}
\end{align}
    The final step follows from the relaxed triangle inequality (Lemma~\ref{lem:norm-sum}). Applying the contractive mapping Lemma for $\eta_l \leq 1/\beta$ shows
\begin{align}
	\begin{split}
		\mathcal{T}_2=\frac{1}{m} \sum_{i \in \mathcal{S}}\left\|\boldsymbol{\theta}_i-\eta_l \nabla F_i\left(\boldsymbol{\theta}_i\right)+\eta_l \nabla F_i(\boldsymbol{\theta})-\boldsymbol{\theta}\right\|^2 \leq\left\|\boldsymbol{\theta}_i-\boldsymbol{\theta}\right\|^2 .
	\end{split}
\end{align}
Once again using our relaxed triangle inequality to expand the other term $\mathcal{T}_3$, we get
\begin{align}
	\begin{split}
		\mathcal{T}_3 & =\mathbb{E}_{t-1}\left[\frac{1}{m} \sum_{i \in \mathcal{S}}\left\|\boldsymbol{m}-\boldsymbol{c}_i+\nabla F_i(\boldsymbol{\theta})\right\|^2\right] \\
		& =\frac{1}{m} \sum_{j=1}^m\left\|\boldsymbol{m}-\boldsymbol{c}_i+\nabla F_i(\boldsymbol{\theta})\right\|^2 \\
		& =\frac{1}{m} \sum_{j=1}^m\left\|\boldsymbol{m}-\boldsymbol{c}_i+\nabla F_i\left(\boldsymbol{\theta}^{\star}\right)+\nabla F_i(\boldsymbol{\theta})-\nabla F_i\left(\boldsymbol{\theta}^{\star}\right)\right\|^2 \\
		& \leq 3\|\boldsymbol{m}\|^2+\frac{3}{m} \sum_{j=1}^m\left\|\boldsymbol{c}_i-\nabla F_i\left(\boldsymbol{\theta}^{\star}\right)\right\|^2+\frac{3}{m} \sum_{j=1}^m\left\|\nabla F_i(\boldsymbol{\theta})-\nabla F_i\left(\boldsymbol{\theta}^{\star}\right)\right\|^2 \\
		& \leq 3\|\boldsymbol{m}\|^2+\frac{3}{m} \sum_{j=1}^m\left\|\boldsymbol{c}_i-\nabla F_i\left(\boldsymbol{\theta}^{\star}\right)\right\|^2+\frac{3}{m} \sum_{j=1}^m\left\|\nabla F_i(\boldsymbol{\theta})-\nabla F_i\left(\boldsymbol{\theta}^{\star}\right)\right\|^2 \\
		& \leq 3\|\boldsymbol{m}\|^2+\frac{3}{m} \sum_{j=1}^m\left\|\boldsymbol{c}_i-\nabla F_i\left(\boldsymbol{\theta}^{\star}\right)\right\|^2+6 \beta\left(f(\boldsymbol{\theta})-f\left(\boldsymbol{\theta}^{\star}\right)\right) .
	\end{split}
\end{align}
The last step used the smoothness of $F_i$. Combining the bounds on $\mathcal{T}_2$ and $\mathcal{T}_3$ in the original inequality and using $a=\frac{1}{K-1}$ gives,
\begin{align}
	\begin{split}
		& \frac{1}{m} \sum_i \mathbb{E}\left\|\boldsymbol{\theta}_{i, k}-\boldsymbol{\theta}\right\|^2 \leq \frac{\left(1+\frac{1}{K-1}\right)}{m} \sum_i \mathbb{E}\left\|\boldsymbol{\theta}_{i, k-1}-\boldsymbol{\theta}\right\|^2+\eta_l^2 \sigma^2 \\
		&+6 \eta_l^2 K \beta\left(f(\boldsymbol{\theta})-f\left(\boldsymbol{\theta}^{\star}\right)\right)+\frac{3 K \eta_l^2}{m} \sum_i \mathbb{E}\left\|\boldsymbol{c}_i-\nabla F_i\left(\boldsymbol{\theta}^{\star}\right)\right\|^2+3 K \eta_l^2\|\boldsymbol{m}\|^2
	\end{split}
\end{align}
Recall that with the choice of $c_i$, the variance of $c_i$ is less than $\frac{\sigma^2}{K}$. Separating its mean and variance gives
\begin{align}
	\begin{split}
		\frac{1}{m} \sum_i \mathbb{E}\left\|\boldsymbol{\theta}_{i, k}-\boldsymbol{\theta}\right\|^2 \leq\left(1+\frac{1}{K-1}\right) \frac{1}{m} \sum_i \mathbb{E}\left\|\boldsymbol{\theta}_{i, k-1}-\boldsymbol{\theta}\right\|^2+7 \eta_l^2 \sigma^2+ \\
		6 \eta_l^2 K \beta\left(f(\boldsymbol{\theta})-f\left(\boldsymbol{\theta}^{\star}\right)\right)+\frac{3 K \eta_l^2}{m} \sum_i \mathbb{E}\left\|\boldsymbol{c}_i-\nabla F_i\left(\boldsymbol{\theta}^{\star}\right)\right\|^2+3 K \eta_l^2\|\boldsymbol{m}\|^2
	\end{split}
\end{align}
Unrolling the recursion, we get the following for any $k \in\{1, \ldots, K\}$
\begin{align}
	\begin{split}
		\frac{1}{m} \sum_i \mathbb{E}\left\|\boldsymbol{\theta}_{i, k}-\boldsymbol{\theta}\right\|^2 & \leq\left(6 K \beta \eta_l^2\left(f(\boldsymbol{\theta})-f\left(\boldsymbol{\theta}^{\star}\right)\right)+3 K \eta_l^2 \mathcal{C}_{t-1}+3 K \eta_l^2 \mathcal{D}_{t-1}+7 \beta \eta_l^2 \sigma^2\right)\left(\sum_{k=0}^{k-1}\left(1+\frac{1}{K-1}\right)^k\right) \\
		& \leq\left(6 K \beta \eta_l^2\left(f(\boldsymbol{\theta})-f\left(\boldsymbol{\theta}^{\star}\right)\right)+3 K \eta_l^2 \mathcal{C}_{t-1}+3 K \eta_l^2 \mathcal{D}_{t-1}+7 \beta \eta_l^2 \sigma^2\right)(K-1)\left(\left(1+\frac{1}{K-1}\right)^K-1\right) \\
		& \leq\left(6 K \beta \eta_l^2\left(f(\boldsymbol{\theta})-f\left(\boldsymbol{\theta}^{\star}\right)\right)+3 K \eta_l^2 \mathcal{C}_{t-1}+3 K \eta_l^2 \mathcal{D}_{t-1}+7 \beta \eta_l^2 \sigma^2\right) 3 K \\
		& \leq 18 K^2 \beta \eta_l^2\left(f(\boldsymbol{\theta})-f\left(\boldsymbol{\theta}^{\star}\right)\right)+9 K^2 \eta_l^2 \mathcal{C}_{t-1}+9 K^2 \eta_l^2 \mathcal{D}_{t-1}+21 K \beta \eta_l^2 \sigma^2 .
	\end{split}
\end{align}
The inequality $(K-1)\left(\left(1+\frac{1}{K-1}\right)^K-1\right) \leq 3 K$ can be verified for $K=2,3$ manually. For $K \geq 4$,
\begin{align}
	\begin{split}
		(K-1)\left(\left(1+\frac{1}{K-1}\right)^K-1\right)<K\left(\exp \left(\frac{K}{K-1}\right)-1\right) \leq K\left(\exp \left(\frac{4}{3}\right)-1\right)<3 K .
	\end{split}
\end{align}
Again averaging over $k$ and multiplying by $3 \beta$ yields
\begin{align}
	\begin{split}
		3 \beta \mathcal{E}_t & \leq 54 K^2 \beta^2 \eta_l^2\left(f(\boldsymbol{\theta})-f\left(\boldsymbol{\theta}^{\star}\right)\right)+27 K^2 \beta \eta_l^2 \mathcal{C}_{t-1}+27 K^2 \beta \eta_l^2 \mathcal{D}_{t-1}+63 \beta K \eta_l^2 \sigma^2 \\
		& =\frac{1}{\alpha^2}\left(54 \beta^2 \tilde{\eta}^2\left(f(\boldsymbol{\theta})-f\left(\boldsymbol{\theta}^{\star}\right)\right)+27 \beta \tilde{\eta}^2 \mathcal{C}_{t-1}+27 \beta \tilde{\eta}^2 \mathcal{C}_{t-1}+63 \beta \tilde{\eta}^2 \frac{\sigma^2}{K}\right) \\
		& \leq \frac{1}{\alpha^2}\left(\frac{1}{25}\left(f(\boldsymbol{\theta})-f\left(\boldsymbol{\theta}^{\star}\right)\right)+\frac{1}{3} \tilde{\eta} \mathcal{C}_{t-1}+\frac{1}{3} \tilde{\eta} \mathcal{D}_{t-1}+\tilde{\eta} \frac{\sigma^2}{K}\right) .
	\end{split}
\end{align}
The equality follows from the definition $\tilde{\eta}=K \eta_l \alpha$, and the final inequality uses the bound that $\tilde{\eta} \leq \frac{1}{81 \beta}$.
\end{proof}
\begin{lemma}\label{lem:progress-sample}
	Suppose assumptions \ref{asm:strong-convexity}--\ref{asm:smoothness} are true. Then the following holds for any step-sizes satisfying $\alpha \geq 1, \eta_l \leq$ $\min \left(\frac{1}{81 \beta K \alpha}, \frac{s}{15 \mu m K \alpha}\right)$, and effective step-size $\tilde{\eta}:=K \alpha \eta_l$
	\begin{align}
		\mathbb{E}\left[\left\|\boldsymbol{\theta}^t-\boldsymbol{\theta}^{\star}\right\|^2+\frac{9 m\tilde{\eta}^2}{s} \mathcal{C}_t\right] &\leq\left(1-\frac{\mu \tilde{\eta}}{2}\right)\left(\mathbb{E}\left\|\boldsymbol{\theta}^{t-1}-\boldsymbol{\theta}^{\star}\right\|^2+\frac{9 m\tilde{\eta}^2}{s} \mathcal{C}_{t-1}\right)\\
		&-\tilde{\eta}\left(\mathbb{E}\left[f\left(\boldsymbol{\theta}^{t-1}\right)\right]-f\left(\boldsymbol{\theta}^{\star}\right)\right)
		+\frac{12 \tilde{\eta}^2}{K s}\left(1+\frac{s}{\alpha^2}\right) \sigma^2
	\end{align}
\end{lemma}
\begin{proof}
    
We can then apply Lemma~\ref{lem:server-variance-sample} to bound the second moment of the server update as
$$
\Delta \boldsymbol{\theta}=-\frac{\tilde{\eta}}{K s} \sum_{k, i \in \mathcal{S}}\left(g_i\left(\boldsymbol{\theta}_{i, k-1}\right)+\boldsymbol{m}-\boldsymbol{c}_i\right), \text { and } \mathbb{E}[\Delta \boldsymbol{\theta}]=-\frac{\tilde{\eta}}{K m} \sum_{k, i} g_i\left(\boldsymbol{\theta}_{i, k-1}\right) .
$$
We can  bound the second moment of the server update as
\begin{align}
	\begin{split}
		& \mathbb{E}_{t-1}\left\|\boldsymbol{\theta}+\Delta \boldsymbol{\theta}-\boldsymbol{\theta}^{\star}\right\|^2=\mathbb{E}_{t-1}\left\|\boldsymbol{\theta}-\boldsymbol{\theta}^{\star}\right\|^2-\frac{2 \tilde{\eta}}{K s} \mathbb{E}_{t-1} \sum_{k, i \in \mathcal{S}}\left\langle\nabla F_i\left(\boldsymbol{\theta}_{i, k-1}\right), \boldsymbol{\theta}-\boldsymbol{\theta}^{\star}\right\rangle+\mathbb{E}_{t-1}\|\Delta \boldsymbol{\theta}\|^2 \\
		& \leq\underbrace{\frac{2 \tilde{\eta}}{K s} \mathbb{E}_{t-1} \sum_{k, i \in \mathcal{S}}\left\langle\nabla F_i\left(\boldsymbol{\theta}_{i, k-1}\right), \boldsymbol{\theta}^{\star}-\boldsymbol{\theta}\right\rangle}_{\mathcal{T}_4}+\mathbb{E}_{t-1}\left\|\boldsymbol{\theta}-\boldsymbol{\theta}^{\star}\right\|^2 \\
		&+8 \beta \tilde{\eta}^2\left(\mathbb{E}\left[f\left(\boldsymbol{\theta}^{t-1}\right)\right]-f\left(\boldsymbol{\theta}^{\star}\right)\right)+4 \tilde{\eta}^2 \mathcal{C}_{t-1}+4 \tilde{\eta}^2 \mathcal{D}_{t-1}+4 \tilde{\eta}^2 \beta^2 \mathcal{E}+\frac{12 \tilde{\eta}^2 \sigma^2}{K s} .
	\end{split}
\end{align}
The term $\mathcal{T}_4$ can be bounded by using perturbed strong-convexity with $h=F_i, \boldsymbol{\theta}=\boldsymbol{\theta}_{i, k-1}, \boldsymbol{\theta}=\boldsymbol{\theta}^{\star}$, and $\boldsymbol{z}=\boldsymbol{\theta}$ to get
\begin{align}
	\begin{split}
		\mathbb{E}\left[\mathcal{T}_4\right] & =\frac{2 \tilde{\eta}}{K s} \mathbb{E} \sum_{k, i \in \mathcal{S}}\left\langle\nabla F_i\left(\boldsymbol{\theta}_{i, k-1}\right), \boldsymbol{\theta}^{\star}-\boldsymbol{\theta}\right\rangle \\
		& \leq \frac{2 \tilde{\eta}}{K s} \mathbb{E} \sum_{k, i \in \mathcal{S}}\left(F_i\left(\boldsymbol{\theta}^{\star}\right)-F_i(\boldsymbol{\theta})+\beta\left\|\boldsymbol{\theta}_{i, k-1}-\boldsymbol{\theta}\right\|^2-\frac{\mu}{4}\left\|\boldsymbol{\theta}-\boldsymbol{\theta}^{\star}\right\|^2\right) \\
		& =-2 \tilde{\eta} \mathbb{E}\left(f(\boldsymbol{\theta})-f\left(\boldsymbol{\theta}^{\star}\right)+\frac{\mu}{4}\left\|\boldsymbol{\theta}-\boldsymbol{\theta}^{\star}\right\|^2\right)+2 \beta \tilde{\eta} \mathcal{E} .
	\end{split}
\end{align}
Plugging $\mathcal{T}_4$ back, we can further simplify the expression to get
\begin{align}
	\begin{split}
		\mathbb{E}\left\|\boldsymbol{\theta}+\Delta \boldsymbol{\theta}-\boldsymbol{\theta}^{\star}\right\|^2 \leq \mathbb{E}\left\|\boldsymbol{\theta}-\boldsymbol{\theta}^{\star}\right\|^2-2 \tilde{\eta}\left(f(\boldsymbol{\theta})-f\left(\boldsymbol{\theta}^{\star}\right)+\frac{\mu}{4}\left\|\boldsymbol{\theta}-\boldsymbol{\theta}^{\star}\right\|^2\right)+2 \beta \tilde{\eta} \mathcal{E} \\
		+\frac{12 \tilde{\eta}^2 \sigma^2}{K s}+8 \beta \tilde{\eta}^2\left(\mathbb{E}\left[f\left(\boldsymbol{\theta}^{t-1}\right)\right]-f\left(\boldsymbol{\theta}^{\star}\right)\right)+4\tilde{\eta}^2 \mathcal{C}_{t-1}+4\tilde{\eta}^2 \mathcal{D}_{t-1}+4 \tilde{\eta}^2 \beta^2 \mathcal{E} \\
		=\left(1-\frac{\mu \tilde{\eta}}{2}\right)\left\|\boldsymbol{\theta}-\boldsymbol{\theta}^{\star}\right\|^2+\left(8 \beta \tilde{\eta}^2-2 \tilde{\eta}\right)\left(f(\boldsymbol{\theta})-f\left(\boldsymbol{\theta}^{\star}\right)\right) \\
		+\frac{12 \tilde{\eta}^2 \sigma^2}{K s}+\left(2 \beta \tilde{\eta}+4 \beta^2 \tilde{\eta}^2\right) \mathcal{E}+4 \tilde{\eta}^2 \mathcal{C}_{t-1}+4 \tilde{\eta}^2 \mathcal{D}_{t-1}
	\end{split}
\end{align}
We  use Lemma \ref{lem:control-error-sample}  can  (scaled by $9 \tilde{\eta}^2 \frac{m}{s}$ )  bound the
control-lag
\begin{align}
	\begin{split}
		9 \tilde{\eta}^2 \frac{1}{\gamma}\mathcal{D}_t \leq\left(1-\frac{\mu \tilde{\eta}}{2}\right) 9 \tilde{\eta}^2 \frac{1}{\gamma}\mathcal{D}_{t-1}+9\left(\frac{\mu \tilde{\eta} }{2}\frac{1}{\gamma}-1\right) \tilde{\eta}^2 \mathcal{D}_{t-1}+9 \tilde{\eta}^2\left(4 \beta\left(\mathbb{E}\left[f\left(\boldsymbol{\theta}^{t-1}\right)\right]-f\left(\boldsymbol{\theta}^{\star}\right)\right)+2 \beta^2 \mathcal{E}\right)
	\end{split}
\end{align}
    Now recall that Lemma~\ref{lem:drift-bound-sampling} bounds the client-drift:
\begin{align}
	\begin{split}
		3 \beta \tilde{\eta} \mathcal{E}_t \leq \frac{ \tilde{\eta}^2}{3} \mathcal{C}_{t-1}+\frac{ \tilde{\eta}^2}{3} \mathcal{D}_{t-1}+\frac{\tilde{\eta}}{25 \alpha^2}\left(\mathbb{E}\left[f\left(\boldsymbol{\theta}^{t-1}\right)\right]-f\left(\boldsymbol{\theta}^{\star}\right)\right)+\frac{\tilde{\eta}^2}{K \alpha^2} \sigma^2 .
	\end{split}
\end{align}
Adding all three inequalities together,
\begin{align}
	\begin{split}
		&\mathbb{E}\left\|\boldsymbol{\theta}+\Delta \boldsymbol{\theta}-\boldsymbol{\theta}^{\star}\right\|^2+\frac{9 \tilde{\eta}^2 m \mathcal{C}_t}{S}+9 \tilde{\eta}^2 \frac{1}{\gamma} \mathcal{D}_t \\
		&\leq\left(1-\frac{\mu \tilde{\eta}}{2}\right)\left(\mathbb{E}\left\|\boldsymbol{\theta}-\boldsymbol{\theta}^{\star}\right\|^2+\frac{9 \tilde{\eta}^2 m \mathcal{C}_{t-1}}{s}\right)+\left(44 \beta \tilde{\eta}^2-\frac{49}{25} \tilde{\eta}\right)\left(f(\boldsymbol{\theta})-f\left(\boldsymbol{\theta}^{\star}\right)\right) \\
		&+\frac{12 \tilde{\eta}^2 \sigma^2}{K s}\left(1+\frac{s}{\alpha^2}\right)+\left(22 \beta^2 \tilde{\eta}^2-\beta \tilde{\eta}\right) \mathcal{E}+\left(\frac{9 \mu \tilde{\eta} m}{2 S}-\frac{1}{3}\right) \tilde{\eta}^2 \mathcal{C}_{t-1}+\left(\frac{9 \mu \tilde{\eta}}{2 \gamma}-\frac{1}{3}\right) \tilde{\eta}^2 \mathcal{D}_{t-1}
	\end{split}
\end{align}
Finally, the lemma follows from noting that $\tilde{\eta} \leq \frac{1}{81 \beta}$ implies $44 \beta^2 \tilde{\eta}^2 \leq \frac{24}{25} \tilde{\beta}$ and $\tilde{\eta} \leq \frac{s}{15 \mu m}$ implies $\frac{9 \mu \tilde{\eta} m}{2 s} \leq \frac{1}{3}$. Note that if $c_i^0=g_i\left(\boldsymbol{\theta}^0\right)$, then $\frac{\tilde{\eta} m}{s} \mathcal{C}_0$ can be bounded in terms of function sub-optimality $F$. \textbf{The final rate for strongly convex} follows simply by unrolling the recursive bound in Lemma~\ref{lem:progress-sample} using Lemma~\ref{lem:constant}. Also note that if $c^0_i = g_i(x^0)$, then $\frac{\tilde{\eta} m}{s} \mathcal{C}_0$ can be bounded in terms of function sub-optimality $F$.
\end{proof}

\subsection{ Convergence Analysis for FedMoSWA under non-convex setting }	

We now analyze the most general case of FedMoSWA with option II on functions that are potentially non-convex. Just as in the non-convex proof, we will first bound the variance of the server update in Lemma~\ref{lem:non-convex-variance-sample}, the change in control lag in Lemma~\ref{lem:non-convex-control-error-sample}, and finally we bound the client-drift in Lemma~\ref{lem:nonconvex-drift-bound-sampling}. Combining these three together gives us the progress made in one round in Lemma~\ref{lem:non-convex-progress-sample}. The final rate is derived from the progress made using Lemma~\ref{lemma:general}.

Additional notation. Recall that in round $r$, we update the control variate as 
$$
\boldsymbol{c}_i^t= \begin{cases}\frac{1}{K} \sum_{k=1}^K g_i\left(\boldsymbol{\theta}_{i, k-1}^t\right) & \text { if } i \in \mathcal{S}^t, \\ \boldsymbol{c}_i^{t-1} & \text { otherwise } .\end{cases}
$$
We introduce the following notation to keep track of the 'lag' in the update of the control variate: define a sequence of parameters $\left\{\boldsymbol{\alpha}_{i, k-1}^{t-1}\right\}$ such that for any $i \in[m]$ and $k \in[K]$ we have $\boldsymbol{\alpha}_{i, k-1}^0:=\boldsymbol{\theta}^0$ and for $t \geq 1$,
$$
\boldsymbol{\alpha}_{i, k-1}^t:= \begin{cases}\boldsymbol{\theta}_{i, k-1}^t & \text { if } i \in \mathcal{S}^t, \\ \boldsymbol{\alpha}_{i, k-1}^{t-1} & \text { otherwise . }\end{cases}
$$
By the update rule for control variates  and the definition of $\left\{\boldsymbol{\alpha}_{i, k-1}^{t-1}\right\}$ above, the following property always holds:
$$
\boldsymbol{c}_i^t=\frac{1}{K} \sum_{k=1}^K g_i\left(\boldsymbol{\alpha}_{i, k-1}^t\right) .
$$
We can then define the following $\Xi_t$ to be the error in control variate for round $r$ :
$$
\Xi_t:=\frac{1}{K m} \sum_{k=1}^K \sum_{i=1}^m \mathbb{E}\left\|\boldsymbol{\alpha}_{i, k-1}^t-\boldsymbol{\theta}^t\right\|^2 .
$$
Also recall the closely related definition of client drift caused by local updates:
$$
\mathcal{E}_t:=\frac{1}{K m} \sum_{k=1}^K \sum_{i=1}^m \mathbb{E}\left[\left\|\boldsymbol{\theta}_{i, k}^t-\boldsymbol{\theta}^{t-1}\right\|^2\right] .
$$
Variance of server update. Let us analyze how the control variates effect the variance of the aggregate server update.\\ 
Define a sequence of parameters $\left\{\boldsymbol{\alpha}_{i, k-1}^{t-1}\right\}$ such that for any $i \in[m]$ and $k \in[K]$ we have $\boldsymbol{\alpha}_{i, k-1}^0:=\boldsymbol{\theta}^0$ and for $r \geq 1$,
$$
\boldsymbol{\theta}_{i, k-1}^t:= \begin{cases}\mathbf{y}_{i, k-1}^t & p=\beta \\ \boldsymbol{\theta}_{i, k-1}^{t-1} & p=1-\beta\end{cases}
$$
$$
\begin{aligned}
	& \boldsymbol{m}^t=\boldsymbol{m}^{t-1}+\beta \frac{1}{m} \sum_{i=1}^m\left(c_i^t-\boldsymbol{m}^{t-1}\right)=(1-\beta) \boldsymbol{m}^{t-1}+\beta \frac{1}{m} \sum_{i=1}^m c_i^t \\
	& \mathbb{E}[\boldsymbol{m}]=(1-\beta) \mathbb{E}\left[\boldsymbol{m}^{t-1}\right]+\beta \frac{1}{m} \mathbb{E}\left[\sum_{i=1}^m c_i^t\right]=\frac{1}{m K} \sum_{i=1}^m \sum_k g_i\left(\boldsymbol{\theta}_{i, k-1}^t\right)
\end{aligned}
$$
We can then define the following $\Xi_t$ to be the error in control variate for round $r$ :
$$
\mathcal{G}_{t-1}:=\frac{1}{K m} \sum_{k=1}^K \sum_{i=1}^m \mathbb{E}\left\|\boldsymbol{\theta}_{i, k-1}^t-\boldsymbol{\theta}^t\right\|^2 .
$$

 \paragraph{Variance of server update.} Let us analyze how the control variates effect the variance of the aggregate server update.
\begin{lemma}\label{lem:non-convex-variance-sample}
	For updates and assumptions \ref{asm:variance} and \ref{asm:smoothness}, the following holds true for any $\tilde{\eta}:=\eta_l \alpha K \in[0,1 / \beta]$ :
	$$
	\begin{gathered}
		\mathbb{E}\left\|\mathbb{E}_{t-1}\left[\boldsymbol{\theta}^t\right]-\boldsymbol{\theta}^{t-1}\right\|^2 \leq 2 \tilde{\eta}^2 \beta^2 \mathcal{E}_t+2 \tilde{\eta}^2 \mathbb{E}\left\|\nabla F\left(\boldsymbol{\theta}^{t-1}\right)\right\|^2, \text { and } \\
		\mathbb{E}\left\|\boldsymbol{\theta}^t-\boldsymbol{\theta}^{t-1}\right\|^2 \leq 4 \tilde{\eta}^2 \beta^2 \mathcal{E}_t+4 \tilde{\eta}^2 \beta^2 \Xi_{t-1}+4 \tilde{\eta}^2 \beta^2 \mathcal{G}_{t-1}+4 \tilde{\eta}^2 \mathbb{E}\left\|\nabla F\left(\boldsymbol{\theta}^{t-1}\right)\right\|^2+\frac{9 \tilde{\eta}^2 \sigma^2}{K s} .
	\end{gathered}
	$$
\end{lemma}
\begin{proof}
    
 Recall that that the server update satisfies
$$
\mathbb{E}[\Delta \boldsymbol{\theta}]=-\frac{\tilde{\eta}}{K m} \sum_{k, i} \mathbb{E}\left[g_i\left(\boldsymbol{\theta}_{i, k-1}\right)\right] .
$$
From the definition of $\boldsymbol{\alpha}_{i, k-1}^{t-1}$ and dropping the superscript everywhere we have
$$
\Delta \boldsymbol{\theta}=-\frac{\tilde{\eta}}{K s} \sum_{k, i \in \mathcal{S}}\left(g_i\left(\boldsymbol{\theta}_{i, k-1}\right)+\boldsymbol{m}-\boldsymbol{c}_i\right) \text { where } \boldsymbol{c}_i=\frac{1}{K} \sum_k g_i\left(\boldsymbol{\alpha}_{i, k-1}\right) .
$$
Taking norm on both sides and separating mean and variance, we proceed as
\begin{align}
	\begin{split}
		\mathbb{E}\|\Delta \boldsymbol{\theta}\|^2 & =\mathbb{E}\left\|-\frac{\tilde{\eta}}{K s} \sum_{k, i \in \mathcal{S}}\left(g_i\left(\boldsymbol{\theta}_{i, k-1}\right)-g_i\left(\boldsymbol{\alpha}_{i, k-1}\right)+\boldsymbol{m}-\boldsymbol{c}_i\right)\right\|^2 \\
		& \leq \mathbb{E}\left\|-\frac{\tilde{\eta}}{K s} \sum_{k, i \in \mathcal{S}}\left(\nabla F_i\left(\boldsymbol{\theta}_{i, k-1}\right)+\mathbb{E}[\boldsymbol{m}]-\mathbb{E}\left[\boldsymbol{c}_i\right]\right)\right\|^2+\frac{9 \tilde{\eta}^2 \sigma^2}{K s} \\
		& \leq \mathbb{E}\left[\frac{\tilde{\eta}^2}{K s} \sum_{k, i \in \mathcal{S}}\left\|\nabla F_i\left(\boldsymbol{\theta}_{i, k-1}\right)+\mathbb{E}[\boldsymbol{m}]-\mathbb{E}\left[\boldsymbol{c}_i\right]\right\|^2\right]+\frac{9 \tilde{\eta}^2 \sigma^2}{K s} \\
		& =\frac{\tilde{\eta}^2}{K m} \sum_{k, i} \mathbb{E}\left\|\left(\nabla F_i\left(\boldsymbol{\theta}_{i, k-1}\right)-\nabla F_i(\boldsymbol{\theta})\right)+(\mathbb{E}[\boldsymbol{m}]-\nabla F(\boldsymbol{\theta}))+\nabla F(\boldsymbol{\theta})-\left(\mathbb{E}\left[\boldsymbol{c}_i\right]-\nabla F_i(\boldsymbol{\theta})\right)\right\|^2+\frac{9 \tilde{\eta}^2 \sigma^2}{K s} \\
		& \leq \frac{4 \tilde{\eta}^2}{K m} \sum_{k, i} \mathbb{E}\left\|\nabla F_i\left(\boldsymbol{\theta}_{i, k-1}\right)-\nabla F_i(\boldsymbol{\theta})\right\|^2+\frac{4 \tilde{\eta}^2}{K m} \sum_{k, i} \mathbb{E}\left\|\nabla F_i\left(\boldsymbol{\alpha}_{i, k-1}\right)-\nabla F_i(\boldsymbol{\theta})\right\|^2 \\
		& +\frac{4 \tilde{\eta}^2}{K m} \sum_{k, i} \mathbb{E}\left\|\nabla F_i\left(\boldsymbol{\theta}_{i, k-1}\right)-\nabla F_i(\boldsymbol{\theta})\right\|^2 \quad+4 \tilde{\eta}^2 \mathbb{E}\|\nabla F(\boldsymbol{\theta})\|^2+\frac{9 \tilde{\eta}^2 \sigma^2}{K s} \\
		& \leq 4 \tilde{\eta}^2 \beta^2 \mathcal{E}_t+4 \beta^2 \tilde{\eta}^2 \Xi_{t-1}+4 \beta^2 \tilde{\eta}^2 \mathcal{G}_{t-1}+4 \tilde{\eta}^2 \mathbb{E}\|\nabla F(\boldsymbol{\theta})\|^2+\frac{9 \tilde{\eta}^2 \sigma^2}{K s} .
	\end{split}
\end{align}

In the first inequality, note that the three random variables $-\frac{1}{K s} \sum_{k, i \in \mathcal{S}} g_i\left(\boldsymbol{\theta}_{i, k}\right), \frac{1}{m} \sum_{i \in \mathcal{S}} \boldsymbol{c}_i$, and $\boldsymbol{m}$-may not be independent but each have variance smaller than $\frac{\sigma^2}{K s}$ and so we can apply Lemma~\ref{lem:independent}. The rest of the inequalities follow from repeated applications of the relaxed triangle inequality, $\beta$-Lipschitzness of $F_i$, and the definition of $\Xi_{t-1}$. This proves the second statement. The first statement follows from our expression of $\mathbb{E}_{t-1}[\Delta \boldsymbol{\theta}]$ and similar computations.
\end{proof}

\paragraph{Lag in the control variates.} We now analyze the `lag' in the control variates due to us sampling only a small subset of clients each round. Because we cannot rely on convexity anymore but only on the Lipschitzness of the gradients, the control-lag increases faster in the non-convex case.
\begin{lemma}\label{lem:non-convex-control-error-sample}
	Assumptions \ref{asm:variance}, \ref{asm:smoothness}, the following holds true for any $\tilde{\eta} \leq \frac{1}{24 \beta}\left(\gamma\right)^a$ for $a \in\left[\frac{1}{2}, 1\right]$ where $\tilde{\eta}:=\eta_l \alpha K$ :
	$$
	\mathcal{G}_{r}\leq\left(1-\frac{17}{36}\gamma\right) \mathcal{G}_{t-1}+\frac{1}{48 \beta^2}\left(\gamma\right)^{2 a-1}\left\|\nabla F\left(\boldsymbol{\theta}^{t-1}\right)\right\|^2+\frac{97}{48}\left(\gamma\right)^{2 a-1} \mathcal{E}_t+\left(\frac{\gamma}{ \beta^2}\right) \frac{\sigma^2}{32 K s} .
	$$
\end{lemma}
\begin{proof}
The proof  except that we cannot rely on convexity. Recall that after round $t$, the definition of $\boldsymbol{\theta}_{i, k-1}^t$ implies that

The proof proceeds similar to that of Lemma~\ref{lem:control-error-sample} except that we cannot rely on convexity.
$$
\mathbb{E}_{\mathcal{S}^t}\left[\boldsymbol{\theta}_{i, k-1}^t\right]=\left(1-\gamma\right) \boldsymbol{\theta}_{i, k-1}^{t-1}+\gamma \boldsymbol{\theta}_{i, k-1}^t .
$$
Plugging the above expression in the definition of $\Xi_t$ we get
\begin{align}
	\begin{split}
		\Xi_t & =\frac{1}{K m} \sum_{i, k} \mathbb{E}\left\|\boldsymbol{\theta}_{i, k-1}^t-\boldsymbol{\theta}^t\right\|^2 \\
		& =\left(1-\gamma\right) \cdot \underbrace{\frac{1}{K m} \sum_i \mathbb{E}\left\|\boldsymbol{\theta}_{i, k-1}^{t-1}-\boldsymbol{\theta}^t\right\|^2}_{\mathcal{T}_5}+\gamma \cdot \underbrace{\frac{1}{K m} \sum_{k, i} \mathbb{E}\left\|\boldsymbol{\theta}_{i, k-1}^t-\boldsymbol{\theta}^t\right\|^2}_{\mathcal{T}_6} .
	\end{split}
\end{align}
We can expand the second term $\mathcal{T}_6$ with the relaxed triangle inequality to claim
$$
\mathcal{T}_6 \leq 2\left(\mathcal{E}_t+\mathbb{E}\left\|\Delta \boldsymbol{\theta}^t\right\|^2\right) .
$$
We will expand the first term $\mathcal{T}_5$ to claim for a constant $b \geq 0$ to be chosen later
\begin{align}
	\begin{split}
		\mathcal{T}_5 & =\frac{1}{K m} \sum_i \mathbb{E}\left(\left\|\boldsymbol{\theta}_{i, k-1}^{t-1}-\boldsymbol{\theta}^{t-1}\right\|^2+\left\|\Delta \boldsymbol{\theta}^t\right\|^2+\mathbb{E}_{t-1}\left\langle\Delta \boldsymbol{\theta}^t, \boldsymbol{\theta}_{i, k-1}^{t-1}-\boldsymbol{\theta}^{t-1}\right\rangle\right) \\
		& \leq \frac{1}{K m} \sum_i \mathbb{E}\left(\left\|\boldsymbol{\theta}_{i, k-1}^{t-1}-\boldsymbol{\theta}^{t-1}\right\|^2+\left\|\Delta \boldsymbol{\theta}^t\right\|^2+\frac{1}{b}\left\|\mathbb{E}_{t-1}\left[\Delta \boldsymbol{\theta}^t\right]\right\|^2+b\left\|\boldsymbol{\theta}_{i, k-1}^{t-1}-\boldsymbol{\theta}^{t-1}\right\|^2\right)
	\end{split}
\end{align}
where we used Young's inequality which holds for any $b \geq 0$. Combining the bounds for $\mathcal{T}_5$ and $\mathcal{T}_6$,
\begin{align}
	\begin{split}
		\mathcal{G}_{r} & \leq\left(1-\gamma\right)(1+b) \mathcal{G}_{t-1}+2 \gamma\mathcal{E}_t+2 \mathbb{E}\left\|\Delta \boldsymbol{\theta}^t\right\|^2+\frac{1}{b} \mathbb{E}\left\|\mathbb{E}_{t-1}\left[\Delta \boldsymbol{\theta}^t\right]\right\|^2 \\
		& \left.\leq\left(\left(1-\gamma\right)(1+b)+16 \tilde{\eta}^2 \beta^2\right) \Xi_{t-1}+\left(2\gamma{m}+8 \tilde{\eta}^2 \beta^2+2 \frac{1}{b} \tilde{\eta}^2 \beta^2\right) \mathcal{E}_t+\left(8+2 \frac{1}{b}\right) \tilde{\eta}^2 \mathbb{E}\|\nabla F(\boldsymbol{\theta})\|^2\right)+\frac{18 \tilde{\eta}^2 \sigma^2}{K s}
	\end{split}
\end{align}
Verify that with choice of $b=\frac{\gamma}{2(1-\gamma)}$, we have $\left(1-\gamma\right)(1+b) \leq\left(1-\frac{\gamma}{2 }\right)$ and $\frac{1}{b} \leq \frac{2}{\gamma}$. Plugging these values along with the bound on the step-size $16 \beta^2 \tilde{\eta}^2 \leq \frac{1}{36}\left(\gamma\right)^{2 a} \leq \frac{\gamma}{36}$ completes the lemma.
Bounding the drift. We will next bound the client drift $\mathcal{E}_t$. 
\end{proof}

\paragraph{Lag in the control variates.} We now analyze the `lag' in the control variates due to us sampling only a small subset of clients each round. Because we cannot rely on convexity anymore but only on the Lipschitzness of the gradients, the control-lag increases faster in the non-convex case.
\begin{lemma}\label{lem:non-convex-control-error-sample}
	Assumptions \ref{asm:variance}--\ref{asm:smoothness}, the following holds true for any $\tilde{\eta} \leq \frac{1}{24 \beta}\left(\frac{s}{m}\right)^a$ for $a \in\left[\frac{1}{2}, 1\right]$ where $\tilde{\eta}:=\eta_l \alpha K$ :
	$$
	\Xi_t \leq\left(1-\frac{17 s}{36 m}\right) \Xi_{t-1}+\frac{1}{48 \beta^2}\left(\frac{s}{m}\right)^{2 a-1}\left\|\nabla F\left(\boldsymbol{\theta}^{t-1}\right)\right\|^2+\frac{97}{48}\left(\frac{s}{m}\right)^{2 a-1} \mathcal{E}_t+\left(\frac{s}{m \beta^2}\right) \frac{\sigma^2}{32 K s} .
	$$
\end{lemma}

\begin{proof}
The proof   except that we cannot rely on convexity. Recall that after round $t$, the definition of $\boldsymbol{\alpha}_{i, k-1}^t$ implies that
$$
\mathbb{E}_{\mathcal{S}^t}\left[\boldsymbol{\alpha}_{i, k-1}^t\right]=\left(1-\frac{s}{m}\right) \boldsymbol{\alpha}_{i, k-1}^{t-1}+\frac{s}{m} \boldsymbol{\theta}_{i, k-1}^t .
$$
Plugging the above expression in the definition of $\Xi_t$ we get
\begin{align}
	\begin{split}
		\Xi_t & =\frac{1}{K m} \sum_{i, k} \mathbb{E}\left\|\boldsymbol{\alpha}_{i, k-1}^t-\boldsymbol{\theta}^t\right\|^2 \\
		& =\left(1-\frac{s}{m}\right) \cdot \underbrace{\frac{1}{K m} \sum_i \mathbb{E}\left\|\boldsymbol{\alpha}_{i, k-1}^{t-1}-\boldsymbol{\theta}^t\right\|^2}_{\mathcal{T}_5}+\frac{s}{m} \cdot \underbrace{\frac{1}{K m} \sum_{k, i} \mathbb{E}\left\|\boldsymbol{\theta}_{i, k-1}^t-\boldsymbol{\theta}^t\right\|^2}_{\mathcal{T}_6} .
	\end{split}
\end{align}
We can expand the second term $\mathcal{T}_6$ with the relaxed triangle inequality to claim
\begin{align}
	\begin{split}
		\mathcal{T}_6 \leq 2\left(\mathcal{E}_t+\mathbb{E}\left\|\Delta \boldsymbol{\theta}^t\right\|^2\right) .
	\end{split}
\end{align}
We will expand the first term $\mathcal{T}_5$ to claim for a constant $b \geq 0$ to be chosen later
\begin{align}
	\begin{split}
		\mathcal{T}_5 & =\frac{1}{K m} \sum_i \mathbb{E}\left(\left\|\boldsymbol{\alpha}_{i, k-1}^{t-1}-\boldsymbol{\theta}^{t-1}\right\|^2+\left\|\Delta \boldsymbol{\theta}^t\right\|^2+\mathbb{E}_{t-1}\left\langle\Delta \boldsymbol{\theta}^t, \boldsymbol{\alpha}_{i, k-1}^{t-1}-\boldsymbol{\theta}^{t-1}\right\rangle\right) \\
		& \leq \frac{1}{K m} \sum_i \mathbb{E}\left(\left\|\boldsymbol{\alpha}_{i, k-1}^{t-1}-\boldsymbol{\theta}^{t-1}\right\|^2+\left\|\Delta \boldsymbol{\theta}^t\right\|^2+\frac{1}{b}\left\|\mathbb{E}_{t-1}\left[\Delta \boldsymbol{\theta}^t\right]\right\|^2+b\left\|\boldsymbol{\alpha}_{i, k-1}^{t-1}-\boldsymbol{\theta}^{t-1}\right\|^2\right)
	\end{split}
\end{align}
where we used Young's inequality which holds for any $b \geq 0$. Combining the bounds for $\mathcal{T}_5$ and $\mathcal{T}_6$,
\begin{align}
	\begin{split}
		\Xi_t & \leq\left(1-\frac{s}{m}\right)(1+b) \Xi_{t-1}+2 \frac{s}{m} \mathcal{E}_t+2 \mathbb{E}\left\|\Delta \boldsymbol{\theta}^t\right\|^2+\frac{1}{b} \mathbb{E}\left\|\mathbb{E}_{t-1}\left[\Delta \boldsymbol{\theta}^t\right]\right\|^2 \\
		& \left.\leq\left(\left(1-\frac{s}{m}\right)(1+b)+16 \tilde{\eta}^2 \beta^2\right) \Xi_{t-1}+\left(\frac{2 s}{m}+8 \tilde{\eta}^2 \beta^2+2 \frac{1}{b} \tilde{\eta}^2 \beta^2\right) \mathcal{E}_t+\left(8+2 \frac{1}{b}\right) \tilde{\eta}^2 \mathbb{E}\|\nabla F(\boldsymbol{\theta})\|^2\right)+\frac{18 \tilde{\eta}^2 \sigma^2}{K s}
	\end{split}
\end{align}
Verify that with choice of $b=\frac{s}{2(m-s)}$, we have $\left(1-\frac{s}{m}\right)(1+b) \leq\left(1-\frac{s}{2 m}\right)$ and $\frac{1}{b} \leq \frac{2m}{s}$. Plugging these values along with the bound on the step-size $16 \beta^2 \tilde{\eta}^2 \leq \frac{1}{36}\left(\frac{s}{m}\right)^{2 a} \leq \frac{s}{36m}$ completes the lemma.
Bounding the drift. We will next bound the client drift $\mathcal{E}_t$. 

\end{proof}

\paragraph{Bounding the drift.}
We will next bound the client drift $\mathcal{E}_t$. For this, convexity is not crucial and we will recover a very similar result to Lemma~\ref{lem:drift-bound-sampling} only use the Lipschitzness of the gradient.
\begin{lemma}\label{lem:nonconvex-drift-bound-sampling}
	Suppose our step-sizes satisfy $\eta_l \leq \frac{1}{24 \beta K \alpha}$ and $F_i$ satisfies assumptions  \ref{asm:variance}--\ref{asm:smoothness}. Then, for any global $\alpha \geq 1$ we can bound the drift as
	$$
	\frac{5}{3} \beta^2 \tilde{\eta} \mathcal{E}_t \leq \frac{5}{6} \beta^3 \tilde{\eta}^2 \Xi_{t-1}+\frac{5}{6} \beta^3 \tilde{\eta}^2 \mathcal{G}_{t-1}+\frac{\tilde{\eta}}{24 \alpha^2} \mathbb{E}\left\|\nabla F\left(\boldsymbol{\theta}^{t-1}\right)\right\|^2+\frac{\tilde{\eta}^2 \beta}{4 K \alpha^2} \sigma^2
	$$\label{L15}
\end{lemma}
\begin{proof}
 First, observe that if $K=1, \mathcal{E}_t=0$ since $\boldsymbol{\theta}_{i, 0}=\boldsymbol{\theta}$ for all $i \in[m]$ and that $\Xi_{t-1}$ and the right hand side are both positive. Thus the Lemma is trivially true if $K=1$ and we will henceforth assume $K \geq 2$. Starting from the update  for $i \in[m]$ and $k \in[K]$.
\begin{align}
	\begin{split}
		\mathbb{E}\left\|\boldsymbol{\theta}_{i, k}-\boldsymbol{\theta}\right\|^2= & \mathbb{E}\left\|\boldsymbol{\theta}_{i, k-1}-\eta_l\left(g_i\left(\boldsymbol{\theta}_{i, k-1}\right)+\boldsymbol{m}-\boldsymbol{c}_i\right)-\boldsymbol{\theta}\right\|^2 \\
		\leq & \mathbb{E}\left\|\boldsymbol{\theta}_{i, k-1}-\eta_l\left(\nabla F_i\left(\boldsymbol{\theta}_{i, k-1}\right)+\boldsymbol{m}-\boldsymbol{c}_i\right)-\boldsymbol{\theta}\right\|^2+\eta_l^2 \sigma^2 \\
		\leq & \left(1+\frac{1}{K-1}\right) \mathbb{E}\left\|\boldsymbol{\theta}_{i, k-1}-\boldsymbol{\theta}\right\|^2+K \eta_l^2 \mathbb{E}\left\|\nabla F_i\left(\boldsymbol{\theta}_{i, k-1}\right)+\boldsymbol{m}-\boldsymbol{c}_i\right\|^2+\eta_l^2 \sigma^2 \\
		= & \left(1+\frac{1}{K-1}\right) \mathbb{E}\left\|\boldsymbol{\theta}_{i, k-1}-\boldsymbol{\theta}\right\|^2+\eta_l^2 \sigma^2 \\
		& \quad+K \eta_l^2 \mathbb{E} \| \nabla F_i\left(\boldsymbol{\theta}_{i, k-1}\right)-\nabla F_i(\boldsymbol{\theta})+(\boldsymbol{m}-\nabla F(\boldsymbol{\theta}))+\nabla F(\boldsymbol{\theta})-\left(\boldsymbol{c}_i-\nabla F_i(\boldsymbol{\theta}) \|^2\right. \\
		\leq & \left(1+\frac{1}{K-1}\right) \mathbb{E}\left\|\boldsymbol{\theta}_{i, k-1}-\boldsymbol{\theta}\right\|^2+4 K \eta_l^2 \mathbb{E}\left\|\nabla F_i\left(\boldsymbol{\theta}_{i, k-1}\right)-\nabla F_i(\boldsymbol{\theta})\right\|^2+\eta_l^2 \sigma^2 \\
		& \quad+4 K \eta_l^2 \mathbb{E}\|\boldsymbol{m}-\nabla F(\boldsymbol{\theta})\|^2+4 K \eta_l^2 \mathbb{E}\|\nabla F(\boldsymbol{\theta})\|^2+4 K \eta_l^2 \mathbb{E}\left\|\boldsymbol{c}_i-\nabla F_i(\boldsymbol{\theta})\right\|^2 \\
		\leq & \left(1+\frac{1}{K-1}+4 K \beta^2 \eta_l^2\right) \mathbb{E}\left\|\boldsymbol{\theta}_{i, k-1}-\boldsymbol{\theta}\right\|^2+\eta_l^2 \sigma^2+4 K \eta_l^2 \mathbb{E}\|\nabla F(\boldsymbol{\theta})\|^2 \\
		& \quad+4 K \eta_l^2 \mathbb{E}\|\boldsymbol{m}-\nabla F(\boldsymbol{\theta})\|^2+4 K \eta_l^2 \mathbb{E}\left\|\boldsymbol{c}_i-\nabla F_i(\boldsymbol{\theta})\right\|^2
	\end{split}
\end{align}
The inequalities above follow from repeated application of the relaxed triangle inequalities and the $\beta$-Lipschitzness of $F_i$.
Averaging the above over $i$, the definition of $\boldsymbol{m}=\frac{1}{m} \sum_i \boldsymbol{c}_i$ and $\Xi_{t-1}$  gives
\begin{align}
	\begin{split}
		\frac{1}{m} \sum_i \mathbb{E}\left\|\boldsymbol{\theta}_{i, k}-\boldsymbol{\theta}\right\|^2 \leq & \left(1+\frac{1}{K-1}+4 K \beta^2 \eta_l^2\right) \frac{1}{m} \sum_i \mathbb{E}\left\|\boldsymbol{\theta}_{i, k-1}-\boldsymbol{\theta}\right\|^2 \\
		& +\eta_l^2 \sigma^2+4 K \eta_l^2 \mathbb{E}\|\nabla F(\boldsymbol{\theta})\|^2+4 K \eta_l^2 \beta^2 \Xi_{t-1}+4 K \eta_l^2 \beta^2 \mathcal{G}_{t-1}\\
		\leq & \left(\eta_l^2 \sigma^2+4 K \eta_l^2 \mathbb{E}\|\nabla F(\boldsymbol{\theta})\|^2+4 K \eta_l^2 \beta^2 \Xi_{t-1}+4 K \eta_l^2 \beta^2 \mathcal{G}_{t-1}\right)\left(\sum_{k=0}^{k-1}\left(1+\frac{1}{K-1}+4 K \beta^2 \eta_l^2\right)^k\right) \\
		= & \left(\frac{\tilde{\eta}^2 \sigma^2}{K^2 \alpha^2}+\frac{4 \tilde{\eta}^2}{K \alpha^2} \mathbb{E}\|\nabla F(\boldsymbol{\theta})\|^2+\frac{4 \tilde{\eta}^2 \beta^2}{K \alpha^2} \Xi_{t-1}+\frac{4 \tilde{\eta}^2 \beta^2}{K \alpha^2} \mathcal{G}_{t-1}\right)\left(\sum_{k=0}^{k-1}\left(1+\frac{1}{K-1}+\frac{4 \beta^2 \tilde{\eta}^2}{K \alpha^2}\right)^k\right) \\
		\leq & \left(\frac{\tilde{\eta} \sigma^2}{24 \beta K^2 \alpha^2}+\frac{1}{144 \beta^2 K \alpha^2} \mathbb{E}\|\nabla F(\boldsymbol{\theta})\|^2+\frac{\tilde{\eta} \beta}{6 K \alpha^2} \Xi_{t-1}+\frac{\tilde{\eta} \beta}{6 K \alpha^2} \mathcal{G}_{t-1}\right) 3 K .
	\end{split}
\end{align}
The last inequality used the bound on the step-size $\beta \tilde{\eta} \leq \frac{1}{24}$. Averaging over $k$ and multiplying both sides by $\frac{5}{3} \beta^2 \tilde{\eta}$ yields the lemma statement. Progress made in each round. Given that we can bound all sources of error, we can finally prove the progress made in each round.
\end{proof}

\paragraph{Progress made in each round.} Given that we can bound all sources of error, we can finally prove the progress made in each round.
\begin{lemma}\label{lem:non-convex-progress-sample}
	Suppose the updates  satisfy assumptions \ref{asm:variance}--\ref{asm:smoothness}. For any effective step-size $\tilde{\eta}:=K \alpha \eta_l$ satisfying $\tilde{\eta} \leq \frac{1}{24 \beta}\left(\frac{s}{m}\right)^{\frac{2}{3}}$,
	$$
	\left(\mathbb{E}\left[f\left(\boldsymbol{\theta}^t\right)\right]+12 \beta^3 \tilde{\eta}^2 \frac{m}{s} \Xi_t\right) \leq\left(\mathbb{E}\left[f\left(\boldsymbol{\theta}^{t-1}\right)\right]+12 \beta^3 \tilde{\eta}^2 \frac{m}{s} \Xi_{t-1}\right)+\frac{5 \beta \tilde{\eta}^2 \sigma^2}{K s}\left(1+\frac{s}{\alpha^2}\right)-\frac{\tilde{\eta}}{14} \mathbb{E}\left\|\nabla F\left(\boldsymbol{\theta}^{t-1}\right)\right\|^2 .
	$$\label{L16}
\end{lemma}
\begin{proof}
    
Starting from the smoothness of $f$ and taking conditional expectation gives
\begin{align}
	\begin{split}
		\mathbb{E}_{t-1}[f(\boldsymbol{\theta}+\Delta \boldsymbol{\theta})] \leq f(\boldsymbol{\theta})+\left\langle\nabla F(\boldsymbol{\theta}), \mathbb{E}_{t-1}[\Delta \boldsymbol{\theta}]\right\rangle+\frac{\beta}{2} \mathbb{E}_{t-1}\|\Delta \boldsymbol{\theta}\|^2 .
	\end{split}
\end{align}
We as usual dropped the superscript everywhere. Recall that the server update can be written as
\begin{align}
	\begin{split}
		\Delta \boldsymbol{\theta}=-\frac{\tilde{\eta}}{K s} \sum_{k, i \in \mathcal{S}}\left(g_i\left(\boldsymbol{\theta}_{i, k-1}\right)+\boldsymbol{m}-\boldsymbol{c}_i\right), \text { and } \mathbb{E}_{\mathcal{S}}[\Delta \boldsymbol{\theta}]=-\frac{\tilde{\eta}}{K m} \sum_{k, i} g_i\left(\boldsymbol{\theta}_{i, k-1}\right) .
	\end{split}
\end{align}
Substituting this in the previous inequality  to bound $\mathbb{E}\left[\|\Delta \boldsymbol{\theta}\|^2\right]$ gives
\begin{align}
	\begin{split}
		&\mathbb{E}[f(\boldsymbol{\theta}+\Delta \boldsymbol{\theta})]-f(\boldsymbol{\theta}) \\
		& \leq-\frac{\tilde{\eta}}{K m} \sum_{k, i}\left\langle\nabla F(\boldsymbol{\theta}), \mathbb{E}\left[\nabla F_i\left(\boldsymbol{\theta}_{i, k-1}\right)\right]\right\rangle+\frac{\beta}{2} \mathbb{E}\|\Delta \boldsymbol{\theta}\|^2 \\
		& \leq-\frac{\tilde{\eta}}{K m} \sum_{k, i}\left\langle\nabla F(\boldsymbol{\theta}), \mathbb{E}\left[\nabla F_i\left(\boldsymbol{\theta}_{i, k-1}\right)\right]\right\rangle+ \\
		& \quad 2 \tilde{\eta}^2 \beta^3 \mathcal{E}_t+2 \tilde{\eta}^2 \beta^3 \Xi_{t-1}+2 \tilde{\eta}^2 \beta^3 \mathcal{G}_{t-1}+2 \beta \tilde{\eta}^2 \mathbb{E}\|\nabla F(\boldsymbol{\theta})\|^2+\frac{9 \beta \tilde{\eta}^2 \sigma^2}{2 K s} \\
		& \leq-\frac{\tilde{\eta}}{2}\|\nabla F(\boldsymbol{\theta})\|^2+\frac{\tilde{\eta}}{2} \sum_{i, k} \mathbb{E}\left\|\frac{1}{K m} \sum_{i, k} \nabla F_i\left(\boldsymbol{\theta}_{i, k-1}\right)-\nabla F(\boldsymbol{\theta})\right\|^2+ \\
		&2 \tilde{\eta}^2 \beta^3 \mathcal{E}_t+2 \tilde{\eta}^2 \beta^3 \Xi_{t-1}+2 \tilde{\eta}^2 \beta^3 \mathcal{G}_{t-1}+2 \beta \tilde{\eta}^2 \mathbb{E}\|\nabla F(\boldsymbol{\theta})\|^2+\frac{9 \beta \tilde{\eta}^2 \sigma^2}{2 K s}\\
		& \leq-\frac{\tilde{\eta}}{2}\|\nabla F(\boldsymbol{\theta})\|^2+\frac{\tilde{\eta}}{2 K m} \sum_{i, k} \mathbb{E}\left\|\nabla F_i\left(\boldsymbol{\theta}_{i, k-1}\right)-\nabla F_i(\boldsymbol{\theta})\right\|^2+ \\
		& 2 \tilde{\eta}^2 \beta^3 \mathcal{E}_t+2 \tilde{\eta}^2 \beta^3 \Xi_{t-1}+2 \tilde{\eta}^2 \beta^3 \mathcal{G}_{t-1}+2 \beta \tilde{\eta}^2 \mathbb{E}\|\nabla F(\boldsymbol{\theta})\|^2+\frac{9 \beta \tilde{\eta}^2 \sigma^2}{2 K s} \\
		& \leq-\left(\frac{\tilde{\eta}}{2}-2 \beta \tilde{\eta}^2\right)\|\nabla F(\boldsymbol{\theta})\|^2+\left(\frac{\tilde{\eta}}{2}+2 \beta \tilde{\eta}^2\right) \beta^2 \mathcal{E}_t+2 \beta^3 \tilde{\eta}^2 \mathcal{G}_{t-1}+2 \beta^3 \tilde{\eta}^2 \Xi_{t-1}+\frac{9 \beta \tilde{\eta}^2 \sigma^2}{2 K s} .
	\end{split}
\end{align}
The third inequality follows from the observation that $-a b=\frac{1}{2}\left((b-a)^2-a^2\right)-\frac{1}{2} b^2 \leq \frac{1}{2}\left((b-a)^2-a^2\right)$ for any $a, b \in \mathbb{R}$, and the last from the $\beta$-Lipschitzness of $F_i$. Now we use Lemma~\ref{lem:non-convex-control-error-sample} to $\Xi_t$ as
\begin{align}
	\begin{split}
		12 \beta^3 \tilde{\eta}^2 \frac{m}{s} \Xi_t & \leq 12 \beta^3 \tilde{\eta}^2 \frac{m}{s}\left(\left(1-\frac{17s}{36 m}\right) \Xi_{t-1}+\frac{1}{48 \beta^2}\left(\frac{s}{m}\right)^{2 a-1}\left\|\nabla F\left(\boldsymbol{\theta}^{t-1}\right)\right\|^2+\frac{97}{48}\left(\frac{s}{m}\right)^{2 a-1} \mathcal{E}_t+\left(\frac{s}{m\beta^2}\right) \frac{\sigma^2}{32 K s}\right) \\
		& =12 \beta^3 \tilde{\eta}^2 \frac{m}{s} \Xi_{t-1}-\frac{17}{3} \beta^3 \tilde{\eta}^2 \Xi_{t-1}+\frac{1}{4} \beta \tilde{\eta}^2\left(\frac{m}{s}\right)^{2-2 a}\|\nabla F(\boldsymbol{\theta})\|^2+\frac{97}{4} \beta^3 \tilde{\eta}^2\left(\frac{m}{s}\right)^{2-2 a} \mathcal{E}_t+\frac{3 \beta \tilde{\eta}^2 \sigma^2}{8 K s} .
	\end{split}
\end{align}
Now we use Lemma ~\ref{lem:nonconvex-drift-bound-sampling} to bound $\mathcal{G}_{t-1}$ as
\begin{align}
	\begin{split}
		12 \beta^3 \tilde{\eta}^2 \frac{1}{\gamma} \mathcal{G}_{r} & \leq 12 \beta^3 \tilde{\eta}^2 \frac{1}{\gamma}\left(\left(1-\frac{17 \gamma}{36}\right) \mathcal{G}_{t-1}+\frac{1}{48 \beta^2}\left(\gamma\right)^{2 a-1}\left\|\nabla F\left(\boldsymbol{\theta}^{t-1}\right)\right\|^2+\frac{97}{48}\left(\gamma\right)^{2 a-1} \mathcal{G}_{t-1}+\left(\frac{\gamma}{ \beta^2}\right) \frac{\sigma^2}{32 K s}\right) \\
		& =12 \beta^3 \tilde{\eta}^2 \frac{1}{\gamma}\mathcal{G}_{t-1}-\frac{17}{3} \beta^3 \tilde{\eta}^2 \mathcal{G}_{t-1}+\frac{1}{4} \beta \tilde{\eta}^2\left(\frac{1}{\gamma}\right)^{2-2 \alpha}\|\nabla F(\boldsymbol{\theta})\|^2+\frac{97}{4} \beta^3 \tilde{\eta}^2\left(\frac{1}{\gamma}\right)^{2-2 a} \mathcal{E}_t+\frac{3 \beta \tilde{\eta}^2 \sigma^2}{8 K s} .
	\end{split}
\end{align}
Also recall that Lemma \refeq{L16} states that
\begin{align}
	\begin{split}
		\frac{5}{3} \beta^2 \tilde{\eta} \mathcal{E}_t \leq \frac{5}{6} \beta^3 \tilde{\eta}^2 \Xi_{t-1}+\frac{5}{6} \beta^3 \tilde{\eta}^2 \mathcal{G}_{t-1}+\frac{\tilde{\eta}}{24 a^2} \mathbb{E}\left\|\nabla F\left(\boldsymbol{\theta}^{t-1}\right)\right\|^2+\frac{\tilde{\eta}^2 \beta}{4 K \alpha^2} \sigma^2 .
	\end{split}
\end{align}
Adding these bounds on $\Xi_t$ and $\mathcal{E}_t$ to that of $\mathbb{E}[f(\boldsymbol{\theta}+\Delta \boldsymbol{\theta})]$ gives
\begin{align}
	\begin{split}
		(\mathbb{E}[f(\boldsymbol{\theta} & \left.+\Delta \boldsymbol{\theta})]+12 \beta^3 \tilde{\eta}^2 \frac{m}{S} \Xi_t+12 \beta^3 \tilde{\eta}^2 \frac{1}{\gamma} \mathcal{G}_{r}\right) \leq\left(\mathbb{E}[f(\boldsymbol{\theta})]+12 \beta^3 \tilde{\eta}^2 \frac{m}{s} \Xi_{t-1}\right)+\left(\frac{5}{6}-\frac{17}{3}\right) \beta^3 \tilde{\eta}^2 \Xi_{t-1} \\
		&+\left(\frac{5}{6}-\frac{17}{3}\right) \beta^3 \tilde{\eta}^2 \mathcal{G}_{t-1} -\left(\frac{\tilde{\eta}}{2}-2 \beta \tilde{\eta}^2-\frac{1}{4} \beta \tilde{\eta}^2\left(\frac{m}{s}\right)^{2-2 a}\right)\|\nabla F(\boldsymbol{\theta})\|^2\\
		&+\left(\frac{\tilde{\eta}}{2}-\frac{5 \tilde{\eta}}{3}+2 \beta \tilde{\eta}^2+\frac{97}{4} \beta \tilde{\eta}^2\left(\frac{m}{s}\right)^{2-2 a}\right) \beta^2 \mathcal{E}_t+\frac{39 \beta \tilde{\eta}^2 \sigma^2}{8 K s}\left(1+\frac{s}{\alpha^2}\right)
	\end{split}
\end{align}
By our choice of $a=\frac{2}{3}$ and plugging in the bound on step-size $\beta \tilde{\eta}\left(\frac{m}{s}\right)^{2-2a} \leq \frac{1}{24}$ proves the lemma. The non-convex rate of convergence now follows by unrolling the recursion  and selecting an appropriate step-size $\tilde{\eta}$. Finally note that if we initialize $\boldsymbol{c}_i^0=g_i\left(\boldsymbol{\theta}^0\right)$ then we have $\Xi_0=0$.

The \textbf{non-convex rate} of convergence now follows by unrolling the recursion in Lemma~\ref{lem:non-convex-progress-sample} and selecting an appropriate step-size $\tilde\eta$ as in Lemma~\ref{lemma:general}. Finally note that if $\boldsymbol{c}_i^0=g_i\left(\boldsymbol{\theta}^0\right)$ then we have $\Xi_0=0$.

\end{proof}

\section{Effectiveness of local learning rate decay }
\begin{theorem}[Effectiveness of local learning rate decay]
	We define client-drift to be how much the clients move from their starting point. When using a constant learning rate $\eta_k^{(t)}=\eta_l$  , we define the client-side drift to be $\mathcal{E}_t:=\frac{1}{km} \sum_{k=1}^k \sum_{i=1}^m \mathbb{E}\left[\left\|\boldsymbol{v}_{i, k}^t-\boldsymbol{\theta}^{t-1}\right\|^2\right]$, and when using local learning rate decay $\eta_k^{(t)}=\left(\frac{K-2}{K-1}\right)^k$,	$\tilde{\mathcal{E}}_t:=\frac{1}{k m} \sum_{k=1}^k \sum_{i=1}^m \mathbb{E}\left[\left\|\tilde{\boldsymbol{v}}_{i, k}^{t}-\tilde{\boldsymbol{\theta}}^{t-1}\right\|^2\right]$.
	$$
	\begin{aligned}
		& \mathcal{E}_t \leq \frac{\tilde{\eta}}{9 \beta} \mathcal{C}_{t-1}+\frac{\tilde{\eta}}{9 \beta} \mathcal{D}_{t-1}+\frac{1}{75 \alpha^2 \beta}\left(\mathbb{E}\left[f\left(\mathbf{x}^{t-1}\right)\right]-f\left(\mathbf{x}^{\star}\right)\right)+\frac{\tilde{\eta}}{3 K \alpha^2 \beta} \sigma^2 \\
		& \tilde{\mathcal{E}}_t \leq \frac{1}{3}\left[\frac{\tilde{\eta}}{9 \beta} \mathcal{C}_{t-1}+\frac{\tilde{\eta}}{9 \beta} \mathcal{D}_{t-1}+\frac{1}{75 \alpha^2 \beta}\left(\mathbb{E}\left[f\left(\mathbf{x}^{t-1}\right)\right]-f\left(\mathbf{x}^{\star}\right)\right)+\frac{\tilde{\eta}}{3 K \alpha^2 \beta} \sigma^2\right].
	\end{aligned}
	$$
	So local learning rate decay client-side drift $\tilde{\mathcal{E}}_t$ is better than $\mathcal{E}_t$
	
\end{theorem}	

\section{Generalization Analysis for FedSAM and  MoFedSAM   under non-convex setting }	

\begin{lemma}
	Suppose Assumptions 2-5 hold. Then for FedSAM with $\eta^t \leq  / \beta KT$,
	$$
	\mathbb{E}\left\|\theta_{i, k}^t-\theta_t\right\| \leq(1+\overline{\mathbf{c}})^{K -1} \tilde{\eta}_{ t}\left(\mathbb{E}\left\|\nabla F\left(\theta_t\right)\right\|+ \sigma_{g, i}+\sigma\right), \quad \forall k=1, \ldots, K,
	$$
	where $\tilde{\eta}_{ t}=\sum_{k=0}^{K-1} \eta^t$, 	 $\overline{c} \leq 1+(1+c)^{K-1} \beta \tilde{\eta}_t \leq 1+(1+c)^{K-1} \beta \sum_{k=0}^{K-1} \eta^t$ \\
\end{lemma}
\textit{Proof.} Considering local update  of FedSAM
\begin{align}
	\begin{split}
		\mathbb{E}\left\|\theta_{i, k+1}-\theta_t\right\| & =\mathbb{E}\left\|\theta_{i, k}^t-\eta^t g_i\left(\theta_{i, k}^t\right)-\theta_t\right\| \\
		& \leq \mathbb{E}\left\|\theta_{i, k}^t-\theta_t-\eta^t\left(g_i\left(\theta_{i, k}^t\right)-g_i\left(\theta_t\right)\right)\right\|+\eta^t \mathbb{E}\left\|g_i\left(\theta_t\right)\right\| \\
		& \leq\left(1+\beta \eta^t\right) \mathbb{E}\left\|\theta_{i, k}^t-\theta_t\right\|+\eta^t \mathbb{E}\left\|g_i\left(\theta_t\right)\right\| \\
		& \leq\left(1+\beta \eta^t\right) \mathbb{E}\left\|\theta_{i, k}^t-\theta_t\right\|+\eta^t\left(\mathbb{E}\left\|g_i\left(\theta_t\right)-\nabla F_i\left(\theta_t\right)\right\|+\mathbb{E}\left\|\nabla F_i\left(\theta_t\right)\right\|\right) \\
		& \leq\left(1+\beta \eta^t\right) \mathbb{E}\left\|\theta_{i, k}^t-\theta_t\right\|+\eta^t\left(\mathbb{E}\left\|\nabla F_i\left(\theta_t\right)\right\|+\sigma\right),
	\end{split}
\end{align}
where we use Assumptions \ref{A2} and \ref{A3}. Unrolling the above and noting $\theta_{i, 0}=\theta_t$ yields
\begin{align}
	\begin{split}
		\mathbb{E}\left\|\theta_{i, k}^t-\theta_t\right\| & \leq \sum_{l=0}^{k-1} \eta^t\left(\mathbb{E}\left\|\nabla F_i\left(\theta_t\right)\right\|+\sigma\right)(1+c)^{k-1-l} \\
		& \leq \sum_{l=0}^{K-1} \eta^t\left(\mathbb{E}\left\|\nabla F_i\left(\theta_t\right)\right\|+\sigma\right)(1+c)^{K-1} \\
		& \leq(1+c)^{K-1} \tilde{\eta}_{ t}\left(\mathbb{E}\left\|\nabla F\left(\theta_t\right)\right\|+ \sigma_{g, i}+\sigma\right),
	\end{split}
\end{align}
where the last inequality follows Assumption \ref{A5}.

\begin{lemma}
	Given Assumptions $1-3$, for $\eta^t \leq  / \beta KT$, we have
	$$
	\mathbb{E}\left\|g_i\left(\theta_{i, k}^t\right)\right\| \leq\left(1+(1+c)^{K-1} \beta \tilde{\eta}_{ t}\right)\left(\mathbb{E}\left\|\nabla F\left(\theta_t\right)\right\|+ \sigma_{g, i}+\sigma\right)
	$$
	where $g_i(\cdot)$ is the sampled gradient of client $i, \tilde{\eta}_{ t}=\sum_{k=0}^{K-1} \eta^t$.
\end{lemma}
\textit{Proof.}  we obtain
\begin{align}
	\begin{split}
		\mathbb{E}\left\|g_i\left(\theta_{i, k}^t\right)\right\| & \leq \mathbb{E}\left\|g_i\left(\theta_{i, k}^t\right)-\nabla F_i\left(\theta_{i, k}^t\right)\right\|+\mathbb{E}\left\|\nabla F_i\left(\theta_{i, k}^t\right)\right\| \\
		& \leq \mathbb{E}\left\|\nabla F_i\left(\theta_{i, k}^t\right)\right\|+\sigma \\
		& \leq \mathbb{E}\left\|\nabla F_i\left(\theta_t\right)\right\|+\mathbb{E}\left\|\nabla F_i\left(\theta_{i, k}^t\right)-\nabla F_i\left(\theta_t\right)\right\|+\sigma \\
		& \leq \mathbb{E}\left\|\nabla F\left(\theta_t\right)\right\|+\mathbb{E}\left\|\nabla F_i\left(\theta_t\right)-\nabla F\left(\theta_t\right)\right\|+\beta \mathbb{E}\left\|\theta_{i, k}^t-\theta_t\right\|+\sigma \\
		& \leq\left(1+(1+c)^{K-1} \beta \tilde{\alpha}_i\right)\left(\mathbb{E}\left\|\nabla F\left(\theta_t\right)\right\|+ \sigma_{g, i}+\sigma\right) .
	\end{split}
\end{align}
\begin{theorem} (FedSAM). Suppose Assumptions 1-3 hold and consider FedSAM. Let $k=K, \forall i \in[m]$ and $\eta^t \leq \frac{1}{\beta KT}$. Then,
	$$
	\varepsilon_{\text {gen }}\leq \frac{2L c}{m n} \frac{1}{\beta} e^{1+\frac{1}{ T}}\left(\mathbb{E}\left\|\nabla F\left(\theta_t\right)\right\|+\sigma_g+\sigma\right),
	$$
\end{theorem}
\textit{Proof.}
For client $i$, there are two cases to consider. In the first case, SGD selects non-perturbed samples in $\mathcal{S}$ and $\mathcal{S}^{(i)}$, which happens with probability $1-1 / n$. Then, we have
\begin{align}
	\begin{split}
		\left\|\theta_{i, k+1}^t-\theta_{i, k+1}^{t,\prime}\right\| \leq\left(1+\beta \eta_{k}^t\right)\left\|\theta_{i, k}^t-\theta_{i, k}^{t,\prime}\right\| .
	\end{split}
\end{align}
In the second case, SGD encounters the perturbed sample at time step $k$, which happens with probability $1 / n$. Then, we have
\begin{align}
	\begin{split}
		\left\|\theta_{i, k+1}^t-\theta_{i, k+1}^{t,\prime}\right\| & =\left\|\theta_{i, k}^t-\theta_{i, k}^{t,\prime}-\eta^t\left(g_i\left(\theta_{i, k}^t\right)-g_i^{\prime}\left(\theta_{i, k}^{t,\prime}\right)\right)\right\| \\
		& \leq\left\|\theta_{i, k}^t-\theta_{i, k}^{t,\prime}-\eta^t\left(g_i\left(\theta_{i, k}^t\right)-g_i\left(\theta_{i, k}^{t,\prime}\right)\right)\right\|+\eta^t\left\|g_i\left(\theta_{i, k}^{t,\prime}\right)-g_i^{\prime}\left(\theta_{i, k}^{t,\prime}\right)\right\| \\
		& \leq\left(1+\beta \eta_{k}^t\right)\left\|\theta_{i, k}^t-\theta_{i, k}^{t,\prime}\right\|+\eta^t\left\|g_i\left(\theta_{i, k}^{t,\prime}\right)-g_i^{\prime}\left(\theta_{i, k}^{t,\prime}\right)\right\| .
	\end{split}
\end{align}
Combining these two cases for client $i$ we have
\begin{align}
	\begin{split}
		\mathbb{E}\left\|\theta_{i, k+1}^t-\theta_{i, k+1}^{t,\prime}\right\| \leq & \left(1+\beta \eta^t\right) \mathbb{E}\left\|\theta_{i, k}^t-\theta_{i, k}^{t,\prime}\right\|+\frac{\eta^t}{n} \mathbb{E}\left\|g_i\left(\theta_{i, k}^{t,\prime}\right)-g_i^{\prime}\left(\theta_{i, k}^{t,\prime}\right)\right\| \\
		\leq & \left(1+\beta \eta_{k}^t\right) \mathbb{E}\left\|\theta_{i, k}^t-\theta_{i, k}^{t,\prime}\right\|+\frac{\eta_{k}^t}{n} \mathbb{E}\left\|g_i\left(\theta_{i, k}^t\right)\right\| \\
		\leq & \left(1+\beta \eta^t\right) \mathbb{E}\left\|\theta_{i, k}^t-\theta_{i, k}^{t,\prime}\right\|+\frac{2 \eta_{k}^t}{n}\left(1+(1+c)^{K-1} \beta \tilde{\eta}_{ t}\right)(\sigma +\mathbb{E}\left\|\nabla F\left(\theta_t\right)\right\|+ \sigma_{g, i}) \\
		\leq & \left(1+\beta \eta^t\right) \mathbb{E}\left\|\theta_{i, k}^t-\theta_{i, k}^{t,\prime}\right\|+\frac{2 \eta^t c}{n}\left(\mathbb{E}\left\|\nabla F\left(\theta_t\right)\right\|+ \sigma_{g, i}+\sigma\right)
	\end{split}
\end{align}
We let $\overline{\mathrm{c}}$ be an upper bound of $1+(1+c)^{K-1} \beta \tilde{\eta}_{ t}$ since $\tilde{\eta}_{ t}$ is bounded above. Then unrolling it gives
\begin{align}
	\begin{split}
		& \mathbb{E}\left\|\theta_{i, K}^t-\theta_{i, K}^{t,\prime}\right\| \leq \prod_{k=0}^{K-1}\left(1+\beta \eta_{k}^t\right) \mathbb{E}\left\|\theta_t-\theta_t^{\prime}\right\|+\left(\frac{2}{n} \sum_{k=0}^{K-1} \eta_{k}^t c\prod_{l=k+1}^{K-1}\left(1+\beta \eta_{l}^t\right)\left(\mathbb{E}\left\|\nabla F\left(\theta_t\right)\right\|+ \sigma_{g, i}+\sigma\right)\right) \\
		& \leq e^{\beta \tilde{\eta}_{ t}} \mathbb{E}\left\|\theta_t-\theta_t^{\prime}\right\|+\frac{2}{n} c \tilde{\eta}_{ t} e^{\beta \tilde{\eta}_{ t}}\left(\mathbb{E}\left\|\nabla F\left(\theta_t\right)\right\|+ \sigma_{g, i}+\sigma\right) . \\
		&
	\end{split}
\end{align}
We have
\begin{align}
	\begin{split}
		& \mathbb{E}\left\|\theta_{t+1}-\theta_{t+1}^{\prime}\right\| \leq \sum_{i=1}^m \frac{1}{m} \mathbb{E}\left\|\theta_{i, K}^t-\theta_{i, K}^{t,\prime}\right\| \\
		& \leq e^{\beta \tilde{\eta}_{ t}} \mathbb{E}\left\|\theta_t-\theta_t^{\prime}\right\|+\frac{2}{mn} c \tilde{\eta}_{ t} e^{\beta \tilde{\eta}_{ t}}\left(\mathbb{E}\left\|\nabla F\left(\theta_t\right)\right\|+ \sigma_{g}+\sigma\right) \\
		&
	\end{split}
\end{align}
Further, unrolling the above over $t$ and noting $\theta_0=\theta_0^{\prime}$, we obtain
\begin{align}
	\begin{split}
		\mathbb{E}\left\|\theta_T-\theta_T^{\prime}\right\| \leq \frac{2 \tilde{c}}{mn} \sum_{t=0}^{T-1} \exp \left(\beta \sum_{l=t+1}^{T-1} \tilde{\eta}_{ t}\right) \tilde{\eta}_{ t} e^{\beta \tilde{\eta}_{ t}}\left(\mathbb{E}\left\|\nabla F\left(\theta_t\right)\right\|+ \sigma_{g}+\sigma\right) ,
	\end{split}
\end{align}
\begin{align}
	\begin{split}
		\mathbb{E}\left\|\theta_T-\theta_T^{\prime}\right\| \leq \frac{2 \overline{c}}{m n} \frac{1}{\beta} e^{\frac{1}{T}+1}\left(\mathbb{E}\left\|\nabla F\left(\theta_t\right)\right\|+\sigma_g+\sigma\right),
	\end{split}
\end{align}

With the Assumption 2,
$$
\varepsilon_{\text {gen }}\leq \frac{2L \overline{c}}{m n} \frac{1}{\beta} e^{1+\frac{1}{ T}}\left(\mathbb{E}\left\|\nabla F\left(\theta_t\right)\right\|+\sigma_g+\sigma\right),
$$
$$
\begin{aligned}
	& \overline{\mathbf{c}} \leq 1+(1+c)^{K-1} \beta \tilde{\eta}_t \leq 1+(1+c)^{K-1} \beta \sum_{k=0}^{K-1} \eta^t \\
	& \tilde{\mathbf{c}} \leq 1+(1+c)^{K-1} \beta \tilde{\eta}_t \leq 1+(1+c)^{K-1} \beta \sum_{k=0}^{K-1} \eta_k^t \\
	& \tilde{\mathbf{c}} \leq \overline{\mathbf{c}}
\end{aligned}
$$
When the diminishing stepsizes are chosen in the statement of the theorem, we conclude the proof. The proof on MoFedSAM can be derived simply from FedSAM.

\end{document}